\documentclass{article}

\usepackage[final]{neurips_2025}

\usepackage[utf8]{inputenc} %
\usepackage[T1]{fontenc}    %
\usepackage{hyperref}       %
\usepackage{url}            %
\usepackage{booktabs}       %
\usepackage{amsfonts}       %
\usepackage{nicefrac}       %
\usepackage{microtype}      %

\usepackage{float}
\usepackage{minitoc}
\usepackage{titletoc}
\makeatletter
\g@addto@macro\normalsize{%
\setlength\abovedisplayskip{4pt}
\setlength\belowdisplayskip{4pt}
\setlength\abovedisplayshortskip{4pt}
\setlength\belowdisplayshortskip{4pt}
}
\makeatother

\usepackage{gensymb}
\usepackage{placeins}

\usepackage{booktabs}

\usepackage{lipsum}  
\usepackage{wrapfig}
\usepackage{nicefrac}
\usepackage{hyperref}
\hypersetup{
    colorlinks=true,
    linkcolor=blue,
    filecolor=magenta,      
    urlcolor=cyan
}
\usepackage{mathtools}
\usepackage{amsmath,amsfonts}
\usepackage{amssymb}
\usepackage[ruled,vlined]{algorithm2e}
\usepackage{algorithmic}
\usepackage{makecell}
\usepackage{diagbox}
\usepackage{multirow}
\usepackage{textgreek}

\usepackage{listings}

\usepackage{caption}
\usepackage{subcaption}

\usepackage{amsthm}
\newtheorem{theorem}{Theorem}

\usepackage[nameinlink,capitalize]{cleveref}

\crefname{section}{\S}{\S}
\Crefname{section}{\S}{\S}
\crefname{figure}{Fig.}{Figs.}
\Crefname{figure}{Fig.}{Figs.}
\crefname{table}{Tab.}{Tabs.}
\Crefname{table}{Tab.}{Tabs.}
\crefname{appendix}{App.}{Apps.}
\Crefname{appendix}{App.}{Apps.}
\crefname{theorem}{Thm.}{Thms.}
\Crefname{theorem}{Thm.}{Thms.}
\crefname{proposition}{Prop.}{Props.}
\Crefname{proposition}{Prop.}{Props.}

\crefname{algorithm}{Alg.}{Algs.}
\Crefname{algorithm}{Alg.}{Algs.}
\crefname{assumption}{Asm.}{Asms.}
\Crefname{assumption}{Asm.}{Asms.}
\crefname{mechanism}{Mech.}{Mechs.}
\Crefname{mechanism}{Mech.}{Mechs.}

\crefformat{footnote}{#2\footnotemark[#1]#3}
\crefmultiformat{footnote}{#2\footnotemark[#1]#3}%
{\textsuperscript{,}#2\footnotemark[#1]#3}{\textsuperscript{,}#2\footnotemark[#1]#3}{\textsuperscript{,}#2\footnotemark[#1]#3}

\makeatletter
\newcommand\footnoteref[1]{\protected@xdef\@thefnmark{\ref{#1}}\@footnotemark}
\makeatother

\newcommand{\highlightblock}[1]{\begin{center}\vspace{-0.1cm}\emph{#1}\vspace{-0.1cm}\end{center}}

\newcommand{\myparatightestn}[1]{ \noindent\textbf{{#1}}}

\newcounter{packednmbr}

\newenvironment{packeditemize}{\begin{list}{$\bullet$}{\setlength{\itemsep}{0.5pt}\addtolength{\labelwidth}{-4pt}\setlength{\leftmargin}{\labelwidth}\addtolength{\leftmargin}{-2pt}\setlength{\listparindent}{\parindent}\setlength{\parsep}{1pt}\setlength{\topsep}{0pt}}}{\end{list}}

\NewDocumentCommand{\codeword}{v}{%
\texttt{\textcolor{blue}{#1}}%
}

 \newcommand\blfootnote[1]{%
  \begingroup
  \renewcommand\thefootnote{}\footnote{#1}%
  \addtocounter{footnote}{-1}%
  \endgroup
}

\usepackage{bbold}

\usepackage[table,dvipsnames]{xcolor}

\definecolor{gray0}{gray}{0.9}

\newcommand{\graybox}[1]{\colorbox{gray0}{#1}}

\newcommand{\name}{\textcolor{orange}{\texttt{Latent Zoning Network}}}
\newcommand{\nameshort}{\textcolor{orange}{\texttt{LZN}}}

\newcommand{\bra}[1]{\left( #1 \right)}

\newcommand{\brc}[1]{\left\{ #1 \right\}}

\newcommand{\brn}[1]{\left\lVert #1 \right\rVert}

\newcommand{\supp}[1]{\textrm{Supp}\bra{#1}}

\newcommand{\sample}{x}
\newcommand{\sampley}{y}

\newcommand{\latentcomputationnotation}{C}

\newcommand{\latentcomputation}[1]{\latentcomputationnotation\bra{#1}}

\newcommand{\latent}{z}

\newcommand{\encodernotation}{E_x}
\newcommand{\encoder}[1]{\encodernotation\bra{#1}}
\newcommand{\encoderynotation}{E_y}

\newcommand{\decodernotation}{D_x}
\newcommand{\decoder}[1]{\decodernotation\bra{#1}}

\newcommand{\decoderynotation}{D_y}
\newcommand{\decodery}[1]{\decoderynotation\bra{#1}}

\newcommand{\anchor}{a}

\newcommand{\calN}{\mathcal{N}}

\newcommand{\normaldistribution}[2]{\calN\bra{#1,#2}}
\newcommand{\uniformdistribution}[1]{\textrm{Uniform}\brc{#1}}

\newcommand{\identity}{\mathbf{I}}

\newcommand{\expectation}[2]{\mathbb{E}_{#1}\bra{#2}}

\newcommand{\rf}[1]{\mathrm{FM}_\sample\bra{#1}}
\newcommand{\irf}[1]{\mathrm{IFM}_\sample\bra{#1}}

\newcommand{\labelanchor}{A}

\newcommand{\rflabelnotation}{\mathrm{FM}_\labelanchor}
\newcommand{\rflabel}[1]{\rflabelnotation\bra{#1}}

\newcommand{\prob}[1]{\mathbb{P}\bra{#1}}

\newcommand{\diff}{\mathrm{d}}

\newcommand{\bigO}{\mathcal{O}}

\newcommand{\mapping}{k}

\newcommand{\distance}[1]{d\bra{#1}}

\newcommand{\probassign}[2]{\mathbb{P}\bra{{#1} | {#2}}}

\newcommand{\solversteps}{r}

\newcommand{\solverstepsstart}{u}

\newcommand{\latentalignnotation}{\mathrm{Align}}
\newcommand{\latentalign}[1]{\latentalignnotation\bra{#1}}

\newcommand{\sampleset}{\mathcal{X}}
\newcommand{\sampleyset}{\mathcal{Y}}

\newcommand{\latentdim}{q}
\newcommand{\numclass}{c}

\newcommand{\real}{\mathbb{R}}

\newcommand{\mnist}{{\tt MNIST}}
\newcommand{\cifar}{{\tt CIFAR10}}
\newcommand{\afhqcat}{{\tt AFHQ-Cat}}
\newcommand{\celebahq}{{\tt CelebA-HQ}}
\newcommand{\lsunbedroom}{{\tt LSUN-Bedroom}}
\newcommand{\imagenet}{{\tt ImageNet}}

\newcommand{\threetwo}{($32\times 32$)}
\newcommand{\twofivesix}{($256\times 256$)}

\newcommand{\latentscale}{\alpha}

\title{\name{}: \\A Unified Principle for Generative Modeling, Representation Learning, and Classification}

\author{%
  Zinan Lin\thanks{Correspondence to: Zinan Lin (\url{zinanlin@microsoft.com}).} \\
  Microsoft Research\\
  Redmond, WA, USA \\
  \And
   Enshu Liu \\
  Tsinghua University \\
  Beijing, China \\
  \And
  Xuefei Ning \\
  Tsinghua University \\
  Beijing, China \\
  \And
  Junyi Zhu%
  \\
  KU Leuven \\
  Leuven, Belgium \\
  \And
  Wenyu Wang \\
  Redmond, WA, USA \\
  \And
  Sergey Yekhanin \\
  Microsoft Research\\
  Redmond, WA, USA 
}

\begin{document}

\vspace*{-0.8cm}
\maketitle

\vspace{-0.8cm}
\begin{abstract}
\vspace{-0.2cm}
Generative modeling, representation learning, and classification are three core problems in machine learning (ML), yet their state-of-the-art (SoTA) solutions remain largely disjoint. In this paper, we ask: \emph{Can a unified principle address all three?} Such unification could simplify ML pipelines and foster greater synergy across tasks. We introduce \name{} (\nameshort{}) as a step toward this goal. At its core, \nameshort{} creates a shared Gaussian latent space that encodes information across all tasks. Each data type (e.g., images, text, labels) is equipped with an encoder that maps samples to \emph{disjoint latent zones}, and a decoder that maps latents back to data. ML tasks are expressed as compositions of these encoders and decoders: for example, label-conditional image generation uses a label encoder and image decoder; image embedding uses an image encoder; classification uses an image encoder and label decoder. We demonstrate the promise of \nameshort{} in three increasingly complex scenarios: \textbf{(1) \nameshort{} can enhance existing models (image generation):} When combined with the SoTA Rectified Flow model, \nameshort{} improves FID on \cifar{} from 2.76 to 2.59—without modifying the training objective. \textbf{(2) \nameshort{} can solve tasks independently (representation learning):} \nameshort{} can implement unsupervised representation learning without auxiliary loss functions, outperforming the seminal MoCo and SimCLR methods by 9.3\% and 0.2\%, respectively, on downstream linear classification on \imagenet{}. \textbf{(3) \nameshort{} can solve multiple tasks simultaneously (joint generation and classification)}: With image and label encoders/decoders, \nameshort{} performs both tasks jointly by design, improving FID and achieving SoTA classification accuracy on \cifar{}. The code and trained models are available at \url{https://github.com/microsoft/latent-zoning-networks}. The project website is at \url{https://zinanlin.me/blogs/latent_zoning_networks.html}.

\end{abstract}

\blfootnote{$^\dagger$ Main updates in arXiv V2: Improving \cref{fig:latent}, fixing \cref{fig:operation}, expanding related work discussions in \cref{sec:gen}. }
\vspace{-0.5cm}

\vspace{-0.5cm}
\section{Introduction}
\label{sec:intro}
\vspace{-0.2cm}

\begin{figure}[t]
    \centering
    \vspace{-0.8cm}    \includegraphics[width=0.9\linewidth]{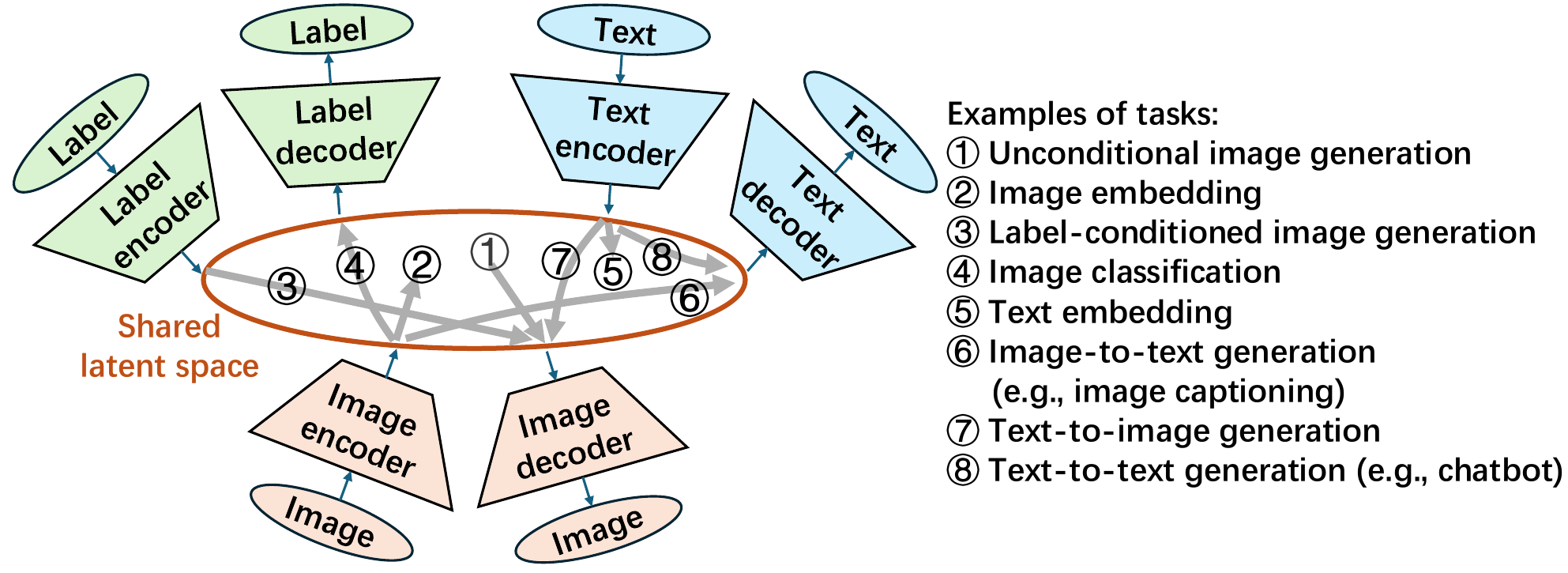}
    \caption{\name{} (\nameshort{}) connects multiple encoders and decoders through a shared latent space, enabling a wide range of ML tasks via different encoder-decoder combinations or standalone encoders/decoders. The figure illustrates eight example tasks, but more could be supported. Only tasks 1-4 are evaluated in this paper, while the rest are for illustration. }
    \label{fig:framework}
\end{figure}

\begin{figure}[t]
\centering
\vspace{-0.cm}
\begin{minipage}{.43\textwidth}
  \centering
  \includegraphics[width=.9\linewidth]{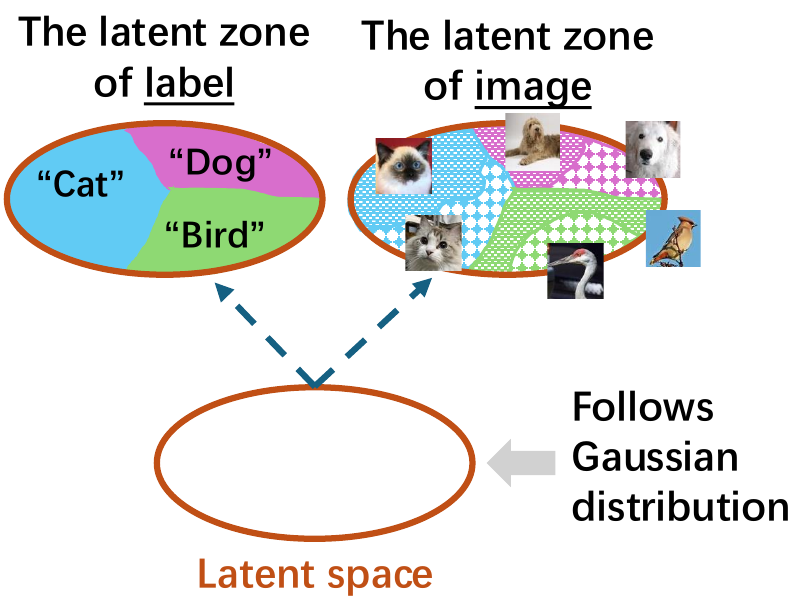}
  \captionof{figure}{The latent space of \nameshort{} has two key properties:  
\textbf{(1) Generative:} It follows a simple Gaussian prior, allowing easy sampling for generation tasks.  
\textbf{(2) Unified:} It serves as a shared representation across all \emph{data types} (e.g., image, text, label).  
Each data type induces a distinct partitioning of the latent space into \emph{latent zones}, where each zone corresponds to a specific sample (e.g., an individual image or label). The latent space is shown as a closed circle for illustration, but it is unbounded in practice.}
  \label{fig:latent}
\end{minipage}%
~~~
\begin{minipage}{.57\textwidth}
  \centering
  \includegraphics[width=.85\linewidth]{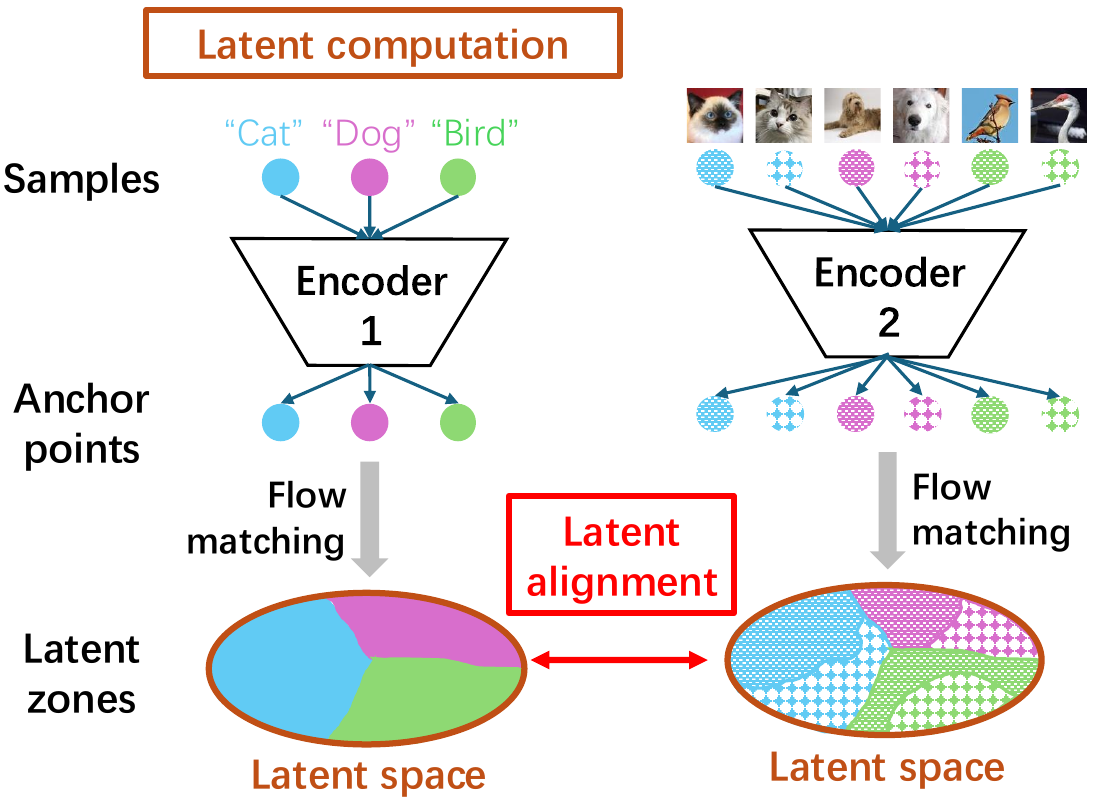}
  \captionof{figure}{Training and inference in \nameshort{} rely on two atomic operations:  
\textbf{(1) Latent computation (\cref{sec:lzn_latent_computation}):} Computes latent zones for a data type by encoding samples into \emph{anchor points} and using flow matching (FM) \cite{liu2022flow,lipman2022flow} to partition the latent space. Conversely, any latent point can be mapped to a sample via the decoder (not shown).  
\textbf{(2) Latent alignment (\cref{sec:lzn_latent_alignment}):} Aligns latent zones across data types by matching their FM processes.
\emph{This figure also illustrates the approach for \nameshort{} in joint conditional generative modeling and classification (\cref{sec:gen_and_class}).}
}

  \label{fig:operation}
\end{minipage}
\vspace{-0.4cm}
\end{figure}

Generative modeling, representation learning, and classification are three of the most widely used machine learning (ML) tasks. Generative models like DALLE~\cite{ramesh2021zero,ramesh2022hierarchical,betker2023improving} and GPT~\cite{radford2018improving,radford2019language,brown2020language,achiam2023gpt} power applications such as question answering and content creation. Representation learning, exemplified by CLIP~\cite{radford2021learning}, supports tasks like information retrieval. Classification is central to tasks such as object recognition~\cite{deng2009imagenet} and sentiment analysis~\cite{devlin2019bert,liu2019roberta}.

Notably, the state-of-the-art (SoTA) techniques for these tasks differ. For example, SoTA generative modeling relies on diffusion models \cite{ho2020denoising,sohl2015deep,song2019generative} and auto-regressive transformers \cite{radford2018improving,radford2019language,brown2020language,achiam2023gpt}; SoTA representation learning employs contrastive loss \cite{he2020momentum,chen2020simple,grill2020bootstrap}; and SoTA classification uses dedicated models trained with cross-entropy loss and its variants \cite{lin2017focal}.
Although using distinct methods for these tasks has long been established and widely accepted in the community, we revisit this methodology from first principles and question its necessity. Specifically, we ask, out of curiosity:
\highlightblock{Can a single principle unify generative modeling, representation learning, and classification?}
Part of our motivation stems from Occam's Razor \cite{Webb2010}, which favors simpler solutions when possible. More importantly, while these tasks differ in formulation, they are fundamentally related and can benefit from one another; a unified principle could facilitate such synergy.\footnote{While auto-regressive (AR) transformers with large-scale pre-training provide one approach to unify these tasks \cite{radford2018improving,radford2019language,brown2020language,achiam2023gpt}, SoTA transformer-based representation learning still relies on contrastive learning \cite{li2023towards,zhang2024mgte}. More importantly, our approach can be viewed as an orthogonal layer on top of transformers and should be seen as complementary rather than competing. See \cref{sec:relationship_to_alternatives} for further discussion.}

In this paper, we reflect on the strengths and limitations of existing techniques and propose a new unified framework, \name{} (\nameshort{}), illustrated in \cref{fig:framework,fig:latent}. At the core of our design is a shared \emph{latent space} that connects a series of encoders and decoders, each corresponding to a specific \emph{data type} (e.g., images, text, labels). 
Encoders map data into a \emph{zone} in the latent space, and the corresponding decoder maps it back to the data.
Different tasks can be interpreted as performing ``translations'' within the latent space—either using encoder–decoder pairs or leveraging a single encoder or decoder.
Compared to popular \emph{representation learning} approaches, which place no constraint on the latent distribution \cite{chen2020simple}, \nameshort{}'s latent space is \emph{generative}: it follows a simple prior distribution for easy sampling. In contrast to modern \emph{generative modeling} approaches, where different conditions (e.g., class labels, text) are treated as separate inputs \cite{zhang2023adding}, \nameshort{} maintains a \emph{single} latent space that unifies different types of conditional information. Finally, unlike standard \emph{classification}, where class labels are model outputs, \nameshort{} treats ``class labels'' as a data type connected through the latent space like any other data type.

To train and perform inference with \nameshort{}, we rely on two atomic operations (\cref{fig:operation,sec:lzn}):  
\textbf{(1) Latent computation:} Given an encoder, we compute the \emph{latent zones} for a batch of samples. Specifically, we first use the encoder to compute each sample's \emph{anchor point}, then apply flow matching \cite{liu2022flow,lipman2022flow} to map these points to their corresponding latent \emph{zones}. This procedure ensures that the resulting zones collectively follow a simple Gaussian distribution, facilitating generation tasks, while also guaranteeing that zones from different samples remain disjoint--allowing them to serve as unique latents for classification and representation learning.  
\textbf{(2) Latent alignment:} Aligning the \emph{latent zones} produced by two different encoders to facilitate the tasks that require translations between encoders and decoders from different data types. This is a fundamentally challenging task due to its discrete nature. To address it, we introduce a novel ``soft'' approximation that performs the alignment midway through the flow matching process, enabling effective and tractable training.

We demonstrate that, despite its simplicity and reliance on just two atomic operations, \nameshort{} is capable of supporting a wide range of seemingly diverse tasks. To illustrate its versatility and practical utility, we present three \emph{levels} of applications:
\begin{packeditemize}
    \item \textbf{L1: Enhancing \emph{one} existing task (\cref{sec:gen}).} Because \nameshort{} latents can be computed without supervision, they can be seamlessly integrated into existing models as an additional conditioning signal--without requiring any changes to the training loss or methods. In this setup, \nameshort{} latents hopefully can learn to improve the task performance. To demonstrate this, we incorporate \nameshort{} into rectified flow models \cite{liu2022flow}--a state-of-the-art generative approach for images--and observe improved sample quality across \cifar{}, \afhqcat{}, \celebahq{}, \lsunbedroom{} datasets. Specifically, on \cifar{}, \nameshort{} closes the FID gap between conditional and unconditional generation by 59\%.
    
    \item \textbf{L2: Solving \emph{one} task \emph{independently} (\cref{sec:repr}).} \nameshort{} can also tackle tasks entirely on its own, without relying on existing methods. As a case study, we use \nameshort{} to implement unsupervised representation learning, a task traditionally addressed with contrastive loss. We find that \nameshort{} can even outperform the seminal methods such as MoCo \cite{he2020momentum} and SimCLR \cite{chen2020simple} by 9.3\% and 0.2\%, respectively,  on downstream \imagenet{} linear classification.
    
    \item \textbf{L3: Solving \emph{multiple} tasks \emph{simultaneously} (\cref{sec:gen_and_class}).} Pushing further, \nameshort{} is capable of handling \emph{multiple} tasks at once. In particular, we employ two encoder–decoder pairs—one for images and one for labels—enabling \nameshort{} to jointly support class-conditional generation and classification within a single, unified framework. Built on rectified flow, this implementation outperforms the baseline conditional rectified flow model in generation quality, while also achieving state-of-the-art classification accuracy.  Notably, the performance on both generation and classification exceeds that of training each task in isolation.  This supports our core motivating intuition: seemingly distinct tasks can benefit from shared representations, and \nameshort{} provides a principled framework for enabling such synergy.
\end{packeditemize}

While the early results are promising, many challenges, open questions, and exciting opportunities remain unexplored. In principle, as more encoder–decoder pairs are added to the shared latent space, the range of applications \nameshort{} can support should grow at least quadratically (\cref{fig:framework}). Whether \nameshort{} can scale gracefully and realize this potential remains to be seen. We hope this work opens a new line of research toward this ambitious vision. See more discussions in \cref{sec:discussions}.

\vspace{-0.2cm}
\section{\name{} (\nameshort{})}
\label{sec:lzn}
\vspace{-0.2cm}

\subsection{Overall Framework}

\myparatightestn{Revisiting existing approaches.}
To motivate our design on a unified framework for diverse ML tasks, we first analyze the strengths and limitations of existing approaches.
\begin{packeditemize}
    \item \textbf{Generative modeling.} Given samples $\sample \sim p$ from an unknown distribution $p$, generative models aim to learn a decoder $\decodernotation$ such that $\decoder{\latent}$ approximates $p$, where $\latent$ is random noise drawn from a simple distribution.\footnote{For diffusion models \cite{ho2020denoising,sohl2015deep,song2019generative}, $\latent$ is the initial Gaussian noise in the sampling process, plus intermediate noise if using SDE sampling \cite{song2020score}. For AR transformers, $\latent$ can be seen as the randomness in token sampling.} While $\latent$  can carry useful information for representation learning and classification \cite{lin2020infogan}, this pipeline has key limitations: \textbf{(1)} The mapping from $\latent$ to $\sample$ lacks flexibility. For example, in diffusion models, the optimal mapping is fixed once the distributions of $\sample$ and $\latent$ are fixed \cite{song2019generative}. To introduce controllability, models often augment $D$ with additional condition inputs $c_1,\ldots,c_k$: $G(\latent, c_1, \ldots, c_k)$ \cite{zhang2023adding}. This is suboptimal--conditions may overlap or conflict (e.g., text vs. label conditions in image generation), and the resulting representation $(\latent, c_1, \ldots, c_k)$ becomes fragmented. \textbf{(2)} Inverting $D$ to recover $\latent$ from a sample $\sample$ is non-trivial for some SoTA generative models \cite{karras2020analyzing}. These issues limit the effectiveness of generative models for representation learning, as also observed in prior work \cite{fuest2024diffusion}.
    \item \textbf{Unsupervised representation learning.}\footnote{We use ``latent'', ``representation'', and ``embedding'' interchangeably in the paper.} SoTA representation learning typically uses contrastive loss \cite{chen2020simple}, where an encoder $E$ maps related sample pairs--either from the same modality (e.g., image augmentations) or across modalities (e.g., image–text pairs)--to similar embeddings, while pushing unrelated samples apart. These embeddings can perform zero-shot classification by comparing pairwise cosine similarities \cite{radford2021learning} or be adapted for classification using a linear head trained with cross-entropy loss \cite{chen2020simple}. However, contrastive loss leads $E$ to discard important details (e.g., augmentations, modalities), making the representations unsuitable for standalone generation. Moreover, with few exceptions \cite{arpit2021momentum}, the representations lack distributional constraints, making them hard to sample from for generative tasks unless training an additional generative model \cite{li2024return}.
    \item \textbf{Classification.} The most common and SoTA classification approach trains a dedicated model with cross-entropy loss to map inputs to class labels \cite{devlin2019bert,liu2019roberta}. Intermediate layer outputs can be used as representations \cite{szegedy2016rethinking}.  However, because the objective focuses solely on classification, these representations tend to discard class-irrelevant information, limiting their utility for general-purpose representation or generation. As with contrastive learning, they also lack distributional constraints for generative tasks.
\end{packeditemize}
While one could combine the above objectives and methods into a single training setup \cite{odena2017conditional,li2023mage}, our focus is on designing a clean, unified framework that naturally integrates all these tasks.

\myparatightestn{Desiderata.} We observe that all the above tasks can be framed as \emph{learning mappings between data and a latent space}. The main differences lie in: the mapping direction (e.g., latent-to-data for generation, data-to-latent for representation/classification), constraints on the latent space (e.g., a simple prior for generative models, none for others), and the amount of information encoded (e.g., class labels for classification tasks, detailed reconstructions for generative models). To support all these tasks in a single framework, we seek: \textbf{(1) A unified latent space} that captures all necessary information of all tasks; \textbf{(2) A generative latent space} that follows a simple distribution; and \textbf{(3) Easy mappings} between data and latent in both directions.

\myparatightestn{Framework.} 
To address the above desiderata, our key designs are (\cref{fig:latent}): 
\begin{packeditemize}
    \item \textbf{A unified latent space.}  
        In existing frameworks, a sample like text can play inconsistent roles--appearing as input latent in text-to-image generation or as output in text generation. This makes it hard to define a unified latent space across tasks.
        
        We address this by introducing a hypothetical \emph{foundation latent space} that represents all possible samples in the world. Each foundation latent is an abstract entity that appears through \emph{observations} in different \emph{data types}, such as images, text, and even class labels (e.g., ``cat'' or ``dog'').
        
        Importantly, different latents can share the same observation (e.g., multiple cat images all labeled ``cat'' and described as ``a cat image''). As a result, each observed sample defines a \emph{latent zone}—a subset of the latent space that produces the same observation in that data type. 
        This provides a unified way to represent and connect all data types within the same latent space.
    \item \textbf{A generative latent space.}  
        We enforce the latent space to follow a Gaussian distribution, enabling easy \emph{unconditional} sampling without constraining any data type.  Our framework also supports easy \emph{conditional} sampling from a \emph{latent zone} induced by an observed sample (e.g., a label).
    \item \textbf{Easy mappings.}  
        Given samples of a data type, we compute their latent zones via the corresponding \emph{encoder}. Conversely, a latent point can be decoded into a data type using its \emph{decoder}.

\end{packeditemize}

\myparatightestn{Tasks.} This design naturally supports a variety of tasks (\cref{fig:framework}):
\begin{packeditemize}
    \item \textbf{Single-module tasks.} A standalone encoder or decoder can perform specific tasks independently. For instance, the image encoder alone produces image embeddings (representations), while the image decoder alone enables unconditional image generation.
    \item \textbf{Cross-module tasks.} Any encoder–decoder pair defines a task. For example, \texttt{label encoder + image decoder} enables class-conditional image generation, \texttt{image encoder + label decoder} does classification, and \texttt{text encoder + text decoder} supports text generation.

\end{packeditemize}
We expect that tasks can benefit from each other through this unified framework (validated in \cref{sec:gen_and_class}). Each task contributes its core information to the latent space, making it increasingly expressive and powerful. Conversely, since all tasks interface through the same latent space, improvements in the latent representations can facilitate learning across tasks.

As the latent space partitions into zones, we name this framework \name{} (\nameshort{}).

\subsection{Implementation of Atomic Operations}

Training and inference in \nameshort{} rely on two operations (\cref{fig:operation}): \emph{latent computation} and \emph{latent alignment}.
We will see in \cref{sec:gen,sec:repr,sec:gen_and_class} that these two operations are sufficient to implement a variety of tasks.

\vspace{-0.1cm}
\subsubsection{Latent Computation}
\label{sec:lzn_latent_computation}
\vspace{-0.1cm}

\myparatightestn{Desiderata.}
Latent computation is important in both training and inference of \nameshort{}.
Given samples $\sampleset=\brc{\sample_1,\ldots,\sample_n}$ of the same data type (e.g., images, text, labels), the goal is to sample their latents $\latent_1,\ldots,\latent_n = \latentcomputation{\sampleset}$ with a \emph{random}  latent computation function $\latentcomputationnotation$ such that: \textbf{(1) Prior distribution is Gaussian:} the latent $z\sim \uniformdistribution{\latent_1,\ldots,\latent_n}$ follows Gaussian distribution $\normaldistribution{0}{\identity}$, and \textbf{(2) The latent zones of different samples are disjoint:}  $\supp{z_i} \cap \supp{z_j} = \emptyset$ for $i\not=j$.

\myparatightestn{Approach.}
To achieve this, we first use a deterministic encoder $\encodernotation$ to map each sample to an \emph{anchor point} $\anchor_i = \encoder{\sample_i}$. We then apply the seminal flow matching (FM) method \cite{liu2022flow,lipman2022flow}, which establishes a one-to-one mapping between distributions, to transform these anchor points into latent zones.
Specifically, we define a family of distributions $\pi_t$ with endpoints $\pi_0 = \mathcal{N}(0, \mathbf{I})$, the desired prior latent distribution, and $\pi_1(s) = \frac{1}{n} \sum_{i=1}^{n} \delta(s - \anchor_i)$, the distribution of the anchor points.\footnote{$\delta$ denotes the Dirac delta function.} The intermediate distribution $\pi_t$ is induced by linearly interpolating between samples $s_0 \sim \pi_0$ and $s_1 \sim \pi_1$ via $\varphi(s_0, s_1, t) = (1 - t)s_0 + t s_1$ (i.e., $\pi_t$ is $\bra{1-t}\pi_0+t\pi_1$).
The velocity field in FM \cite{liu2022flow} can then be computed as
\begin{align*}
    V(s,t) \triangleq \expectation{s_0\sim\pi_0,s_1\sim\pi_1}{\frac{\partial \varphi(s_0,s_1,t)}{\partial t}|\varphi(s_0,s_1,t)=s}=\frac{\sum_{i=1}^n (\anchor_i - s)\exp\bra{-\frac{\bra{s-t\anchor_i}^2}{2(1-t)^2}}}{(1-t)\sum_{i=1}^n \exp\bra{-\frac{\bra{s-t\anchor_i}^2}{2(1-t)^2}}}.
\end{align*}
We can obtain $s_t$ by integrating along the FM trajectory: $s_t = \rf{s_0; t} \triangleq s_0 + \int_{\tau=0}^t V(s_\tau, \tau)\diff\tau$ for $s_0 \sim \pi_0$. It has been shown \cite{liu2022flow} that the distribution of $s_t$ is $\pi_t$. Similarly, integrating backward $s_t = \irf{s_{1-g}; t} \triangleq s_{1-g} + \int_{\tau=1-g}^{t} V(s_\tau, \tau)\diff\tau$ for $s_{1-g}\sim \pi_{1-g} \triangleq (1-g) \pi_1 + g\mathcal{N}(0, \mathbf{I})$ also yields $\pi_t$, where $g$ is a small constant.\footnote{FM is well-defined only when both $\pi_0$ and $\pi_1$ have full support. However, in our case, $\pi_1$ is a mixture of Dirac deltas and lacks full support. Therefore, we use the full-support distribution $\pi_{1-g}$ as the starting point.} With a slight abuse of notation, we also write $s_t = \irf{a_i,\epsilon_i; t}$ to represent $\irf{s_{1-g}; t}$ with $s_{1-g} = (1-g)a_i+g\epsilon_i$.
With these setups, we define the latent computation as  
\begin{align}
    \latent_i=\latentcomputation{\sampleset}_i \triangleq \irf{a_i, \epsilon_i; 0}, \quad \text{where } \epsilon_i \sim \mathcal{N}(0, \mathbf{I}).\label{eq:latent_computation}
\end{align}
Due to the discussed FM properties, we can see that this approach (approximately) satisfies the two desiderata by construction (see \cref{app:latent_computation_property} for more details). 

\myparatightestn{Implementation.}  
Computing $\latentcomputationnotation{}$ involves an integral, which we approximate using standard numerical solvers such as Euler or DPM-Solver \cite{lu2022dpm,lu2022dpmp} that operates on a finite list of time steps $t_1,\ldots,t_\solversteps$, as in prior work on diffusion and RF models \cite{liu2022flow,song2020score}. A key property of our approach is that all operations are differentiable, allowing gradients to backpropagate all the way from latent $\latent_i$ to the encoders $\encodernotation{}$ during training.
More details are deferred to \cref{app:latent_computation_implementation}.

\myparatightestn{Efficiency optimization.}  
Computing the velocity $V$ requires access to all samples, making the memory and computation cost of $\latentcomputationnotation$ high. To address this, we introduce several optimization techniques--latent parallelism, custom gradient checkpointing, and minibatch approximation--that make the training of \nameshort{} scalable. See \cref{app:latent_computation_efficiency} for details.

\vspace{-0.2cm}
\subsubsection{Latent Alignment}
\label{sec:lzn_latent_alignment}
\vspace{-0.2cm}

\myparatightestn{Desiderata.}
Following \cref{sec:lzn_latent_computation}, latent zones from different data types are computed independently, which undermines the purpose of a shared latent space. Many applications require these zones to be \emph{aligned}. We consider two types of alignment:  
\textbf{(1) Many-to-one (and one-to-many) alignment:} for example, the latent zone of the ``cat'' label should cover all latent zones of all cat images.  
\textbf{(2) One-to-one alignment:} for example, in image-text datasets, paired image and text samples should share the same latent zone. Concrete examples will be shown in \cref{sec:repr,sec:gen_and_class}.

Formally, let $\sampleset=\brc{\sample_1,\ldots,\sample_n}$ and $\sampleyset=\brc{\sampley_1,\ldots,\sampley_m}$ be two datasets from different data types.  
The pairing is defined by $\mapping_i$, where $\sampley_i$ (e.g., a cat image) is paired with $\sample_{\mapping_i}$ (e.g., the ``cat'' label).  
We aim to ensure $\supp{\latentcomputation{\sampleset}_{\mapping_i}} \supseteq \supp{\latentcomputation{\sampleyset}_{i}}$ for all $i \in [m]$, meaning the latent zone of $\sample_{\mapping_i}$ covers that of $\sampley_i$.  
This formulation supports many-to-one alignments directly. For one-to-one alignment, a symmetric constraint can be added with $\sample$ and $\sampley$ swapped.

\myparatightestn{Approach.}  
Given the FM integral trajectories, alignment reduces to ensuring that the latent of $\sampley_i$, when mapped via the trajectory, matches the anchor point of $\sample_{\mapping_i}$:  
$
\rf{\latentcomputation{\sampleyset}_i; 1} = \encoder{\sample_{\mapping_i}}.
$

\underline{Challenge: discrete assignment is non-differentiable.}  Before introducing our solution, we illustrate why the problem is nontrivial by examining strawman approaches. 
A natural idea is to directly minimize the distance: 
$
\distance{\rf{\latentcomputation{\sampleyset}_i; 1}, \encoder{\sample_{\mapping_i}}},
$
where $\distance{\cdot,\cdot}$ is a distance metric. 
This approach fails because FM deterministically maps each latent to exactly one anchor point $\encoder{\sample_j}$, so the above objective effectively becomes minimizing the distance between anchor points. However, a latent zone is influenced by all anchor points, not just its own. Therefore, reducing the distance between a pair of anchors does not necessarily improve zone-level alignment.
More fundamentally, the core challenge is that FM induces a \emph{discrete assignment}: each latent deterministically maps to one anchor. This discrete operation is non-differentiable and cannot be directly optimized during training.

\underline{Technique 1: Soft approximation of alignment.} To address this issue, our key idea is to introduce a soft approximation of the discrete anchor assignment process.
Let us define $s_t^i = \rf{\latentcomputation{\sampleyset}_i; t}$ and $\anchor_l = \encoder{\sample_l}$. By construction, the distribution $\pi_t$ is a mixture of Gaussians:
$
\pi_t = \frac{1}{n}\sum_{l=1}^n \mathcal{N}(t\anchor_l, (1 - t)^2 \mathbf{I}),
$
where the $l$-th component corresponds to anchor $\anchor_l$.
We define the (soft) probability that $s_t^i$ is assigned to $\anchor_l$ as being proportional to the density of $s_t^i$ under the $l$-th Gaussian component:
\begin{align*}
\probassign{\anchor_l}{s_t^i} = \nicefrac{\exp\bra{-\frac{\|s_t^i - t\anchor_l\|^2}{2(1 - t)^2}}}{\sum_{j=1}^n \exp\bra{-\frac{\|s_t^i - t\anchor_j\|^2}{2(1 - t)^2}}}.
\end{align*}
This formulation provides a smooth, differentiable approximation of the otherwise discrete assignment. When $t = 0$, the approximation is fully smooth, with $\probassign{\anchor_l}{s_0^i} = 1/n$ for all $i,l$, reflecting a uniform assignment. As $t$ increases toward 1, the assignment becomes sharper. In the limit as $t \to 1$, %
it
converges to the true discrete assignment, where $s_t^i$ deterministically maps to its assigned anchor.

From this, a straightforward idea is to maximize the assignment probability over all time steps such as 
$\sum_{t\in \brc{t_1,\ldots,t_\solversteps}} \probassign{\anchor_{\mapping_i}}{s_t^i}$ 
(recall that $t_i$s are solver time steps; see \cref{sec:lzn_latent_computation}). 
However, our ultimate goal is only to ensure correct assignment at $t = 1$. Even if this is achieved, the above objective would continue to push the intermediate states $s_t$ toward $\anchor_{\mapping_i}$, which is unnecessary and potentially harmful.

\underline{Technique 2: Optimizing maximum assignment probability.} To avoid this, we propose to maximize 
$\max_{t\in \brc{t_1,\ldots,t_\solversteps}} \probassign{\anchor_{\mapping_i}}{s_t^i}$. 
This ensures that once the trajectory reaches the correct anchor near $t = 1$, the objective is maximized (i.e., equals 1) and no further gradient is applied %
as desired.
However, this approach introduces a new issue: if $s_t$ diverges from $\anchor_{\mapping_i}$ early on, the maximum probability remains at the constant $1/n$ (attained at $t = t_1 = 0$), yielding no training signal.

\underline{Technique 3: Early step cutoff.} 
To mitigate this, we truncate the set of time steps used in the maximization, restricting it to the later stages of the trajectory: 
$\brc{t_\solverstepsstart,\ldots,t_\solversteps}$, where $\solverstepsstart$ is a hyperparameter that excludes early time steps. 
Putting it all together, our proposed alignment objective is:
\begin{align}
    \latentalign{\sampleset,\sampleyset} \triangleq \textrm{maximize } \sum_{i=1}^m \max_{t \in \brc{t_\solverstepsstart,\ldots,t_\solversteps}}  \probassign{\anchor_{\mapping_i}}{s_t^i}.
    \label{eq:align_without_log}
\end{align}
Please see \cref{app:latent_alignment} for more implementation details.

\vspace{-0.2cm}
\subsection{Decoder}
\vspace{-0.2cm}

The decoder $\decodernotation$ maps the \nameshort{} latent back to its corresponding sample: $\decoder{\latent_i} = \sample_i$. As we see %
later, it can be implemented using either a generative model (\cref{sec:gen,sec:gen_and_class}) or FM (\cref{sec:gen_and_class}).

\vspace{-0.2cm}
\subsection{Relationships to Alternatives}
\label{sec:relationship_to_alternatives}
\vspace{-0.2cm}

A prominent alternative that can also unifies different ML tasks is the use of AR transformers with large-scale pertaining, such as large language models (LLMs)~\cite{radford2018improving,radford2019language,brown2020language,achiam2023gpt}. These models unify \emph{generation} tasks by representing all data as sequences of tokens and modeling them in an AR manner. \emph{Classification} tasks are cast as generation problems—for instance, prompting the model to complete the sentence ``Which animal is in this image?''~\cite{radford2019language}. Additionally, prior work has shown that the intermediate outputs in these models can serve as strong \emph{representation} for downstream tasks \cite{reimers-2019-sentence-bert}.

As such, this approach is \emph{generation-centric}: the core formulation remains a generative modeling problem. Tasks that can be framed as generation--such as classification-can be unified naturally. However, for other tasks like representation learning, this approach must rely on surrogate methods, such as using intermediate model outputs. In contrast, \emph{\nameshort{} offers a new formulation that seamlessly unifies all these tasks within a single framework}, as we will demonstrate in the following sections.

\textbf{More importantly, AR transformers (and other generative models) should be seen as \emph{orthogonal and complementary} to \nameshort{}, rather than as competitors.} In particular, \nameshort{} decoders that map latents to data samples can be instantiated using any generative model. We demonstrate this in \cref{sec:gen,sec:gen_and_class}. This allows \nameshort{} to leverage the strengths of existing generative models within its unified framework.

\vspace{-0.2cm}
\subsection{Scope of Experiments}
\vspace{-0.2cm}

While the framework is general, this paper focuses specifically on the \emph{image} domain. We present three case studies: (1) generation (\cref{sec:gen}), (2) unsupervised representation learning (\cref{sec:repr}), and (3) joint classification and generation (\cref{sec:gen_and_class}). These case studies are arranged in order of increasing complexity: the first \emph{enhances a single task}, the second solves a task using \emph{only \nameshort{} without any other external objectives}, and the third tackles \emph{multiple tasks simultaneously} within the same framework.

\vspace{-0.2cm}
\section{Case Study 1: Unconditional Generative Modeling}
\label{sec:gen}
\vspace{-0.2cm}

\begin{figure}[t]
    \centering

    \vspace{-0.8cm}       \includegraphics[width=0.9\linewidth]{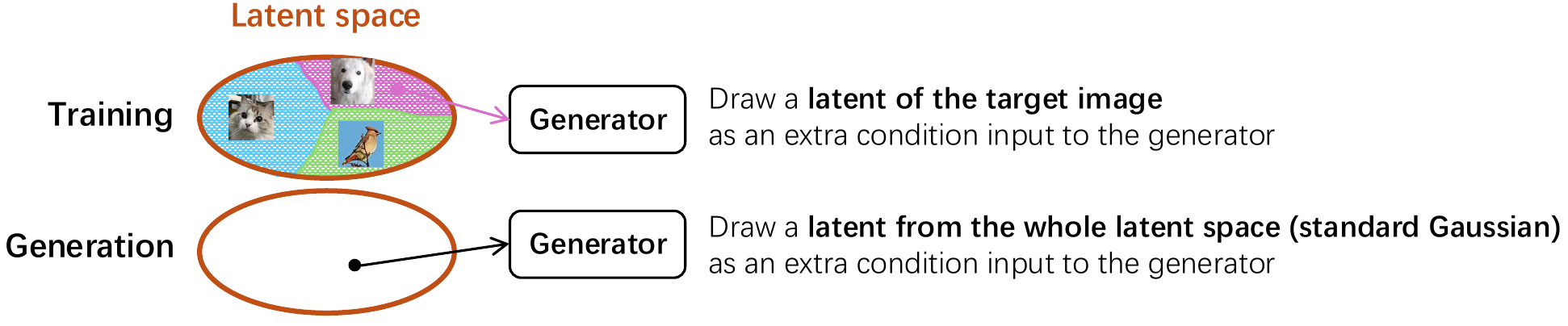}
    \caption{\nameshort{} for unconditional generative modeling (\cref{sec:gen}). During training, the \nameshort{} latent of each target image is fed as an extra condition to the rectified flow (RF) model \cite{liu2022flow}, making the RF learn \emph{conditional} flows based on \nameshort{} latents. The objective remains the standard RF loss, and the \nameshort{} encoder is trained end-to-end within it. During generation, we sample \nameshort{} latents from a standard Gaussian and use them as the extra condition. We illustrate the approach with RF, but since \nameshort{} latents require no supervision and are differentiable, the method could apply to other tasks by adding a condition input for \nameshort{} latents to the task network.
  }
    \label{fig:use_case_1_method}
\end{figure}

\myparatightestn{Approach (\cref{fig:use_case_1_method}).}  
Since \nameshort{} latents $\latentcomputation{\sampleset}$ can be computed using a standalone encoder without introducing new training objectives (\cref{sec:lzn_latent_computation}), they can be easily integrated into existing pipelines as additional network inputs, without modifying the original loss function.
We apply this idea to \emph{unconditional generative models}, where the generator serves as the \nameshort{} decoder. We modify the decoder by adding an extra input to accept the \nameshort{} latents. Both the encoder and decoder are trained jointly using the original generative loss. In this setup, latent alignment (\cref{sec:lzn_latent_alignment}) is not used.
Importantly, the encoder is only used during training to compute $\latentcomputation{\sampleset}$. During inference, latents are sampled from the Gaussian prior and passed directly to the decoder, \textbf{maintaining the original model's inference efficiency.} Please see \cref{app:gen_pseudocode} for detailed pseudocode of the training and generation processes.

\myparatightestn{Why it helps.}  
\nameshort{} latents provide unique representations for each image. This pushes the generator (decoder) to depend more directly on the latent for reconstruction, making the image distribution conditioned on the latent more deterministic. As a result, the generative objective becomes easier to optimize. Our experiments later confirm that \nameshort{} latents indeed capture useful features for generation.

\myparatightestn{Related work.}  
RCG \cite{li2024return} and Diffusion Autoencoder (DA) \cite{preechakul2022diffusion} also aim to improve image generation using unsupervised representations. RCG leverages a pre-trained representation model, while DA uses an encoder trained jointly with the generator. Similar to \nameshort{}, these representations are provided as additional inputs to the generative model. However, because their representation lacks distributional constraints, both methods require training an auxiliary generative model over the representations to enable unconditional generation. In contrast, \nameshort{} enforces a Gaussian representation by construction, eliminating this extra step and allowing everything to be trained end-to-end in one stage.

\begin{table*}[t]
\small
    \centering
    
    \vspace{-1cm}
    \caption{Unconditional image generation quality scores across four datasets. The best results are in \graybox{gray box}. Applying \nameshort{} to generative models improves RF on most image quality metrics. RF is a SoTA method; due to space constraints, we omit additional methods—see \cite{liu2022flow} for extensive comparisons between RF and others. Note that Inception Score (IS) is best suited for natural images like \cifar{}, though we report it for all datasets for completeness.}
    \vspace{-3mm}
    \label{tab:gen}
    \setlength{\tabcolsep}{3.8pt}
    \resizebox{1\textwidth}{!}{
    \begin{tabular}{l|ccccccc|ccccccc}
    \toprule
    \multirow{2}{*}{Algo.} & \multicolumn{7}{c|}{{\cifar{} \threetwo{}}} & \multicolumn{7}{c}{{\afhqcat{} \twofivesix{}}}\\
    \cline{2-15}
     & FID{\color{red}$\downarrow$} & 
     sFID{\color{red}$\downarrow$} & IS{\color{blue}$\uparrow$} & Precision{\color{blue}$\uparrow$} & Recall{\color{blue}$\uparrow$} & CMMD{\color{red}$\downarrow$} & Recon{\color{red}$\downarrow$} & FID{\color{red}$\downarrow$} & 
     sFID{\color{red}$\downarrow$} & IS{\color{blue}$\uparrow$} & Precision{\color{blue}$\uparrow$} & Recall{\color{blue}$\uparrow$} & CMMD{\color{red}$\downarrow$} & Recon{\color{red}$\downarrow$}\\
    \hline
    RF &   2.76 & 4.05 & 9.51 & \graybox{0.70} & \graybox{0.59} & 0.0360 & 0.83
    &  6.08 & 49.60 & 1.80 & 0.86 & 0.28 & 0.5145 & 17.92  \\
    RF+\nameshort{}  &  \graybox{2.59} & \graybox{3.95} & \graybox{9.53} & \graybox{0.70} & \graybox{0.59} & \graybox{0.0355} & \graybox{0.41}
    &  \graybox{5.68} & \graybox{49.32} & \graybox{1.96} & \graybox{0.87} & \graybox{0.30} & \graybox{0.3376} & \graybox{10.29} \\
    \bottomrule
    \toprule
    \multirow{2}{*}{Algo.} & \multicolumn{7}{c|}{{\celebahq{} \twofivesix{}}} & \multicolumn{7}{c}{{\lsunbedroom{} \twofivesix{}}}\\
    \cline{2-15}
     & FID{\color{red}$\downarrow$} & 
     sFID{\color{red}$\downarrow$} & IS{\color{blue}$\uparrow$} & Precision{\color{blue}$\uparrow$} & Recall{\color{blue}$\uparrow$} & CMMD{\color{red}$\downarrow$} & Recon{\color{red}$\downarrow$} & FID{\color{red}$\downarrow$} & 
     sFID{\color{red}$\downarrow$} & IS{\color{blue}$\uparrow$} & Precision{\color{blue}$\uparrow$} & Recall{\color{blue}$\uparrow$} & CMMD{\color{red}$\downarrow$} & Recon{\color{red}$\downarrow$}\\
    \hline
    RF & \graybox{6.95} & 10.61 & 2.91 & \graybox{0.76} & 0.42 & 1.0276 & 26.20 
    &  6.25 & \graybox{16.22} & 2.18 & \graybox{0.60} & 0.40 & 0.5218 & 48.72 \\
    RF+\nameshort{} &  7.17 & \graybox{10.33} & \graybox{2.92} & \graybox{0.76} & \graybox{0.45} & \graybox{0.4901} & \graybox{15.90}
    &  \graybox{5.95} & 17.84 & \graybox{2.22} & 0.59 & \graybox{0.41} & \graybox{0.4843} & \graybox{37.01} \\
    \bottomrule
\end{tabular}
}
\vspace{-0mm}
\end{table*}
\begin{figure}[t]
    \centering
    \includegraphics[width=0.32\linewidth]{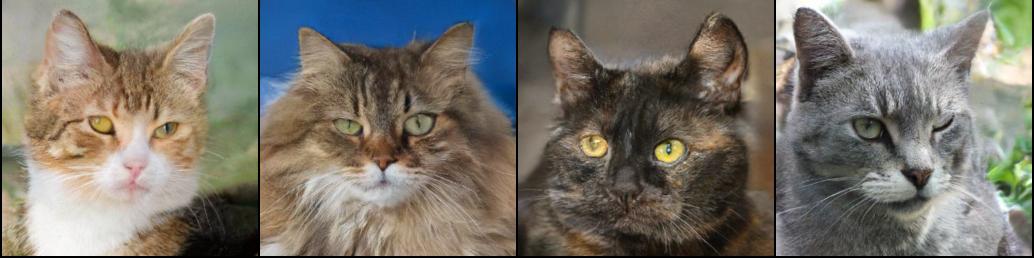}~~~\includegraphics[width=0.32\linewidth]{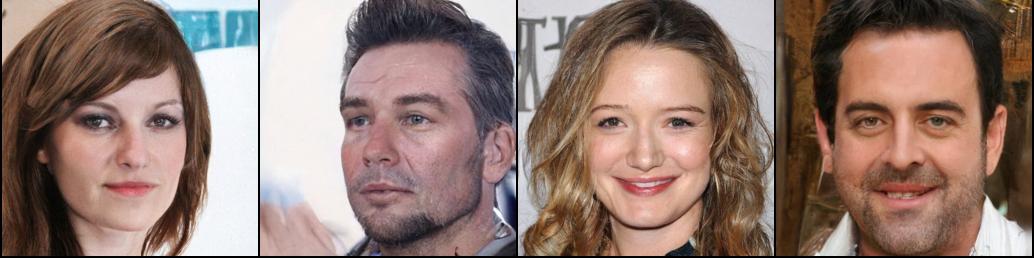}~~~\includegraphics[width=0.32\linewidth]{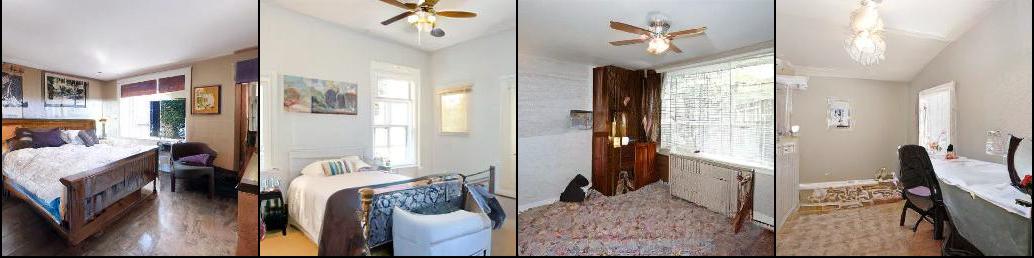}
    \caption{Generated images of RF+\nameshort{} on \afhqcat{}, \celebahq{}, \lsunbedroom{}. More in  \cref{app:gen}.}
    \label{fig:gen_image_few}
    \vspace{-0.5cm}
\end{figure}

\myparatightestn{Results.}
We plug \nameshort{} latents into the seminal rectified flow (RF) models \cite{liu2022flow}, which is closely related to diffusion models \cite{ho2020denoising,sohl2015deep,song2019generative}. RF achieved SoTA image generation performance on several datasets \cite{liu2022flow,lee2024improving}, and has been used in some latest Stable Diffusion \cite{esser2024scaling} and AR models \cite{liu2024distilled,liu2025distilled}.

We follow the experimental setup of RF \cite{liu2022flow}, evaluating on four datasets: \cifar{}, \afhqcat{}, \celebahq{}, and \lsunbedroom{}. The latter three are high-resolution ($256\times256$). In addition to standard metrics—FID \cite{heusel2017gans}, sFID \cite{nash2021generating}, IS \cite{salimans2016improved}, precision \cite{kynkaanniemi2019improved}, recall \cite{kynkaanniemi2019improved}, and CMMD \cite{jayasumana2024rethinking}—we also report \emph{reconstruction error} (the $\ell_2$ distance between an image and its reconstruction), which is relevant for applications like image editing via latent manipulation \cite{shen2020interpreting,lin2020infogan}. Results are shown in \cref{tab:gen}, with three key findings: \textbf{(1) \nameshort{} latents improve image quality.} During inference, the only difference in RF+\nameshort{} is the inclusion of an additional latent drawn from a Gaussian distribution. This improvement indicates that the decoder effectively leverages \nameshort{} latents for meaningful generative features. \textbf{(2) \nameshort{} significantly reduces reconstruction error across all datasets,} further confirming that its latents capture essential image information. \textbf{(3)} While unconditional generation is important, its image quality often trails that of conditional generation \cite{li2024return}. Compared to RF’s \cifar{} results in \cref{tab:gen_and_class_gen}, we find that \textbf{\nameshort{} substantially reduces the FID gap between conditional and unconditional generation by 59\%, and even outperforms conditional RF in sFID and reconstruction error}.

See \cref{fig:gen_image_few} for some generated images. Due to space constraints, please refer to \cref{app:gen} for \textbf{more details on implementation, datasets, metrics, and additional results such as more generated images and ablation studies}.

\vspace{-0.2cm}
\section{Case Study 2: Unsupervised
 Representation Learning}
 \label{sec:repr}
\vspace{-0.2cm}

\begin{figure}[t]
    \centering

    \vspace{-0.8cm}       \includegraphics[width=0.4\linewidth]{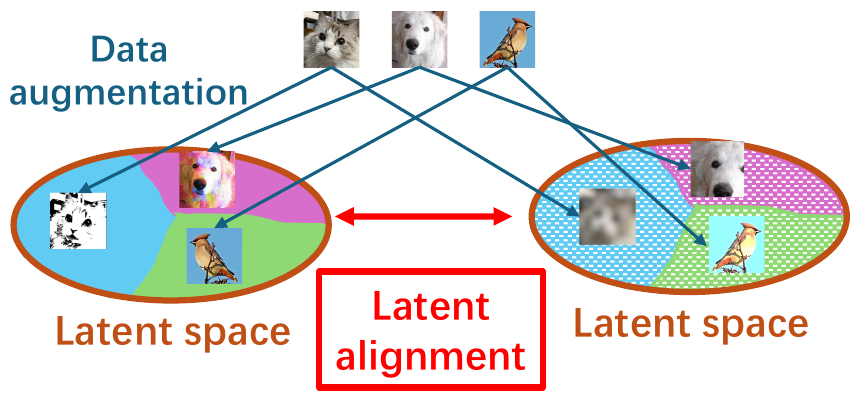}
    \caption{\nameshort{} for unsupervised representation learning (\cref{sec:repr}). During training, each image batch undergoes two sets of data augmentations, and latent zones for each set are computed \emph{using the same encoder}. We then apply latent alignment (\cref{sec:lzn_latent_alignment}) to train the encoder. At inference, we can use the \nameshort{} latents, the encoder outputs (i.e., anchor points), or intermediate encoder outputs (\cref{app:repr_choice}). The latter two options avoid the costly latent computation process.
  }
    \label{fig:use_case_2_method}
\end{figure}
 
\myparatightestn{Related work.}
Contrastive learning is a popular approach to unsupervised representation learning \cite{chen2020simple,chen2020big,he2020momentum,chen2020improved,chen2021empirical,grill2020bootstrap,chen2021exploring}, which pulls similar images (e.g., augmentations of the same image) together in the representation space. 
A central challenge is avoiding collapse, where all inputs are mapped to the same representation. Common solutions include pushing the representations of dissimilar samples from large batches \cite{chen2020simple} or memory banks \cite{he2020momentum} away, or adopting architectural designs that prevent collapse \cite{grill2020bootstrap,chen2021exploring,he2020momentum}. Other unsupervised representation learning approaches \cite{sander2022residual,li2022neural} include masked image modeling \cite{li2023mage,bao2021beit,he2022masked,pai2023masked} and training on other auxiliary tasks \cite{wei2022masked,noroozi2016unsupervised,pathak2016context,gidaris2018unsupervised}.

\myparatightestn{Approach (\cref{fig:use_case_2_method}).}
Inspired by contrastive learning, we train an image encoder using $\latentalign{\sampleset,\sampleyset}$ (\cref{sec:lzn_latent_alignment}), where $\sampleset=\brc{\sample_i}_{i=1}^n$ and $\sampleyset=\brc{\sampley_i}_{i=1}^n$ with mapping $\mapping_i = i$ contain image pairs $\bra{\sample_i,\sampley_i}$ of random augmentations of the same image. Unlike traditional contrastive methods, our approach inherently avoids collapse: different images are mapped to distinct latent zones by design, eliminating the need for large memory banks or specialized architectures. 
Notably, only a single \nameshort{} encoder for both $\sampleset$ and $\sampleyset$ is needed and decoders are not required.
See \cref{app:repr_pseudocode} for pseudocode of the training process.

\myparatightestn{Results.} We follow the canonical setting \cite{chen2020simple,grill2020bootstrap,he2020momentum}, where \nameshort{} is trained on the unlabelled \imagenet{} dataset using  ResNet-50 \cite{he2016deep} to learn representations. A \emph{linear} classifier is then trained on top of these representations in a supervised manner, and its accuracy is evaluated on the \imagenet{} test set. The results are shown in \cref{tab:repr}. Note that these results are obtained without any pretraining or use of external data. \textbf{We observe that \nameshort{} matches or outperforms several established methods in the field, including seminal approaches such as MoCo \cite{he2020momentum} (\nameshort{} outperforms it by 9.3\%) and SimCLR \cite{chen2020simple} (\nameshort{} outperforms it by 0.2\%)}. This is remarkable given that \nameshort{} is a new framework capable of supporting not only unsupervised representation learning but also other tasks (\cref{sec:gen,sec:gen_and_class}). However, there remains a significant gap to SoTA performance. \textbf{We emphasize that our current results are not fully optimized:} \textbf{(1) Training iteration.} The performance of \nameshort{} continues to improve rapidly with ongoing training (\cref{app:repr}), so we expect the gap to SoTA to narrow with full training. \textbf{(2) Architecture.} %
Prior work shows that more advanced architectures like ViT can improve the results significantly \cite{chen2021empirical}. We leave these directions for future work.

Due to space constraints, please refer to \cref{app:repr} for \textbf{more details on implementation, datasets, metrics, and additional results such as representation visualization and ablation studies}.

\begin{figure}[t]
\centering
\vspace{-0.2cm}
\begin{minipage}{.44\textwidth}
  \centering
    \captionof{table}{Classification accuracy on \imagenet{} by training a linear classifier on the unsupervised representations. Methods marked with \textsuperscript{\S} are based on contrastive learning.\protect\footnotemark{} The horizontal line separates baselines that perform worse or better than our \nameshort{}. All methods use the ResNet-50 architecture \cite{he2016deep} for fair comparison.\protect\footnotemark
}
    \label{tab:repr}
    \resizebox{0.95\linewidth}{!}{
    \begin{tabular}{c|c|c}
    \toprule
       Algorithm  & Top-1 Acc{\color{blue}$\uparrow$} &  Top-5 Acc{\color{blue}$\uparrow$} \\\midrule
       InstDisc\textsuperscript{\S} \cite{wu2018unsupervised}        & 54.0 \cite{he2020momentum}         & - \\
       BigBiGAN \cite{donahue2019large}                              & 56.6 \cite{he2020momentum}         & - \\
       LocalAgg\textsuperscript{\S} \cite{zhuang2019local}           & 58.8 \cite{he2020momentum}         & - \\
       MoCo\textsuperscript{\S} \cite{he2020momentum}                & 60.2 \cite{chen2020simple}         & - \\
       PIRL\textsuperscript{\S} \cite{misra2020self}                 & 63.6 \cite{chen2020simple}         & - \\
       CPC v2\textsuperscript{\S} \cite{henaff2020data}              & 63.8 \cite{chen2020simple}         & 85.3 \cite{chen2020simple}\\
       CMC\textsuperscript{\S} \cite{tian2020contrastive}            & 66.2 \cite{grill2020bootstrap}     & 87.0 \cite{grill2020bootstrap} \\
       SimSiam\textsuperscript{\S}\cite{chen2021exploring} & 68.1 \cite{chen2021exploring} & - \\
       SimCLR\textsuperscript{\S} \cite{chen2020simple}              & 69.3 \cite{chen2020simple}         & 89.0 \cite{chen2020simple}\\\hline
       MoCo v2\textsuperscript{\S} \cite{chen2020improved}           & 71.7 \cite{chen2020improved}       & - \\
       SimCLR v2\textsuperscript{\S} \cite{chen2020big}              & 71.7 \cite{chen2020big}            & - \\
       BYOL\textsuperscript{\S} \cite{grill2020bootstrap}            & 74.3 \cite{grill2020bootstrap}     & 91.6 \cite{grill2020bootstrap} \\
       DINO\textsuperscript{\S} \cite{caron2021emerging}            & 75.3 \cite{caron2021emerging}     & -\\\hline\hline
       \rowcolor{lightgray!60}
       \nameshort{}                                                  & 69.5 & 89.3\\\bottomrule
\end{tabular}
}
\end{minipage}
\hfill
\begin{minipage}{.54\textwidth}
  \centering
    \captionof{table}{Classification accuracy on \cifar{}. Baseline results are from \cite{cifaracc}. 
``RF+\nameshort{} (no gen)'' refers to ``RF+\nameshort{}'' with the RF loss for generation disabled. The horizontal line separates baselines that perform worse or better than our RF+\nameshort{}.  Note that these results refer to training purely on the \cifar{} dataset (without pretraining or external data).
}
    \label{tab:gen_and_class_class}
    \resizebox{0.53\linewidth}{!}{
    \begin{tabular}{c|c}
    \toprule
       Algorithm  & Acc{\color{blue}$\uparrow$} \\\midrule
        VGG16 & 92.64\% \\
        ResNet18 & 93.02\% \\
        ResNet50 & 93.62\% \\
        ResNet101 & 93.75\% \\
        RegNetX\_200MF & 94.24\% \\
        RegNetY\_400MF & 94.29\% \\
        MobileNetV2 & 94.43\% \\\hline
        ResNeXt29(32x4d) & 94.73\% \\
        ResNeXt29(2x64d) & 94.82\% \\
        SimpleDLA & 94.89\% \\
        DenseNet121 & 95.04\% \\
        PreActResNet18 & 95.11\% \\
        DPN92 & 95.16\% \\
        DLA & 95.47\% \\\hline\hline
        RF+\nameshort{} (no gen) & 93.59\% \\
        \rowcolor{lightgray!60} RF+\nameshort{} & 94.47\%\\\bottomrule
    \end{tabular}}
\end{minipage}%

\vspace{-0.3cm}

\end{figure}
\addtocounter{footnote}{-1}
\footnotetext{Note that we use the term \emph{contrastive learning} broadly to refer not only to methods employing the traditional contrastive loss, but to all approaches that encourage relevant images to share similar representations; see \cref{sec:repr}.}
\addtocounter{footnote}{1}
\footnotetext{It is known that better architectures \cite{chen2020big} or training on larger datasets \cite{oquab2023dinov2} can yield stronger results. To ensure a fair comparison, we include only methods reporting results with \textit{the ResNet-50 architecture} trained on \textit{the \imagenet{} dataset}. This excludes potentially stronger methods lacking ResNet-50 results on \imagenet{}, such as DINOv2 \cite{oquab2023dinov2}. See \cref{tab:repr_full} for additional baselines using other architectures.
}

\vspace{-0.2cm}
\section{Case Study 3: Conditional Generative Modeling and Classification}
\label{sec:gen_and_class}
\vspace{-0.2cm}

\begin{table*}[t]
\small
    \centering
    \caption{Conditional image generation quality on \cifar{}. The best results are in \graybox{gray box}. Applying \nameshort{} to generative models improves or matches RF on all metrics. %
    }
    \vspace{-3mm}
    \label{tab:gen_and_class_gen}
    \setlength{\tabcolsep}{3.8pt}
    \resizebox{0.6\textwidth}{!}{
    \begin{tabular}{l|ccccccc}
    \toprule
   Algo.
     & FID{\color{red}$\downarrow$} & 
     sFID{\color{red}$\downarrow$} & IS{\color{blue}$\uparrow$} & Precision{\color{blue}$\uparrow$} & Recall{\color{blue}$\uparrow$} & CMMD{\color{red}$\downarrow$} & Recon{\color{red}$\downarrow$} \\
    \hline
    RF &   2.47 & 4.05 & 9.77 & \graybox{0.71} & \graybox{0.58} & 0.0253 & 0.69
  \\
    RF+\nameshort{}  &  \graybox{2.40} & \graybox{3.99} & \graybox{9.88} & \graybox{0.71} & \graybox{0.58} & \graybox{0.0229} & \graybox{0.38}
     \\
    \bottomrule
\end{tabular}
}
\vspace{-0.3cm}
\end{table*}

\myparatightestn{Approach (\cref{fig:operation}).}
Building on \cref{sec:gen}, we consider $\sampleset = \brc{\sample_i}_{i=1}^n$ and $\sampleyset = \brc{\sampley_i}_{i=1}^m$, where each $\sample_i$ is a class label and $\sampley_i$ is an image labeled as $\sample_{\mapping_i}$. In addition to the image encoder and decoder, we introduce a label encoder-decoder pair. Since labels come from a finite set, both modules share a matrix $A \in \real^{\latentdim \times \numclass}$ of label anchor points, where $\latentdim$ is the latent dimension and $\numclass$ is the number of classes. The encoder maps a one-hot label $h$ to its anchor via $Ah$. The decoder recovers the class ID of a latent $g$ by first applying FM to obtain its anchor $\rflabel{g;1}$, and then computing its corresponding class latent zone $\arg\max \rflabel{g;1}^T A$, where $\rflabelnotation$ denotes FM over anchors in $A$. The training objective extends that of \cref{sec:gen} by adding $\latentalign{\sampleset,\sampleyset}$.
After training, the model can perform both conditional and unconditional generation, as well as classification by design.  
See \cref{app:gen_and_class_pseudocode} for detailed pseudocode of the training, generation, and classification processes.

\myparatightestn{Related work.}  
Joint classification and conditional generation are often achieved by augmenting generative models with classification losses or networks \cite{odena2017conditional,salimans2016improved}, treating label inputs to the generator and outputs from the classifier as separate components. In contrast, \nameshort{} unifies label inputs and outputs within a shared latent space.

\myparatightestn{Results.}  
Following the setting in \cref{sec:gen}, we conduct experiments on \cifar{}.  
Image quality metrics are shown in \cref{tab:gen_and_class_gen}, and classification accuracies are shown in \cref{tab:gen_and_class_class}. Key observations are:
\textbf{(1) \nameshort{} improves both image quality and reconstruction error.}  
Similar to \cref{sec:gen}, this confirms that \nameshort{} latents capture useful features for generation.
\textbf{(2) \nameshort{} achieves classification accuracy on par with SoTA.}  
\cref{tab:gen_and_class_class} includes SoTA classification accuracy from networks trained solely for classification. 
The fact that \nameshort{}, which jointly performs generation and classification and differs significantly from standard classification pipelines, can match SoTA performance is notable. 
Currently, \nameshort{} lags behind the best \cifar{} result by 1\%, potentially due to architectural factors: we use the RF encoder (\cref{app:gen_and_class}) without classification-specific optimization. 
With a better architecture design (as in other methods in \cref{tab:gen_and_class_class}), %
\nameshort{} could likely improve further.
\textbf{(3) Joint training on generation and classification improves both.}  
This is evident from: (i) RF+\nameshort{} in \cref{tab:gen_and_class_gen} showing better generation quality than in \cref{tab:gen}; and (ii) “RF+\nameshort{}” achieving higher classification accuracy than “RF+\nameshort{} (no gen)” in \cref{tab:gen_and_class_class}.  
These results support our motivation from \cref{sec:intro} that different ML tasks can benefit from each other, and demonstrate that \nameshort{} is a promising unified framework for achieving this synergy.

Due to space constraints, please refer to \cref{app:gen_and_class} for \textbf{more details on implementation, datasets, metrics, and additional results such as generated images and ablation studies}.

\vspace{-0.2cm}
\section{Limitations and Future Work}
\label{sec:discussions}
\vspace{-0.2cm}

\textbf{(1) Training efficiency.} Training \nameshort{} requires backpropagating through the FM trajectory (\cref{sec:lzn_latent_computation}), which is computationally expensive. In \cref{app:latent_computation}, we describe several optimization strategies we implemented to mitigate this cost. To further improve efficiency, we observe an interesting parallel between training \nameshort{} and training large language models (LLMs) (see \cref{app:discussions}), suggesting that some efficient training techniques developed for LLMs may be applicable here. \textbf{(2) Pure generative modeling.} While \nameshort{} is fundamentally capable of generative modeling without any auxiliary losses (see \cref{app:discussions}), in \cref{sec:gen}, we only demonstrate how it can enhance existing generative models. Exploring how to fully leverage \nameshort{} for standalone generative modeling remains an open direction for future work. \textbf{(3) Improving performance.} Although \nameshort{} achieves competitive results in unsupervised representation learning (\cref{sec:repr}) and classification (\cref{sec:gen_and_class}), there remains a gap to the SoTA. Bridging this gap is an interesting direction. One promising avenue is to incorporate well-established improvements from the literature that we have not yet adopted, such as more advanced architectural designs, as discussed in \cref{sec:repr,sec:gen_and_class}.
\textbf{(4) Multi-modality and multi-tasks.} In this paper, we focus primarily on image-based applications and at most two tasks simultaneously (\cref{sec:gen_and_class}). However, \nameshort{} is designed to be extensible: by incorporating additional encoders and decoders, it can naturally support more modalities and perform multiple tasks concurrently (\cref{fig:framework}). We leave this exploration to future work.

\section*{Acknowledgement}

The authors would like to thank Sangyun Lee for suggesting related work and for helpful discussions, as well as the anonymous reviewers for their valuable feedback. Xuefei Ning acknowledges the support by the National Natural Science Foundation of China (No. 62506197).
In addition, the authors gratefully acknowledge Cat Cookie and Cat Doudou for graciously lending their adorable faces for \cref{fig:latent,fig:operation,fig:use_case_1_method,fig:use_case_2_method}.\footnote{From \url{https://www.kaggle.com/datasets/fjxmlzn/cat-cookie-doudou} released in \cite{lin2023differentially}.}

\bibliographystyle{plain}
\bibliography{main}

\begin{thebibliography}{10}

\bibitem{achiam2023gpt}
Josh Achiam, Steven Adler, Sandhini Agarwal, Lama Ahmad, Ilge Akkaya,
  Florencia~Leoni Aleman, Diogo Almeida, Janko Altenschmidt, Sam Altman,
  Shyamal Anadkat, et~al.
\newblock Gpt-4 technical report.
\newblock {\em arXiv preprint arXiv:2303.08774}, 2023.

\bibitem{arpit2021momentum}
Devansh Arpit, Aadyot Bhatnagar, Huan Wang, and Caiming Xiong.
\newblock Momentum contrastive autoencoder: Using contrastive learning for
  latent space distribution matching in wae.
\newblock {\em arXiv preprint arXiv:2110.10303}, 2021.

\bibitem{assran2023self}
Mahmoud Assran, Quentin Duval, Ishan Misra, Piotr Bojanowski, Pascal Vincent,
  Michael Rabbat, Yann LeCun, and Nicolas Ballas.
\newblock Self-supervised learning from images with a joint-embedding
  predictive architecture.
\newblock In {\em Proceedings of the IEEE/CVF Conference on Computer Vision and
  Pattern Recognition}, pages 15619--15629, 2023.

\bibitem{bao2021beit}
Hangbo Bao, Li~Dong, Songhao Piao, and Furu Wei.
\newblock Beit: Bert pre-training of image transformers.
\newblock {\em arXiv preprint arXiv:2106.08254}, 2021.

\bibitem{betker2023improving}
James Betker, Gabriel Goh, Li~Jing, Tim Brooks, Jianfeng Wang, Linjie Li, Long
  Ouyang, Juntang Zhuang, Joyce Lee, Yufei Guo, et~al.
\newblock Improving image generation with better captions.
\newblock {\em Computer Science. https://cdn. openai. com/papers/dall-e-3.
  pdf}, 2(3):8, 2023.

\bibitem{brock2018large}
Andrew Brock, Jeff Donahue, and Karen Simonyan.
\newblock Large scale gan training for high fidelity natural image synthesis.
\newblock {\em arXiv preprint arXiv:1809.11096}, 2018.

\bibitem{brown2020language}
Tom Brown, Benjamin Mann, Nick Ryder, Melanie Subbiah, Jared~D Kaplan, Prafulla
  Dhariwal, Arvind Neelakantan, Pranav Shyam, Girish Sastry, Amanda Askell,
  et~al.
\newblock Language models are few-shot learners.
\newblock {\em Advances in neural information processing systems},
  33:1877--1901, 2020.

\bibitem{caron2018deep}
Mathilde Caron, Piotr Bojanowski, Armand Joulin, and Matthijs Douze.
\newblock Deep clustering for unsupervised learning of visual features.
\newblock In {\em Proceedings of the European conference on computer vision
  (ECCV)}, pages 132--149, 2018.

\bibitem{caron2021emerging}
Mathilde Caron, Hugo Touvron, Ishan Misra, Herv{\'e} J{\'e}gou, Julien Mairal,
  Piotr Bojanowski, and Armand Joulin.
\newblock Emerging properties in self-supervised vision transformers.
\newblock In {\em Proceedings of the IEEE/CVF international conference on
  computer vision}, pages 9650--9660, 2021.

\bibitem{chen2020simple}
Ting Chen, Simon Kornblith, Mohammad Norouzi, and Geoffrey Hinton.
\newblock A simple framework for contrastive learning of visual
  representations.
\newblock In {\em International conference on machine learning}, pages
  1597--1607. PmLR, 2020.

\bibitem{chen2020big}
Ting Chen, Simon Kornblith, Kevin Swersky, Mohammad Norouzi, and Geoffrey~E
  Hinton.
\newblock Big self-supervised models are strong semi-supervised learners.
\newblock {\em Advances in neural information processing systems},
  33:22243--22255, 2020.

\bibitem{chen2020improved}
Xinlei Chen, Haoqi Fan, Ross Girshick, and Kaiming He.
\newblock Improved baselines with momentum contrastive learning.
\newblock {\em arXiv preprint arXiv:2003.04297}, 2020.

\bibitem{chen2021exploring}
Xinlei Chen and Kaiming He.
\newblock Exploring simple siamese representation learning.
\newblock In {\em Proceedings of the IEEE/CVF conference on computer vision and
  pattern recognition}, pages 15750--15758, 2021.

\bibitem{chen2021empirical}
Xinlei Chen, Saining Xie, and Kaiming He.
\newblock An empirical study of training self-supervised vision transformers.
\newblock In {\em Proceedings of the IEEE/CVF international conference on
  computer vision}, pages 9640--9649, 2021.

\bibitem{choi2020stargan}
Yunjey Choi, Youngjung Uh, Jaejun Yoo, and Jung-Woo Ha.
\newblock Stargan v2: Diverse image synthesis for multiple domains.
\newblock In {\em Proceedings of the IEEE/CVF conference on computer vision and
  pattern recognition}, pages 8188--8197, 2020.

\bibitem{dao2023flow}
Quan Dao, Hao Phung, Binh Nguyen, and Anh Tran.
\newblock Flow matching in latent space.
\newblock {\em arXiv preprint arXiv:2307.08698}, 2023.

\bibitem{deng2009imagenet}
Jia Deng, Wei Dong, Richard Socher, Li-Jia Li, Kai Li, and Li~Fei-Fei.
\newblock Imagenet: A large-scale hierarchical image database.
\newblock In {\em 2009 IEEE conference on computer vision and pattern
  recognition}, pages 248--255. Ieee, 2009.

\bibitem{deng2012mnist}
Li~Deng.
\newblock The mnist database of handwritten digit images for machine learning
  research [best of the web].
\newblock {\em IEEE signal processing magazine}, 29(6):141--142, 2012.

\bibitem{devlin2019bert}
Jacob Devlin, Ming-Wei Chang, Kenton Lee, and Kristina Toutanova.
\newblock Bert: Pre-training of deep bidirectional transformers for language
  understanding.
\newblock In {\em Proceedings of the 2019 conference of the North American
  chapter of the association for computational linguistics: human language
  technologies, volume 1 (long and short papers)}, pages 4171--4186, 2019.

\bibitem{doersch2015unsupervised}
Carl Doersch, Abhinav Gupta, and Alexei~A Efros.
\newblock Unsupervised visual representation learning by context prediction.
\newblock In {\em Proceedings of the IEEE international conference on computer
  vision}, pages 1422--1430, 2015.

\bibitem{donahue2019large}
Jeff Donahue and Karen Simonyan.
\newblock Large scale adversarial representation learning.
\newblock {\em Advances in neural information processing systems}, 32, 2019.

\bibitem{dosovitskiy2014discriminative}
Alexey Dosovitskiy, Jost~Tobias Springenberg, Martin Riedmiller, and Thomas
  Brox.
\newblock Discriminative unsupervised feature learning with convolutional
  neural networks.
\newblock {\em Advances in neural information processing systems}, 27, 2014.

\bibitem{esser2024scaling}
Patrick Esser, Sumith Kulal, Andreas Blattmann, Rahim Entezari, Jonas
  M{\"u}ller, Harry Saini, Yam Levi, Dominik Lorenz, Axel Sauer, Frederic
  Boesel, et~al.
\newblock Scaling rectified flow transformers for high-resolution image
  synthesis.
\newblock In {\em Forty-first international conference on machine learning},
  2024.

\bibitem{fuest2024diffusion}
Michael Fuest, Pingchuan Ma, Ming Gui, Johannes Schusterbauer, Vincent~Tao Hu,
  and Bjorn Ommer.
\newblock Diffusion models and representation learning: A survey.
\newblock {\em arXiv preprint arXiv:2407.00783}, 2024.

\bibitem{gidaris2018unsupervised}
Spyros Gidaris, Praveer Singh, and Nikos Komodakis.
\newblock Unsupervised representation learning by predicting image rotations.
\newblock {\em arXiv preprint arXiv:1803.07728}, 2018.

\bibitem{grill2020bootstrap}
Jean-Bastien Grill, Florian Strub, Florent Altch{\'e}, Corentin Tallec, Pierre
  Richemond, Elena Buchatskaya, Carl Doersch, Bernardo Avila~Pires, Zhaohan
  Guo, Mohammad Gheshlaghi~Azar, et~al.
\newblock Bootstrap your own latent-a new approach to self-supervised learning.
\newblock {\em Advances in neural information processing systems},
  33:21271--21284, 2020.

\bibitem{he2022masked}
Kaiming He, Xinlei Chen, Saining Xie, Yanghao Li, Piotr Doll{\'a}r, and Ross
  Girshick.
\newblock Masked autoencoders are scalable vision learners.
\newblock In {\em Proceedings of the IEEE/CVF conference on computer vision and
  pattern recognition}, pages 16000--16009, 2022.

\bibitem{he2020momentum}
Kaiming He, Haoqi Fan, Yuxin Wu, Saining Xie, and Ross Girshick.
\newblock Momentum contrast for unsupervised visual representation learning.
\newblock In {\em Proceedings of the IEEE/CVF conference on computer vision and
  pattern recognition}, pages 9729--9738, 2020.

\bibitem{he2016deep}
Kaiming He, Xiangyu Zhang, Shaoqing Ren, and Jian Sun.
\newblock Deep residual learning for image recognition.
\newblock In {\em Proceedings of the IEEE conference on computer vision and
  pattern recognition}, pages 770--778, 2016.

\bibitem{henaff2020data}
Olivier Henaff.
\newblock Data-efficient image recognition with contrastive predictive coding.
\newblock In {\em International conference on machine learning}, pages
  4182--4192. PMLR, 2020.

\bibitem{heusel2017gans}
Martin Heusel, Hubert Ramsauer, Thomas Unterthiner, Bernhard Nessler, and Sepp
  Hochreiter.
\newblock Gans trained by a two time-scale update rule converge to a local nash
  equilibrium.
\newblock {\em Advances in neural information processing systems}, 30, 2017.

\bibitem{ho2020denoising}
Jonathan Ho, Ajay Jain, and Pieter Abbeel.
\newblock Denoising diffusion probabilistic models.
\newblock {\em Advances in neural information processing systems},
  33:6840--6851, 2020.

\bibitem{jayasumana2024rethinking}
Sadeep Jayasumana, Srikumar Ramalingam, Andreas Veit, Daniel Glasner, Ayan
  Chakrabarti, and Sanjiv Kumar.
\newblock Rethinking fid: Towards a better evaluation metric for image
  generation.
\newblock In {\em Proceedings of the IEEE/CVF Conference on Computer Vision and
  Pattern Recognition}, pages 9307--9315, 2024.

\bibitem{karras2017progressive}
Tero Karras, Timo Aila, Samuli Laine, and Jaakko Lehtinen.
\newblock Progressive growing of gans for improved quality, stability, and
  variation.
\newblock {\em arXiv preprint arXiv:1710.10196}, 2017.

\bibitem{karras2019style}
Tero Karras, Samuli Laine, and Timo Aila.
\newblock A style-based generator architecture for generative adversarial
  networks.
\newblock In {\em Proceedings of the IEEE/CVF conference on computer vision and
  pattern recognition}, pages 4401--4410, 2019.

\bibitem{karras2020analyzing}
Tero Karras, Samuli Laine, Miika Aittala, Janne Hellsten, Jaakko Lehtinen, and
  Timo Aila.
\newblock Analyzing and improving the image quality of stylegan.
\newblock In {\em Proceedings of the IEEE/CVF conference on computer vision and
  pattern recognition}, pages 8110--8119, 2020.

\bibitem{kingma2018glow}
Durk~P Kingma and Prafulla Dhariwal.
\newblock Glow: Generative flow with invertible 1x1 convolutions.
\newblock {\em Advances in neural information processing systems}, 31, 2018.

\bibitem{krizhevsky2009learning}
Alex Krizhevsky, Geoffrey Hinton, et~al.
\newblock Learning multiple layers of features from tiny images.
\newblock 2009.

\bibitem{cifaracc}
kuangliu.
\newblock Train cifar10 with pytorch.
\newblock \url{https://github.com/kuangliu/pytorch-cifar}, 2025.

\bibitem{kynkaanniemi2019improved}
Tuomas Kynk{\"a}{\"a}nniemi, Tero Karras, Samuli Laine, Jaakko Lehtinen, and
  Timo Aila.
\newblock Improved precision and recall metric for assessing generative models.
\newblock {\em Advances in neural information processing systems}, 32, 2019.

\bibitem{lee2024improving}
Sangyun Lee, Zinan Lin, and Giulia Fanti.
\newblock Improving the training of rectified flows.
\newblock {\em Advances in Neural Information Processing Systems},
  37:63082--63109, 2024.

\bibitem{li2022neural}
Mufan Li, Mihai Nica, and Dan Roy.
\newblock The neural covariance sde: Shaped infinite depth-and-width networks
  at initialization.
\newblock {\em Advances in Neural Information Processing Systems},
  35:10795--10808, 2022.

\bibitem{li2023mage}
Tianhong Li, Huiwen Chang, Shlok Mishra, Han Zhang, Dina Katabi, and Dilip
  Krishnan.
\newblock Mage: Masked generative encoder to unify representation learning and
  image synthesis.
\newblock In {\em Proceedings of the IEEE/CVF Conference on Computer Vision and
  Pattern Recognition}, pages 2142--2152, 2023.

\bibitem{li2024return}
Tianhong Li, Dina Katabi, and Kaiming He.
\newblock Return of unconditional generation: A self-supervised representation
  generation method.
\newblock {\em Advances in Neural Information Processing Systems},
  37:125441--125468, 2024.

\bibitem{li2023towards}
Zehan Li, Xin Zhang, Yanzhao Zhang, Dingkun Long, Pengjun Xie, and Meishan
  Zhang.
\newblock Towards general text embeddings with multi-stage contrastive
  learning.
\newblock {\em arXiv preprint arXiv:2308.03281}, 2023.

\bibitem{lin2017focal}
Tsung-Yi Lin, Priya Goyal, Ross Girshick, Kaiming He, and Piotr Doll{\'a}r.
\newblock Focal loss for dense object detection.
\newblock In {\em Proceedings of the IEEE international conference on computer
  vision}, pages 2980--2988, 2017.

\bibitem{lin2023differentially}
Zinan Lin, Sivakanth Gopi, Janardhan Kulkarni, Harsha Nori, and Sergey
  Yekhanin.
\newblock Differentially private synthetic data via foundation model apis 1:
  Images.
\newblock {\em arXiv preprint arXiv:2305.15560}, 2023.

\bibitem{lin2018pacgan}
Zinan Lin, Ashish Khetan, Giulia Fanti, and Sewoong Oh.
\newblock Pacgan: The power of two samples in generative adversarial networks.
\newblock {\em Advances in neural information processing systems}, 31, 2018.

\bibitem{lin2021spectral}
Zinan Lin, Vyas Sekar, and Giulia Fanti.
\newblock Why spectral normalization stabilizes gans: Analysis and
  improvements.
\newblock {\em Advances in neural information processing systems},
  34:9625--9638, 2021.

\bibitem{lin2020infogan}
Zinan Lin, Kiran Thekumparampil, Giulia Fanti, and Sewoong Oh.
\newblock Infogan-cr and modelcentrality: Self-supervised model training and
  selection for disentangling gans.
\newblock In {\em international conference on machine learning}, pages
  6127--6139. PMLR, 2020.

\bibitem{lipman2022flow}
Yaron Lipman, Ricky~TQ Chen, Heli Ben-Hamu, Maximilian Nickel, and Matt Le.
\newblock Flow matching for generative modeling.
\newblock {\em arXiv preprint arXiv:2210.02747}, 2022.

\bibitem{liu2025distilled}
Enshu Liu, Qian Chen, Xuefei Ning, Shengen Yan, Guohao Dai, Zinan Lin, and
  Yu~Wang.
\newblock Distilled decoding 2: One-step sampling of image auto-regressive
  models with conditional score distillation.
\newblock {\em arXiv preprint arXiv:2510.21003}, 2025.

\bibitem{liu2024distilled}
Enshu Liu, Xuefei Ning, Yu~Wang, and Zinan Lin.
\newblock Distilled decoding 1: One-step sampling of image auto-regressive
  models with flow matching.
\newblock {\em arXiv preprint arXiv:2412.17153}, 2024.

\bibitem{liu2022flow}
Xingchao Liu, Chengyue Gong, and Qiang Liu.
\newblock Flow straight and fast: Learning to generate and transfer data with
  rectified flow.
\newblock {\em arXiv preprint arXiv:2209.03003}, 2022.

\bibitem{liu2019roberta}
Yinhan Liu, Myle Ott, Naman Goyal, Jingfei Du, Mandar Joshi, Danqi Chen, Omer
  Levy, Mike Lewis, Luke Zettlemoyer, and Veselin Stoyanov.
\newblock Roberta: A robustly optimized bert pretraining approach.
\newblock {\em arXiv preprint arXiv:1907.11692}, 2019.

\bibitem{lu2022dpm}
Cheng Lu, Yuhao Zhou, Fan Bao, Jianfei Chen, Chongxuan Li, and Jun Zhu.
\newblock Dpm-solver: A fast ode solver for diffusion probabilistic model
  sampling in around 10 steps.
\newblock {\em Advances in Neural Information Processing Systems},
  35:5775--5787, 2022.

\bibitem{lu2022dpmp}
Cheng Lu, Yuhao Zhou, Fan Bao, Jianfei Chen, Chongxuan Li, and Jun Zhu.
\newblock Dpm-solver++: Fast solver for guided sampling of diffusion
  probabilistic models.
\newblock {\em arXiv preprint arXiv:2211.01095}, 2022.

\bibitem{misra2020self}
Ishan Misra and Laurens van~der Maaten.
\newblock Self-supervised learning of pretext-invariant representations.
\newblock In {\em Proceedings of the IEEE/CVF conference on computer vision and
  pattern recognition}, pages 6707--6717, 2020.

\bibitem{nash2021generating}
Charlie Nash, Jacob Menick, Sander Dieleman, and Peter~W Battaglia.
\newblock Generating images with sparse representations.
\newblock {\em arXiv preprint arXiv:2103.03841}, 2021.

\bibitem{noroozi2016unsupervised}
Mehdi Noroozi and Paolo Favaro.
\newblock Unsupervised learning of visual representations by solving jigsaw
  puzzles.
\newblock In {\em European conference on computer vision}, pages 69--84.
  Springer, 2016.

\bibitem{odena2017conditional}
Augustus Odena, Christopher Olah, and Jonathon Shlens.
\newblock Conditional image synthesis with auxiliary classifier gans.
\newblock In {\em International conference on machine learning}, pages
  2642--2651. PMLR, 2017.

\bibitem{oord2018representation}
Aaron van~den Oord, Yazhe Li, and Oriol Vinyals.
\newblock Representation learning with contrastive predictive coding.
\newblock {\em arXiv preprint arXiv:1807.03748}, 2018.

\bibitem{oquab2023dinov2}
Maxime Oquab, Timoth{\'e}e Darcet, Th{\'e}o Moutakanni, Huy Vo, Marc
  Szafraniec, Vasil Khalidov, Pierre Fernandez, Daniel Haziza, Francisco Massa,
  Alaaeldin El-Nouby, et~al.
\newblock Dinov2: Learning robust visual features without supervision.
\newblock {\em arXiv preprint arXiv:2304.07193}, 2023.

\bibitem{pai2023masked}
Druv Pai, Ziyang~Wu Wu, Sam Buchanan, Yaodong Yu, and Yi~Ma.
\newblock Masked completion via structured diffusion with white-box
  transformers.
\newblock International Conference on Learning Representations, 2023.

\bibitem{parmar2022aliased}
Gaurav Parmar, Richard Zhang, and Jun-Yan Zhu.
\newblock On aliased resizing and surprising subtleties in gan evaluation.
\newblock In {\em Proceedings of the IEEE/CVF conference on computer vision and
  pattern recognition}, pages 11410--11420, 2022.

\bibitem{pathak2016context}
Deepak Pathak, Philipp Krahenbuhl, Jeff Donahue, Trevor Darrell, and Alexei~A
  Efros.
\newblock Context encoders: Feature learning by inpainting.
\newblock In {\em Proceedings of the IEEE conference on computer vision and
  pattern recognition}, pages 2536--2544, 2016.

\bibitem{preechakul2022diffusion}
Konpat Preechakul, Nattanat Chatthee, Suttisak Wizadwongsa, and Supasorn
  Suwajanakorn.
\newblock Diffusion autoencoders: Toward a meaningful and decodable
  representation.
\newblock In {\em Proceedings of the IEEE/CVF conference on computer vision and
  pattern recognition}, pages 10619--10629, 2022.

\bibitem{radford2021learning}
Alec Radford, Jong~Wook Kim, Chris Hallacy, Aditya Ramesh, Gabriel Goh,
  Sandhini Agarwal, Girish Sastry, Amanda Askell, Pamela Mishkin, Jack Clark,
  et~al.
\newblock Learning transferable visual models from natural language
  supervision.
\newblock In {\em International conference on machine learning}, pages
  8748--8763. PmLR, 2021.

\bibitem{radford2018improving}
Alec Radford, Karthik Narasimhan, Tim Salimans, Ilya Sutskever, et~al.
\newblock Improving language understanding by generative pre-training.
\newblock 2018.

\bibitem{radford2019language}
Alec Radford, Jeffrey Wu, Rewon Child, David Luan, Dario Amodei, Ilya
  Sutskever, et~al.
\newblock Language models are unsupervised multitask learners.
\newblock {\em OpenAI blog}, 1(8):9, 2019.

\bibitem{ramesh2022hierarchical}
Aditya Ramesh, Prafulla Dhariwal, Alex Nichol, Casey Chu, and Mark Chen.
\newblock Hierarchical text-conditional image generation with clip latents.
\newblock {\em arXiv preprint arXiv:2204.06125}, 1(2):3, 2022.

\bibitem{ramesh2021zero}
Aditya Ramesh, Mikhail Pavlov, Gabriel Goh, Scott Gray, Chelsea Voss, Alec
  Radford, Mark Chen, and Ilya Sutskever.
\newblock Zero-shot text-to-image generation.
\newblock In {\em International conference on machine learning}, pages
  8821--8831. Pmlr, 2021.

\bibitem{reimers-2019-sentence-bert}
Nils Reimers and Iryna Gurevych.
\newblock Sentence-bert: Sentence embeddings using siamese bert-networks.
\newblock In {\em Proceedings of the 2019 Conference on Empirical Methods in
  Natural Language Processing}. Association for Computational Linguistics, 11
  2019.

\bibitem{salimans2016improved}
Tim Salimans, Ian Goodfellow, Wojciech Zaremba, Vicki Cheung, Alec Radford, and
  Xi~Chen.
\newblock Improved techniques for training gans.
\newblock {\em Advances in neural information processing systems}, 29, 2016.

\bibitem{sander2022residual}
Michael Sander, Pierre Ablin, and Gabriel Peyr{\'e}.
\newblock Do residual neural networks discretize neural ordinary differential
  equations?
\newblock {\em Advances in Neural Information Processing Systems},
  35:36520--36532, 2022.

\bibitem{shen2020interpreting}
Yujun Shen, Jinjin Gu, Xiaoou Tang, and Bolei Zhou.
\newblock Interpreting the latent space of gans for semantic face editing.
\newblock In {\em Proceedings of the IEEE/CVF conference on computer vision and
  pattern recognition}, pages 9243--9252, 2020.

\bibitem{sohl2015deep}
Jascha Sohl-Dickstein, Eric Weiss, Niru Maheswaranathan, and Surya Ganguli.
\newblock Deep unsupervised learning using nonequilibrium thermodynamics.
\newblock In {\em International conference on machine learning}, pages
  2256--2265. pmlr, 2015.

\bibitem{song2023consistency}
Yang Song, Prafulla Dhariwal, Mark Chen, and Ilya Sutskever.
\newblock Consistency models.
\newblock 2023.

\bibitem{song2019generative}
Yang Song and Stefano Ermon.
\newblock Generative modeling by estimating gradients of the data distribution.
\newblock {\em Advances in neural information processing systems}, 32, 2019.

\bibitem{song2020score}
Yang Song, Jascha Sohl-Dickstein, Diederik~P Kingma, Abhishek Kumar, Stefano
  Ermon, and Ben Poole.
\newblock Score-based generative modeling through stochastic differential
  equations.
\newblock {\em arXiv preprint arXiv:2011.13456}, 2020.

\bibitem{szegedy2016rethinking}
Christian Szegedy, Vincent Vanhoucke, Sergey Ioffe, Jon Shlens, and Zbigniew
  Wojna.
\newblock Rethinking the inception architecture for computer vision.
\newblock In {\em Proceedings of the IEEE conference on computer vision and
  pattern recognition}, pages 2818--2826, 2016.

\bibitem{tian2020contrastive}
Yonglong Tian, Dilip Krishnan, and Phillip Isola.
\newblock Contrastive multiview coding.
\newblock In {\em Computer Vision--ECCV 2020: 16th European Conference,
  Glasgow, UK, August 23--28, 2020, Proceedings, Part XI 16}, pages 776--794.
  Springer, 2020.

\bibitem{vahdat2020nvae}
Arash Vahdat and Jan Kautz.
\newblock Nvae: A deep hierarchical variational autoencoder.
\newblock {\em Advances in neural information processing systems},
  33:19667--19679, 2020.

\bibitem{Webb2010}
Geoffrey~I. Webb.
\newblock {\em Occam's Razor}, pages 735--735.
\newblock Springer US, Boston, MA, 2010.

\bibitem{wei2022masked}
Chen Wei, Haoqi Fan, Saining Xie, Chao-Yuan Wu, Alan Yuille, and Christoph
  Feichtenhofer.
\newblock Masked feature prediction for self-supervised visual pre-training.
\newblock In {\em Proceedings of the IEEE/CVF conference on computer vision and
  pattern recognition}, pages 14668--14678, 2022.

\bibitem{wu2018unsupervised}
Zhirong Wu, Yuanjun Xiong, Stella~X Yu, and Dahua Lin.
\newblock Unsupervised feature learning via non-parametric instance
  discrimination.
\newblock In {\em Proceedings of the IEEE conference on computer vision and
  pattern recognition}, pages 3733--3742, 2018.

\bibitem{yu2015lsun}
Fisher Yu, Ari Seff, Yinda Zhang, Shuran Song, Thomas Funkhouser, and Jianxiong
  Xiao.
\newblock Lsun: Construction of a large-scale image dataset using deep learning
  with humans in the loop.
\newblock {\em arXiv preprint arXiv:1506.03365}, 2015.

\bibitem{zhang2023adding}
Lvmin Zhang, Anyi Rao, and Maneesh Agrawala.
\newblock Adding conditional control to text-to-image diffusion models.
\newblock In {\em Proceedings of the IEEE/CVF international conference on
  computer vision}, pages 3836--3847, 2023.

\bibitem{zhang2016colorful}
Richard Zhang, Phillip Isola, and Alexei~A Efros.
\newblock Colorful image colorization.
\newblock In {\em Computer Vision--ECCV 2016: 14th European Conference,
  Amsterdam, The Netherlands, October 11-14, 2016, Proceedings, Part III 14},
  pages 649--666. Springer, 2016.

\bibitem{zhang2024mgte}
Xin Zhang, Yanzhao Zhang, Dingkun Long, Wen Xie, Ziqi Dai, Jialong Tang, Huan
  Lin, Baosong Yang, Pengjun Xie, Fei Huang, et~al.
\newblock mgte: Generalized long-context text representation and reranking
  models for multilingual text retrieval.
\newblock {\em arXiv preprint arXiv:2407.19669}, 2024.

\bibitem{zhao2024flasheval}
Lin Zhao, Tianchen Zhao, Zinan Lin, Xuefei Ning, Guohao Dai, Huazhong Yang, and
  Yu~Wang.
\newblock Flasheval: Towards fast and accurate evaluation of text-to-image
  diffusion generative models.
\newblock In {\em Proceedings of the IEEE/CVF Conference on Computer Vision and
  Pattern Recognition}, pages 16122--16131, 2024.

\bibitem{zhuang2019local}
Chengxu Zhuang, Alex~Lin Zhai, and Daniel Yamins.
\newblock Local aggregation for unsupervised learning of visual embeddings.
\newblock In {\em Proceedings of the IEEE/CVF international conference on
  computer vision}, pages 6002--6012, 2019.

\end{thebibliography}

\clearpage

\appendix
\startcontents[appendix]
\printcontents[appendix]{l}{1}{\section*{Appendix Contents}}
\clearpage

\section{More Details on Latent Computation (\cref{sec:lzn_latent_computation})}
\label{app:latent_computation}

\subsection{The Desired Properties}
\label{app:latent_computation_property}
In this section, we explain in more detail why the construction in \cref{sec:lzn_latent_computation} (approximately) satisfies the two desired properties.

\begin{packeditemize}
    \item \textbf{Prior distribution is Gaussian.} By definition, the distribution of latent $z\sim \uniformdistribution{\latent_1,\ldots,\latent_n}$ is induced by (1) drawing $\anchor\sim\uniformdistribution{\anchor_1,\ldots,\anchor_n}$ and $\epsilon\sim \normaldistribution{0}{\identity}$, (2) computing $s_{1-g}=(1-g)a+g\epsilon$, and (3) computing $\irf{s_{1-g}; 0}$. In the above process, $s_{1-g}\sim \pi_{1-g}$ by definition. Due to the property of FM discussed in \cref{sec:lzn_latent_computation}, $\irf{s_{1-g}; 0}\sim \pi_0$ when $s_{1-g}\sim \pi_{1-g}$. Therefore, the latent $z\sim \pi_0=\normaldistribution{0}{\identity}$.

    \item \textbf{Disjoint latent zones.} 
    Each latent point $z\sim \normaldistribution{0}{\identity}$ can be uniquely map to one of $\brc{\anchor_1,\ldots,\anchor_n}$ through the defined FM. We define the latent zone of $i$-th sample as the set of latents that map to $\anchor_i$: $Z_i=\brc{z: \rf{z;1}=\anchor_i}$. The probability that the latent computed through \cref{eq:latent_computation} falls in the incorrect latent zone can then be defined as $\prob{\irf{\anchor_i, \epsilon; 0} \in Z_j}$ for $i \ne j$, where the probability is over the randomness of $\epsilon \sim \normaldistribution{0}{\identity}$. We can see that this probability can be made arbitrarily small by choosing a sufficiently small $g$. This is because that to make the latent fall in the incorrect latent zone, we need to have $\brn{\epsilon}$ on the scale of $\brn{\frac{(1-g)(\anchor_i-\anchor_j)}{g}}$, whose probability $\to 0$ when $g\to 0$. To make this intuition more precise, we give the closed form of this probability for a toy one-dimensional case below.
\end{packeditemize}

\begin{theorem}
    Assume that there are $n=2$ samples $\sample_1,\sample_2$, with their anchor points $\anchor_1=-1,\anchor_2=1$. We have  
    \begin{align}
        \prob{\irf{\anchor_1, \epsilon; 0} \in Z_2}=
        \prob{\irf{\anchor_2, \epsilon; 0} \in Z_1}= \Phi\bra{\frac{g-1}{g}},
    \end{align}
    where $\Phi$ is the CDF function of the standard Gaussian distribution.
\end{theorem}
\begin{proof}
    With a slight abuse of notation, we define $\rf{s; t_1, t_2}=s+\int_{\tau=t_1}^{t_2} V(s_\tau, \tau)\diff\tau$ as following the FM trajectory from $t_1$ to $t_2$. We generalize the latent zone definition above to all time steps as the latents at time step $t$ that map to the anchor point $a_i$: $Z_i^t=\brc{z: \rf{z;t,1}=\anchor_i}$.
    Due to symmetry, we know that $Z_1^t=(-\infty,0)$ and $Z_2^t=(0,\infty)$. Therefore, we have
    \begin{align*}
        \prob{\irf{\anchor_1, \epsilon; 0} \in Z_2} = \prob{-(1-g)+g\epsilon >0} = \prob{\epsilon > \frac{1-g}{g}} = \Phi\bra{\frac{g-1}{g}},
    \end{align*}
    where $\Phi\bra{\cdot}$ is the CDF function of Gaussian distribution.
    Similarly, we can get that $\prob{\irf{\anchor_2, \epsilon; 0} \in Z_1}=\Phi\bra{\frac{g-1}{g}}$.
\end{proof}

\subsection{More Implementation Details}
\label{app:latent_computation_implementation}

We find that in the training or inference of some tasks benefit from using a more concentrated latent distribution:
\begin{align}
    \latent_i=\latentcomputation{\sample_1,\ldots,\sample_n}_i \triangleq \irf{a_i, \latentscale \epsilon_i; 0}, \label{eq:latentscale}
\end{align}
where $\epsilon_i \sim \mathcal{N}(0, \mathbf{I})$ and $\alpha\in[0,1)$ is the scaling factor. Similar techniques have been used in prior generative models for improving sample quality \cite{karras2019style,kingma2018glow,brock2018large,vahdat2020nvae}.

\subsection{Efficiency Optimizations}
\label{app:latent_computation_efficiency}
We introduce a series of efficiency optimization techniques so that the training of \nameshort{} can scale up to large models and large batch sizes.

\myparatightestn{Minibatch approximation (reducing memory and computation cost).}  
By design, latent computation requires using \emph{all} samples, which is infeasible for large datasets. In practice, we approximate this by using only the current minibatch as $\sample_1,\ldots,\sample_n$, which significantly reduces memory and computation cost. 

Note that this approximation has nuanced
 implications on the two desired properties discussed in \cref{{sec:lzn_latent_computation}}.
\begin{packeditemize}
    \item \textbf{Minibatch approximation still preserves the Gaussian prior.} \nameshort{} ensures that the latent distribution within each minibatch is approximately $\normaldistribution{0}{\identity}$. As a result, the overall latent distribution becomes a mixture of Gaussians with the same parameters, which is still $\normaldistribution{0}{\identity}$. Therefore, the global prior remains valid under the minibatch approximation. 
    \item \textbf{However, minibatch approximation violates the disjoint zone property.} This is because the latent zones of each sample now depends on other samples in the same batch, which could change across different batches. Despite this approximation, our experiments show it performs well.
    \begin{packeditemize}
        \item In generative modeling (\cref{sec:gen,sec:gen_and_class}), the latent does not need perfect zone disjointness—as long as it provides some information about the input sample, it can help reduce the variance needed to learn by the generative model (rectified flow in our case) and improve the generation quality.
        \item In representation learning (\cref{sec:repr}), latent alignment occurs \emph{within a single batch}. Thus, inconsistency across batches is irrelevant.
        \item In classification (\cref{sec:gen_and_class}), we only need to map samples \emph{within a single batch} to the latent zones of labels. Thus, inconsistency across batches is irrelevant.
    \end{packeditemize}
    That said, \emph{larger batch sizes} can improve the accuracy of latent zones and thus improve performance (\cref{sec:gen_and_class}).
\end{packeditemize}

\myparatightestn{Custom gradient checkpointing (reducing memory cost).}  
In PyTorch, forward passes store intermediate results for use in backpropagation, incurring significant memory cost. Gradient checkpointing\footnote{\url{https://pytorch.org/docs/stable/checkpoint.html}} reduces memory usage (with the cost of extra computation) by selectively discarding intermediates in the forward pass and recomputing them during the backward pass. This technique is typically applied within neural networks. In our case, we discover that the main memory bottleneck lies in latent computation, which has memory complexity $\bigO(n^2\latentdim\solversteps)$, where $n$ is the number of samples, $\latentdim$ the latent dimension, and $\solversteps$ the solver steps. We design a custom strategy that skips storing velocity computations and retains only the latent trajectories $s_t$. This reduces memory complexity to $\bigO(n\latentdim\solversteps)$, which makes the training far more manageable.

\myparatightestn{Latent parallelism (making training scalable with multi-GPU).}
For the same reason discussed above, the main computation overhead also lies in latent computation. A natural idea is to parallelize it with multi-GPU. We partition the data samples across GPUs, and each GPU computes anchor points for its assigned subset. These anchor points are then broadcast to all GPUs, allowing each to compute latents for its own samples using the complete set of anchors. To ensure that gradients can propagate back correctly through the anchor points to the originating GPUs, we use the undocumented PyTorch function \texttt{torch.distributed.nn.functional.all\_gather}, which—unlike the standard \texttt{torch.distributed.all\_gather}—maintains gradient flow to the original sources.

\FloatBarrier
\section{More Details on Latent Alignment (\cref{sec:lzn_latent_alignment})}
\label{app:latent_alignment}

\subsection{More Implementation Details}
\label{app:latent_alignment_implementation}

Optionally, we can apply a logarithm to the assignment probability to make the loss resemble a standard cross-entropy formulation.
In that case, our proposed alignment objective is:
\begin{align}
    \latentalign{\sampleset,\sampleyset} \triangleq \max \sum_{i=1}^m \max_{t \in \brc{t_\solverstepsstart,\ldots,t_\solversteps}} \log \probassign{\anchor_{\mapping_i}}{s_t^i}.
    \label{eq:align_with_log}
\end{align}

\subsection{Efficiency Optimizations}
\label{app:latent_alignment_efficiency}

We apply the same efficiency optimizations in \cref{app:latent_computation_efficiency} in latent alignment.
\FloatBarrier
\section{More Details and Results on Case Study 1}
\label{app:gen}

\subsection{Algorithm Pseudocode}
\label{app:gen_pseudocode}
\begin{figure}[t]
    \begin{minipage}{0.48\textwidth} %
\begin{algorithm}[H]
    \DontPrintSemicolon
    \LinesNumbered
	\BlankLine
	\SetKwInOut{Input}{Input}
	\SetKwInOut{Output}{Output}
	\caption{RF training}
	\Input{Training set: $\sampleset$\\
 Decoder: $\decodernotation$\\
 Number of iterations: $T$\\
 Batch size: $B$
	}
	\BlankLine
        \For{iteration $\leftarrow 1,\ldots,T$}{
            $\sample_1,\ldots,\sample_B \leftarrow$ Draw samples from $\sampleset$\\
            $\epsilon_1,\ldots,\epsilon_B \leftarrow$ Gaussian noise\\
            $t_1,\ldots,t_B \leftarrow$ Random RF timesteps\\
            $\xi_i\leftarrow (1-t_i)\epsilon_i + t_i \sample_i$\\
            Training using $\decoder{\xi_i}$
        }
\end{algorithm}
    \end{minipage}
    \hfill %
    \begin{minipage}{0.48\textwidth}
\begin{algorithm}[H]
    \DontPrintSemicolon
    \LinesNumbered
	\BlankLine
	\SetKwInOut{Input}{Input}
	\SetKwInOut{Output}{Output}
	\caption{RF+\nameshort{} training}
	\Input{Training set: $\sampleset$\\
 Decoder: $\decodernotation$\\
 \graybox{Encoder: $\encodernotation$ (used by $\latentcomputationnotation$)}\\
 Number of iterations: $T$\\
 Batch size: $B$
	}
	\BlankLine
        \For{iteration $\leftarrow 1,\ldots,T$}{
            $\sample_1,\ldots,\sample_B \leftarrow$ Draw samples from $\sampleset$\\
            \graybox{$\latent_1,\ldots,\latent_B \leftarrow \latentcomputation{\sample_1,\ldots,\sample_B}$}\\
            $\epsilon_1,\ldots,\epsilon_B \leftarrow$ Gaussian noise\\
            $t_1,\ldots,t_B \leftarrow$ Random RF timesteps\\
            $\xi_i\leftarrow (1-t_i)\epsilon_i + t_i \sample_i$\\
            Training using $\decodernotation(\xi_i;$\graybox{$\latent_i$}$)$
        }
\end{algorithm}
    \end{minipage}
    \caption{\textbf{Comparison between the training processes of RF and RF+\nameshort{}.} 
\textbf{Left:} A simplified illustration of the standard RF \cite{liu2022flow} training process. In each iteration, a batch of real samples and a batch of Gaussian noise are drawn and interpolated to produce noisy inputs, which are then passed through the decoder network to compute the loss. 
\textbf{Right:} A simplified illustration of the RF+\nameshort{} training process. The key differences are highlighted in gray: we compute \nameshort{} latents for the samples using the method in \cref{sec:lzn_latent_computation}, and provide these latents as an additional input to the RF decoder.
 }
    \label{fig:alg_gen_training}
\end{figure}
\begin{figure}[t]
    \begin{minipage}{0.48\textwidth} %
\begin{algorithm}[H]
    \DontPrintSemicolon
    \LinesNumbered
	\BlankLine
	\SetKwInOut{Input}{Input}
	\SetKwInOut{Output}{Output}
	\caption{RF generation}
	\Input{Decoder: $\decodernotation$
	}
	\BlankLine
        $\xi\leftarrow$ Gaussian noise\\
        Generated sample $\leftarrow \decoder{\xi}$
\end{algorithm}
    \end{minipage}
    \hfill %
    \begin{minipage}{0.48\textwidth}
\begin{algorithm}[H]
    \DontPrintSemicolon
    \LinesNumbered
	\BlankLine
	\SetKwInOut{Input}{Input}
	\SetKwInOut{Output}{Output}
	\caption{RF+\nameshort{} generation}
	\Input{Decoder: $\decodernotation$
	}
	\BlankLine
        $\xi\leftarrow$ Gaussian noise\\
        \graybox{$\latent\leftarrow$ Gaussian noise}\\
        Generated sample $\leftarrow \decodernotation(\xi;$\graybox{$\latent$}$)$
\end{algorithm}
    \end{minipage}
    \caption{\textbf{Comparison between the generation processes of RF and RF+\nameshort{}.}  
\textbf{Left:} A simple illustration of the standard RF generation process \cite{liu2022flow}. The decoder takes Gaussian noise as input and generates a sample. The actual process is iterative, but we leave out the steps for simplicity and only show the starting input (Gaussian noise) and the final output (the generated image).  
\textbf{Right:} A simple illustration of the RF+\nameshort{} generation process. The main differences are shown in gray: we sample extra \nameshort{} latents from Gaussian noise and use them as additional inputs to the RF decoder during the iterative generation process.
}
    \label{fig:alg_gen_gen}
\end{figure}

\cref{fig:alg_gen_training} and \cref{fig:alg_gen_gen} show side-by-side comparisons of the training and generation processes of RF and RF+\nameshort{}.

\subsection{More Implementation Details}

\myparatightestn{Architecture.}

\begin{packeditemize}
    \item \textbf{Decoder.} The only change to the RF architecture \cite{liu2022flow} is concatenating the \nameshort{} latent with the timestep embedding.
    \item \textbf{Encoder.} We extend the UNet encoder in RF \cite{liu2022flow} by connecting the output of each ResNet block with a latent transformation block. The sum of the outputs of the latent transformation blocks forms the \nameshort{} latent. Each latent transformation block consists of: (1) a $1\!\times\!1$ convolution that projects the ResNet output to 20 channels, reducing dimensionality; and (2) a small MLP with a 200-dimensional hidden layer that outputs the latent from the flattened convolution output.
\end{packeditemize}

\subsection{More Experimental Settings}

\myparatightestn{Datasets.}
\begin{packeditemize}
    \item \textbf{\cifar{} \threetwo{} \cite{krizhevsky2009learning}} contains 50000 training images and 10000 test images of 10 classes of objects. We only utilize the training set for this experiment. 
    \item \textbf{\afhqcat{} \twofivesix{} \cite{choi2020stargan}} contains 5153  cat%
    images. %
    \item \textbf{\celebahq{} \twofivesix{} \cite{karras2017progressive}} contains 30000 face images. %
    \item \textbf{\lsunbedroom{} \twofivesix{} \cite{yu2015lsun}} contains 3033042 bedroom %
    images. %
\end{packeditemize}

\myparatightestn{Metrics.}

\begin{packeditemize}
    \item \textbf{FID \cite{heusel2017gans} and sFID \cite{nash2021generating}} evaluate the similarity between real and generated images by projecting both into the latent space of a pretrained network (e.g., Inception-v3 \cite{szegedy2016rethinking}), fitting each set of latents with Gaussian distributions, and computing their Wasserstein-2 distance. The key difference between FID and sFID is the feature layer used: FID uses pooled features, while sFID uses intermediate features, making it more sensitive to spatial details.
    \item \textbf{Inception score (IS) \cite{salimans2016improved}} measures image quality by assessing both the quality of each image (how confidently a classifier predicts a class) and diversity across all images (coverage over different classes). Since the classifier is trained on \imagenet{} \cite{deng2009imagenet}, IS is best suited for natural image datasets like \cifar{}. We report IS for all datasets for completeness.
    \item \textbf{Precision and recall \cite{kynkaanniemi2019improved}} evaluate the quality and coverage of generated images.  Intuitively, precision measures the fraction of generated images that are close to real ones, while recall measures the fraction of real images that are close to the generated ones.
    \item \textbf{CMMD \cite{jayasumana2024rethinking}} measure the MMD distances between the CLIP embeddings of the real images and generated images. Compared to FID, it is reported to align better with human preference and have better sample efficiency.
    \item \textbf{Reconstruction error} measures how well a generative model can reconstruct an input image.  This reflects the model’s representational power and is crucial for applications like image editing, where edits are made by modifying the image's latent representation \cite{}.  For RF, we first apply the inverse ODE to map the image to its latent representation, then use the forward ODE to reconstruct the image, and compute the $\ell_2$ distance between the original and reconstructed images.  For RF+\nameshort{}, we add an initial step: compute the image’s \nameshort{} latent $\latentcomputation{\sampleset}$ %
    and feed it into the RF latent computation process as an additional input.

\end{packeditemize}
Following the convention \cite{liu2022flow,song2023consistency,lin2018pacgan,zhao2024flasheval,lin2021spectral}, the metrics are all computed using the training set of the dataset. 

For FID, sFID, IS, precision, recall, and CMMD, we subsample the training set and generate the same number of samples to compute the metrics. The number of samples are:
\begin{packeditemize}
    \item \cifar{}: 50000 (the whole training set).
    \item \afhqcat{}: 5120, the largest multiple of the batch size (256) that is less than or equal to the training set size (5153).
    \item \celebahq{}: 29952, the largest multiple of the batch size (256) that is less than or equal to the training set size (30000).
    \item \lsunbedroom{}: 29952, the largest multiple of the batch size (256) that is less than or equal to 30000. We limit the number of samples to 30000 so that the computation cost of the metrics are reasonable. 
\end{packeditemize}

For reconstruction error, we randomly sample a batch of images (2000 for \cifar{} and 256 for the other datasets) from the training set. Each image is reconstructed 20 times (note that the \nameshort{} latents $\latentcomputation{\sampleset}$ have randomness). We report the average metric over all reconstructions.

Note that for all the random subsampling procedures mentioned above, we ensure the sampled sets are consistent between RF and RF+\nameshort{}, so that the resulting metrics are directly comparable.

\myparatightestn{Sampler.} RF requires a sampler to numerically solve the ODE (integral) trajectory for sample generation.  
For both RF and RF+\nameshort{}, we use the RK45 sampler from RF \cite{liu2022flow}, which adaptively determines the number of steps.  
In \cref{app:gen_results}, we also analyze the effect of varying the number of sampling steps using the Euler sampler \cite{}.

\myparatightestn{Hyperparameters.} 
\begin{packeditemize}
    \item \cifar{}
    \begin{packeditemize}
        \item RF:
        \begin{packeditemize}
            \item Batch size: 2000
            \item Optimizer: Adam
            \item Decoder learning rate: 0.001
            \item Gradient clipping: 1.0
            \item Number of parameters in decoder: 61804419
        \end{packeditemize}
        \item RF+\nameshort{}:
        \begin{packeditemize}
            \item Batch size: 2000
            \item Optimizer: Adam
            \item Decoder learning rate: 0.001
            \item Encoder learning rate: 0.000025
            \item Gradient clipping: 1.0
            \item Latent dimension: 200
            \item Number of parameters in decoder: 61906819
            \item Number of parameters in encoder: 49790260
        \end{packeditemize}
    \end{packeditemize}

    \item \afhqcat{}
    \begin{packeditemize}
        \item RF:
        \begin{packeditemize}
            \item Batch size: 256
            \item Optimizer: Adam
            \item Decoder learning rate: 0.0002
            \item Gradient clipping: 1.0
            \item Number of parameters in decoder: 65574549
        \end{packeditemize}
        \item RF+\nameshort{}:
        \begin{packeditemize}
            \item Batch size: 256
            \item Optimizer: Adam
            \item Decoder learning rate: 0.0002
            \item Encoder learning rate: 0.000002
            \item Gradient clipping: 1.0
            \item Latent dimension: 200
            \item Number of parameters in decoder: 65676949
            \item Number of parameters in encoder: 87768896
        \end{packeditemize}
    \end{packeditemize}

    \item \celebahq{}
    \begin{packeditemize}
        \item RF:
        \begin{packeditemize}
            \item Batch size: 256
            \item Optimizer: Adam
            \item Decoder learning rate: 0.0002
            \item Gradient clipping: 1.0
            \item Number of parameters in decoder: 65574549
        \end{packeditemize}
        \item RF+\nameshort{}:
        \begin{packeditemize}
            \item Batch size: 256
            \item Optimizer: Adam
            \item Decoder learning rate: 0.0002
            \item Encoder learning rate: 0.000004
            \item Gradient clipping: 1.0
            \item Latent dimension: 200
            \item Number of parameters in decoder: 65676949
            \item Number of parameters in encoder: 87768896
        \end{packeditemize}
    \end{packeditemize}

    \item \lsunbedroom{}
    \begin{packeditemize}
        \item RF:
        \begin{packeditemize}
            \item Batch size: 256
            \item Optimizer: Adam
            \item Decoder learning rate: 0.0002
            \item Gradient clipping: 1.0
            \item Number of parameters in decoder: 65574549
        \end{packeditemize}
        \item RF+\nameshort{}:
        \begin{packeditemize}
            \item Batch size: 256
            \item Optimizer: Adam
            \item Decoder learning rate: 0.0002
            \item Encoder learning rate: 0.000002
            \item Gradient clipping: 1.0
            \item Latent dimension: 200
            \item Number of parameters in decoder: 65676949
            \item Number of parameters in encoder: 87768896
        \end{packeditemize}
    \end{packeditemize}
    
\end{packeditemize}

\myparatightestn{Computation cost.} 
Excluding the computation cost of periodic evaluation (i.e., only counting the computation cost of model training), each RF+\nameshort{} experiment takes:
\begin{packeditemize}
    \item \cifar{}: 8 hours on 16 A100 (40 GB) GPUs.
    \item \afhqcat{}: 10 hours on 32 A100 (40 GB) GPUs.
    \item \celebahq{}: 58 hours on 32 A100 (40 GB) GPUs.
    \item \lsunbedroom{}: 341 hours on 32 A100 (40 GB) GPUs.
\end{packeditemize}

\subsection{More Results}
\label{app:gen_results}

\myparatightestn{Generated images.} The generated images of RF and RF+\nameshort{} are in \cref{fig:gen_images_rf_cifar,fig:gen_images_lzn_cifar,fig:gen_images_rf_afhqcat,fig:gen_images_lzn_afhqcat,fig:gen_images_rf_celebahq,fig:gen_images_lzn_celebahq,fig:gen_images_rf_lsunbedroom,fig:gen_images_lzn_lsunbedroom}.

\begin{figure}[t]
    \centering
    \includegraphics[width=0.75\linewidth]{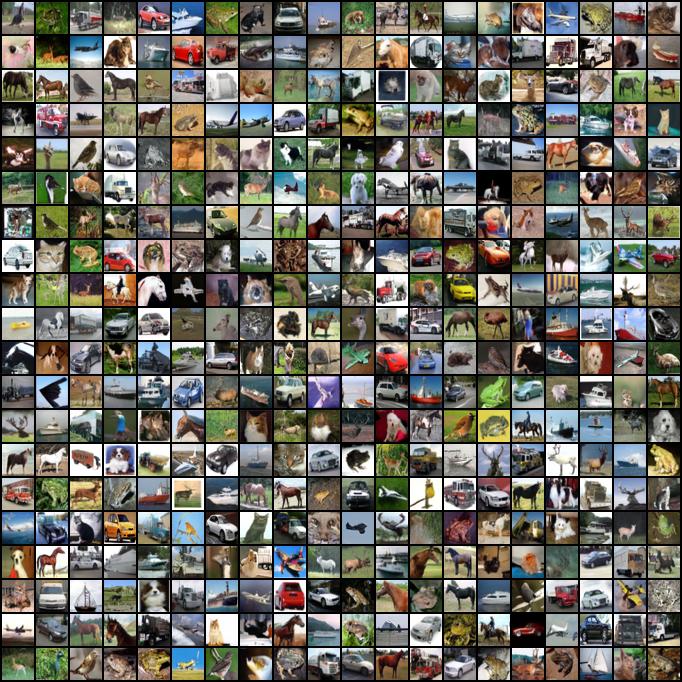}
    \caption{Generated images of RF on \cifar{}.}
    \label{fig:gen_images_rf_cifar}
\end{figure}

\begin{figure}[t]
    \centering
    \includegraphics[width=0.75\linewidth]{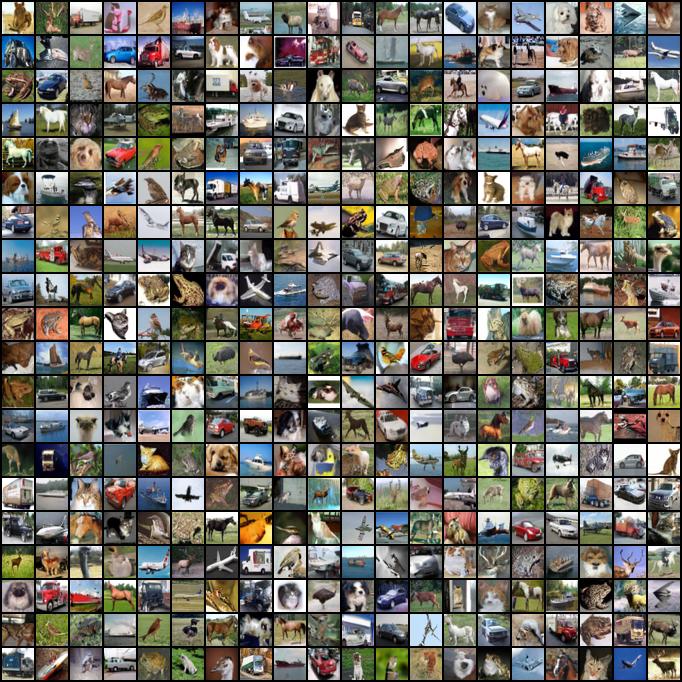}
    \caption{Generated images of RF+\nameshort{} on \cifar{}.}
    \label{fig:gen_images_lzn_cifar}
\end{figure}

\begin{figure}[t]
    \centering
    \includegraphics[width=0.75\linewidth]{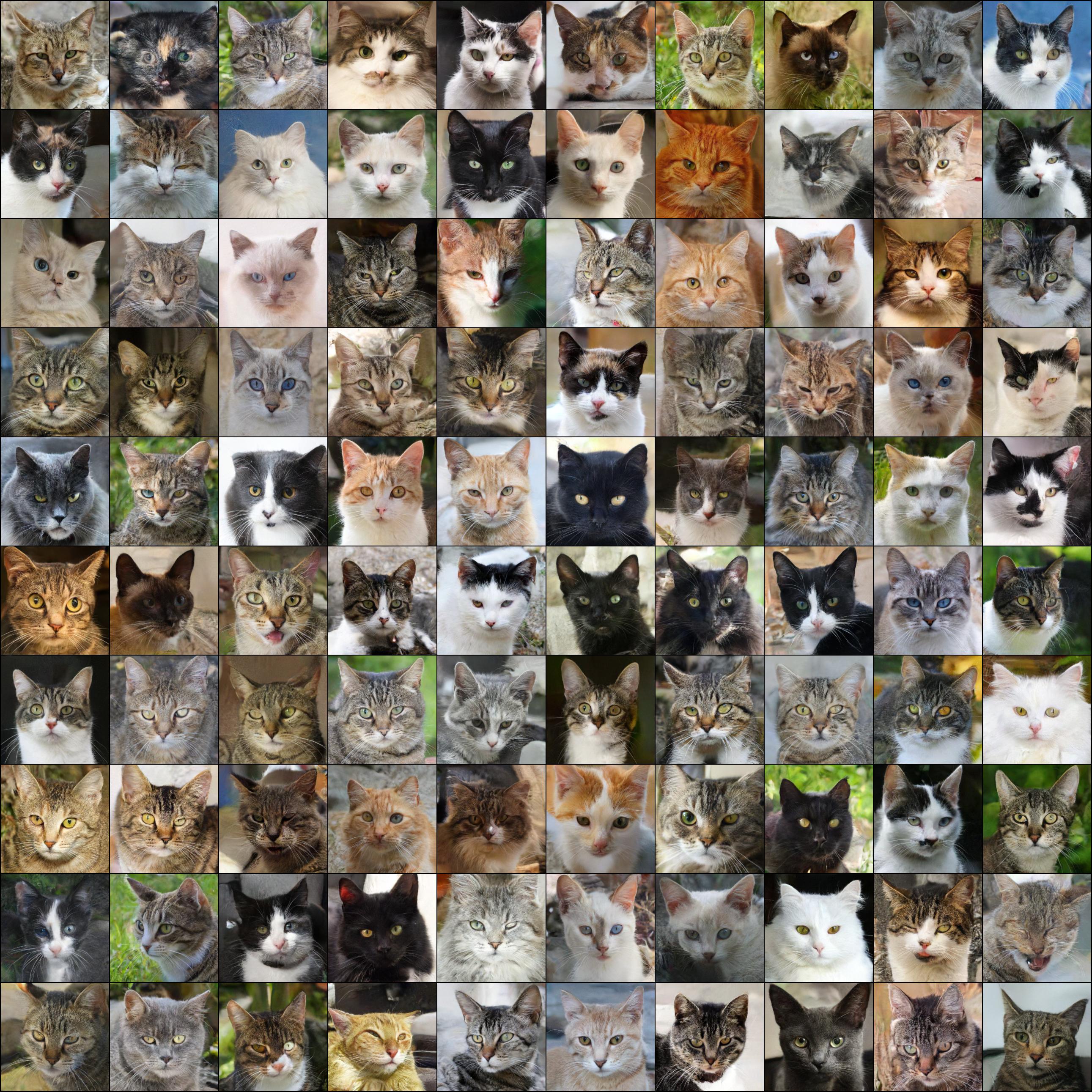}
    \caption{Generated images of RF on \afhqcat{}.}
    \label{fig:gen_images_rf_afhqcat}
\end{figure}

\begin{figure}[t]
    \centering
    \includegraphics[width=0.75\linewidth]{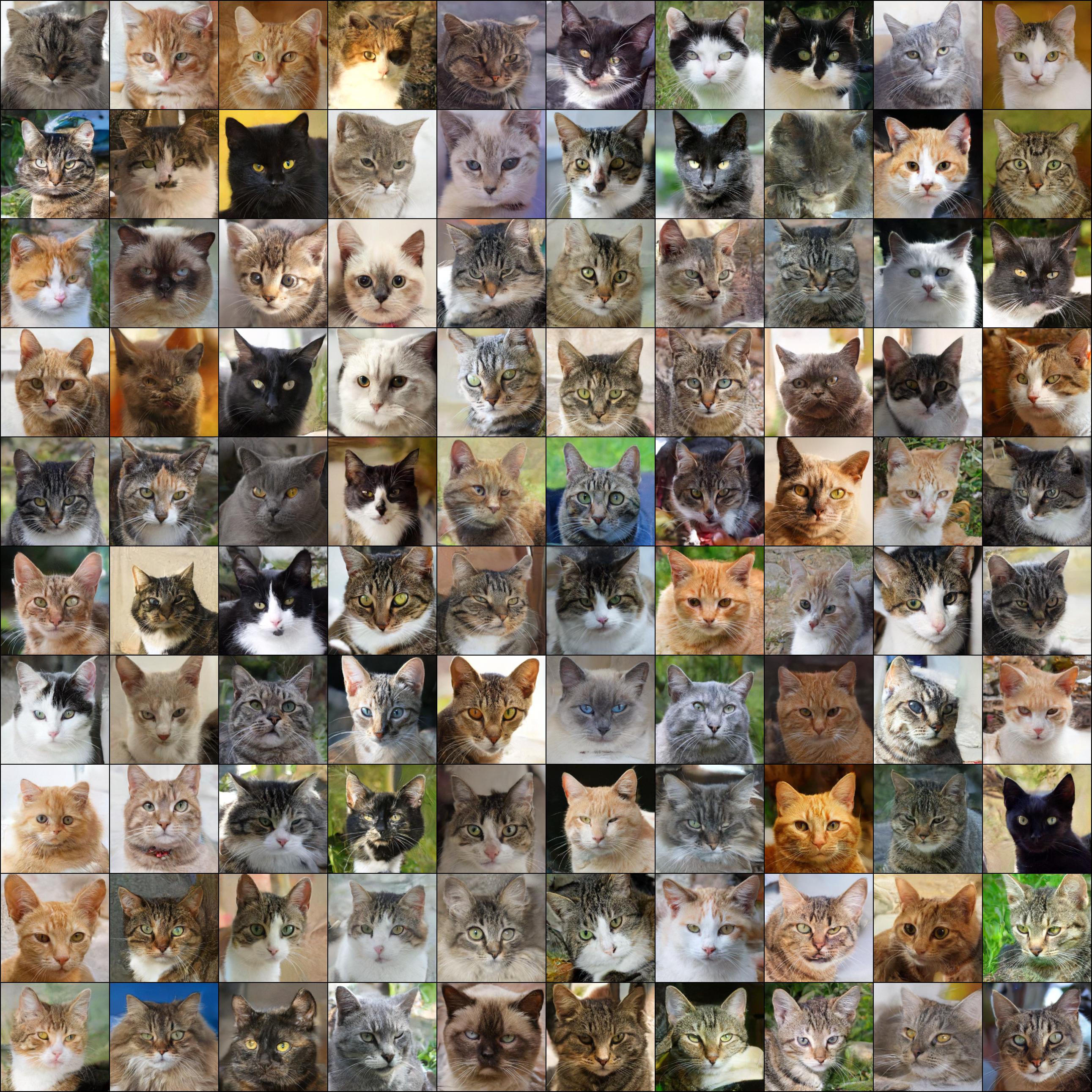}
    \caption{Generated images of RF+\nameshort{} on \afhqcat{}.}
    \label{fig:gen_images_lzn_celebahq}
\end{figure}

\begin{figure}[t]
    \centering
    \includegraphics[width=0.75\linewidth]{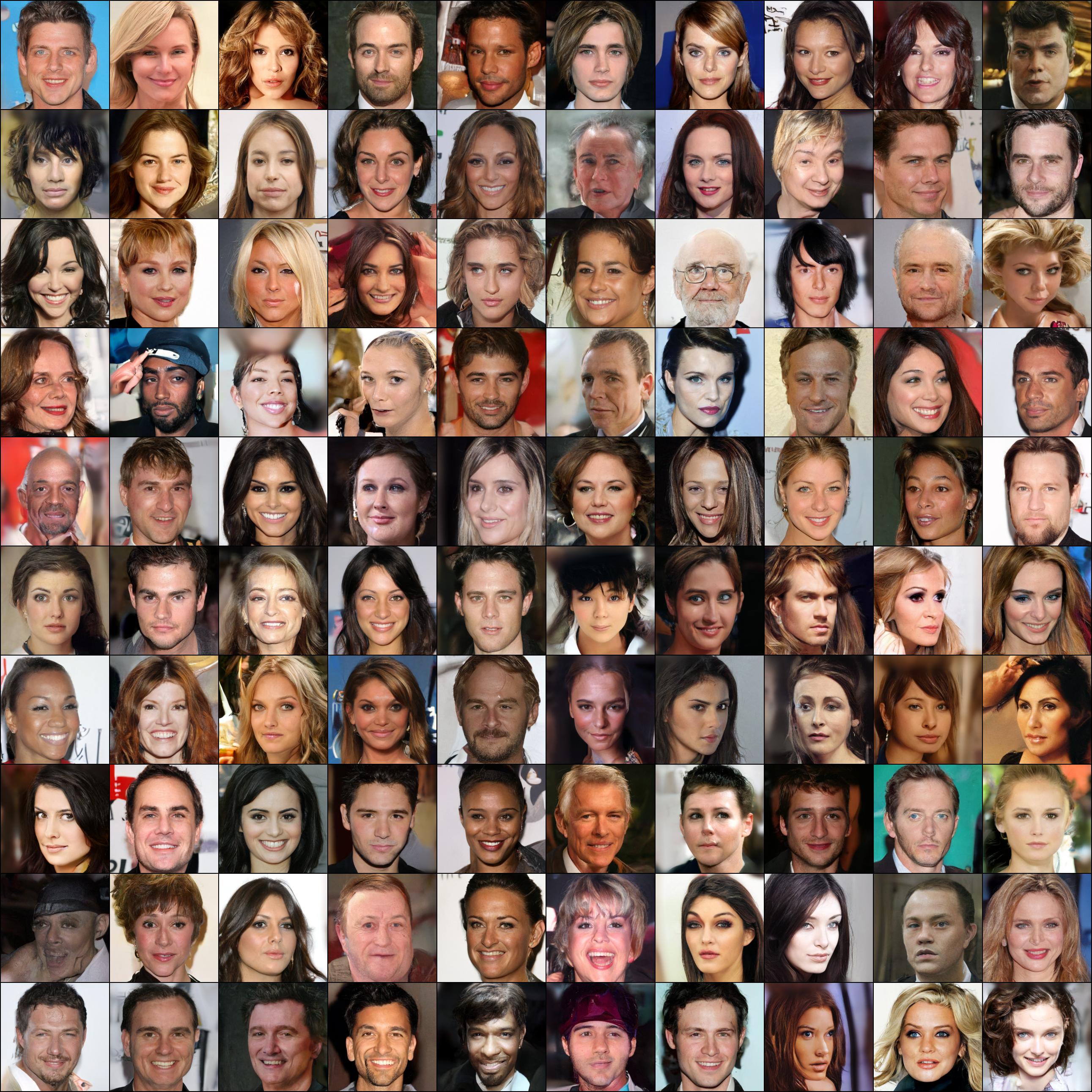}
    \caption{Generated images of RF on \celebahq{}.}
    \label{fig:gen_images_rf_celebahq}
\end{figure}

\begin{figure}[t]
    \centering
    \includegraphics[width=0.75\linewidth]{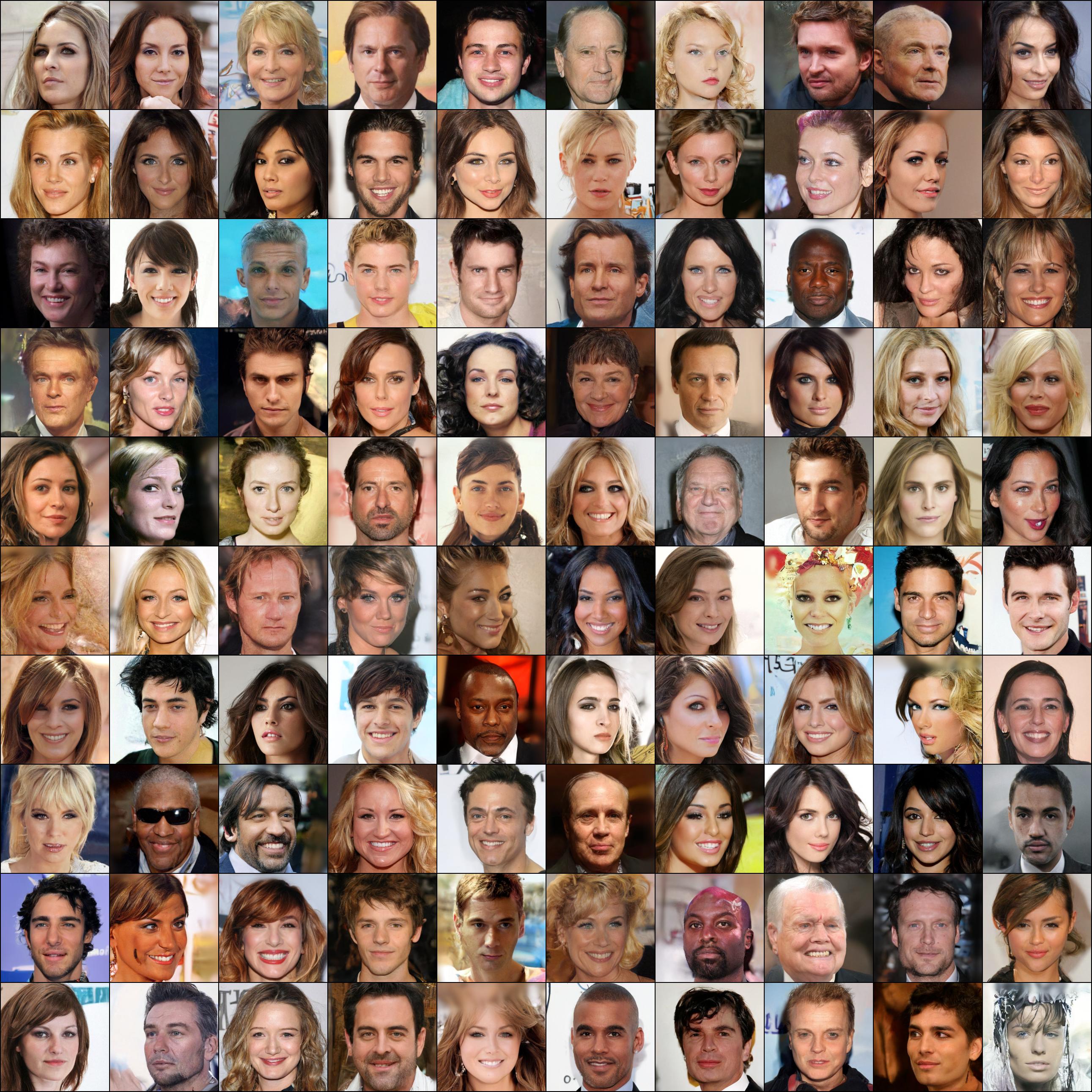}
    \caption{Generated images of RF+\nameshort{} on \celebahq{}.}
    \label{fig:gen_images_lzn_afhqcat}
\end{figure}

\begin{figure}[t]
    \centering
    \includegraphics[width=0.75\linewidth]{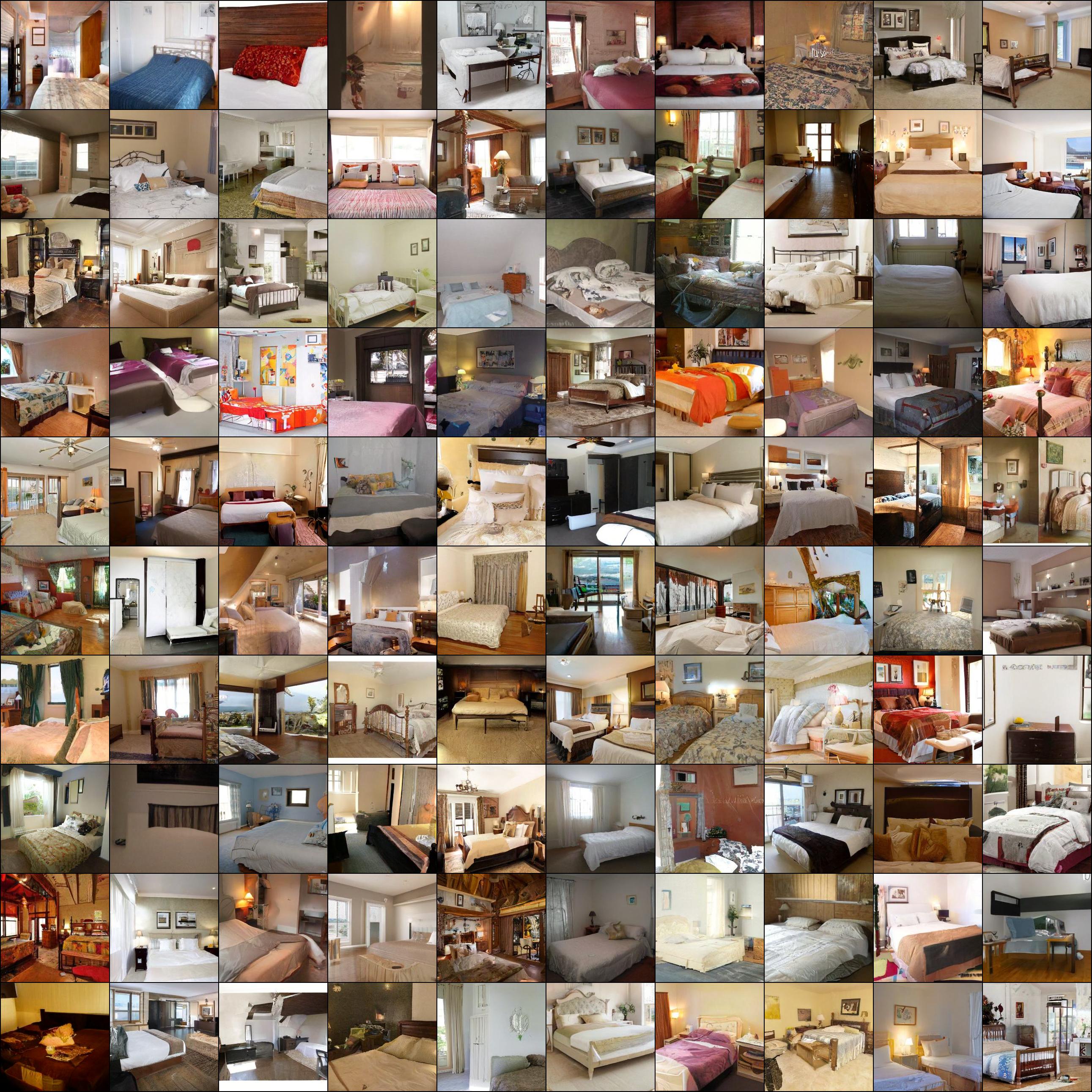}
    \caption{Generated images of RF on \lsunbedroom{}.}
    \label{fig:gen_images_rf_lsunbedroom}
\end{figure}

\begin{figure}[t]
    \centering
    \includegraphics[width=0.75\linewidth]{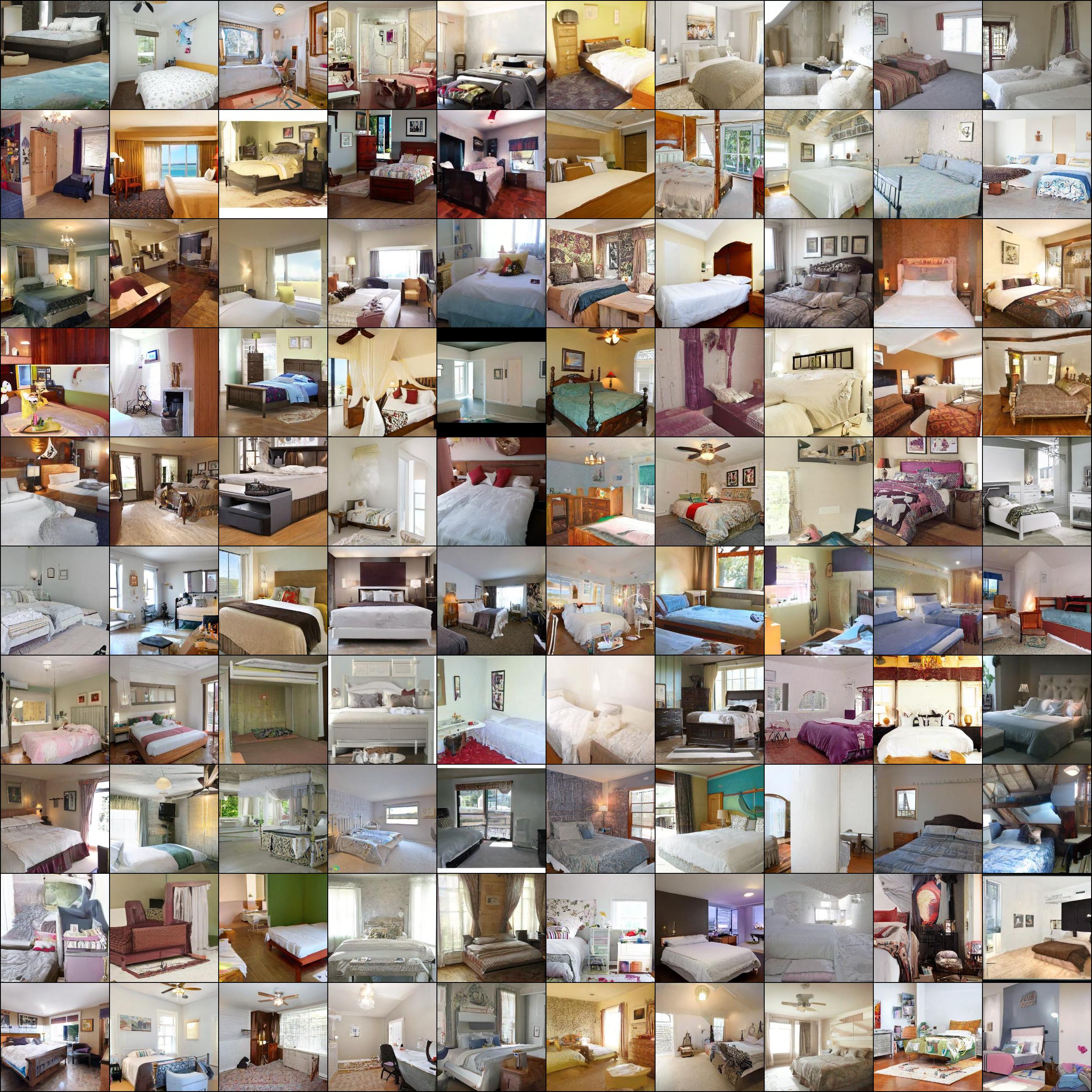}
    \caption{Generated images of RF+\nameshort{} on \lsunbedroom{}.}
    \label{fig:gen_images_lzn_lsunbedroom}
\end{figure}

\myparatightestn{Ablation studies on FID implementation.}
It is known that subtle differences in FID implementation can result in different results \cite{parmar2022aliased}. In our main experiments, we use the implementation in consistency models \cite{song2023consistency}. In \cref{tab:gen_fid}, we additionally show the FID using two other implementations: RF \cite{liu2022flow} and clean FID \cite{parmar2022aliased}. We can see that, while the numbers are different, the relative ranking across all three implementations is consistent. Especially, RF+\nameshort{} achieves the best FID in three out of four datasets.

\begin{table*}[t]
\small
    \centering
    \caption{FID with different implementations for unconditional image generation. ``CM'' denotes consistency models \cite{song2023consistency}; ``RF'' denotes Rectified Flow \cite{liu2022flow}; ``clean'' denotes clean FID \cite{parmar2022aliased}. The best results are in \graybox{gray box}.}
    \label{tab:gen_fid}
    \setlength{\tabcolsep}{3.8pt}
    \resizebox{0.75\textwidth}{!}{
    \begin{tabular}{l|ccc|ccc}
    \toprule
    \multirow{2}{*}{Algo.} & \multicolumn{3}{c|}{{\cifar{} \threetwo{}}} & \multicolumn{3}{c}{{\afhqcat{} \twofivesix{}}}\\
    \cline{2-7}
     & FID (clean){\color{red}$\downarrow$} & 
     FID (RF){\color{red}$\downarrow$} & FID (CM){\color{red}$\downarrow$} & FID (clean){\color{red}$\downarrow$} & 
     FID (RF){\color{red}$\downarrow$} & FID (CM){\color{red}$\downarrow$}\\
    \hline
    RF &   3.18 & 2.77 & 2.76  &  5.99 & 6.20 & 6.08  \\
    RF+\nameshort{}  &   \graybox{3.05} & \graybox{2.61} & \graybox{2.59}  &  \graybox{5.66} & \graybox{5.69} & \graybox{5.68} \\
    \bottomrule
    \toprule
    \multirow{2}{*}{Algo.} & \multicolumn{3}{c|}{{\celebahq{} \twofivesix{}}} & \multicolumn{3}{c}{{\lsunbedroom{} \twofivesix{}}}\\
    \cline{2-7}
     & FID (clean){\color{red}$\downarrow$} & 
     FID (RF){\color{red}$\downarrow$} & FID (CM){\color{red}$\downarrow$} & FID (clean){\color{red}$\downarrow$} & 
     FID (RF){\color{red}$\downarrow$} & FID (CM){\color{red}$\downarrow$}\\
    \hline
    RF &   \graybox{7.10} & \graybox{7.00} & \graybox{6.95}  &  6.39 & 6.25 & 6.25   \\
    RF+\nameshort{}  &   7.31 & 7.23 & 7.17   &  \graybox{5.88} & \graybox{5.87} & \graybox{5.95} \\
    \bottomrule
\end{tabular}
}
\end{table*}

\myparatightestn{Ablation studies on sampling steps.}  
In this experiment, we use the Euler sampler with varying numbers of sampling steps. As shown in \cref{fig:gen_fid_vs_nfe}, RF+\nameshort{} generally achieves better FID than the RF baseline across most settings. Notably, in the only case where RF+\nameshort{} performs worse than RF in \cref{tab:gen}, we observe that the underperformance occurs only at the highest number of sampling steps in the Euler sampler (\cref{fig:gen_fid_vs_nfe_celebahq}).

\begin{figure*}[!t]
\centering

\begin{subfigure}{.45\textwidth}
\includegraphics[width=1\textwidth]{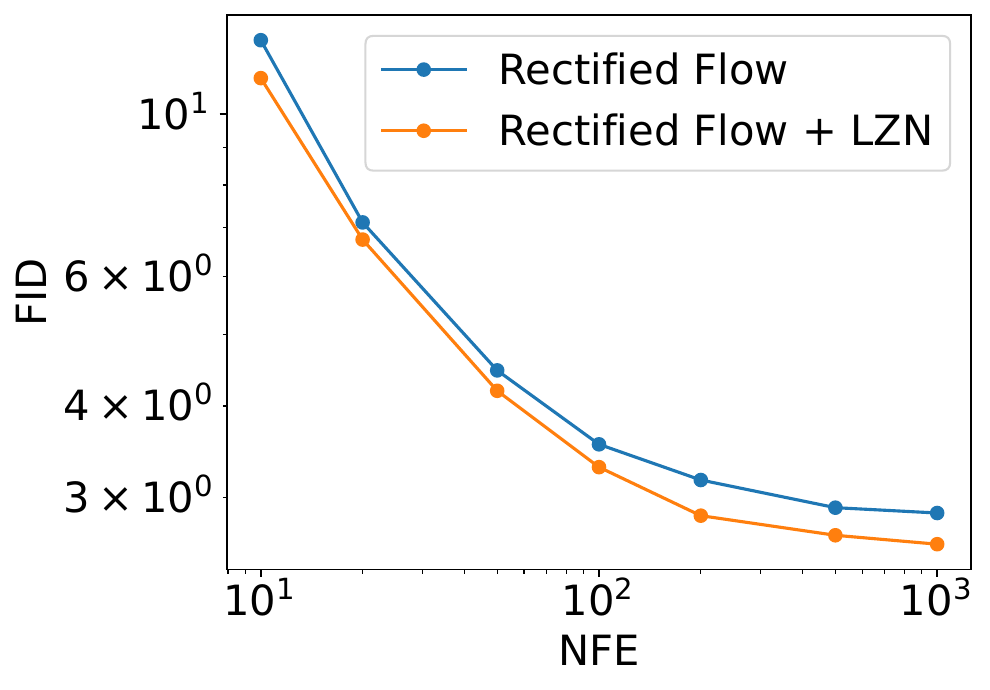}
\caption{\cifar{}.}
\label{fig:gen_fid_vs_nfe_cifar}
\end{subfigure}
\begin{subfigure}{.45\textwidth}
\includegraphics[width=1\textwidth]{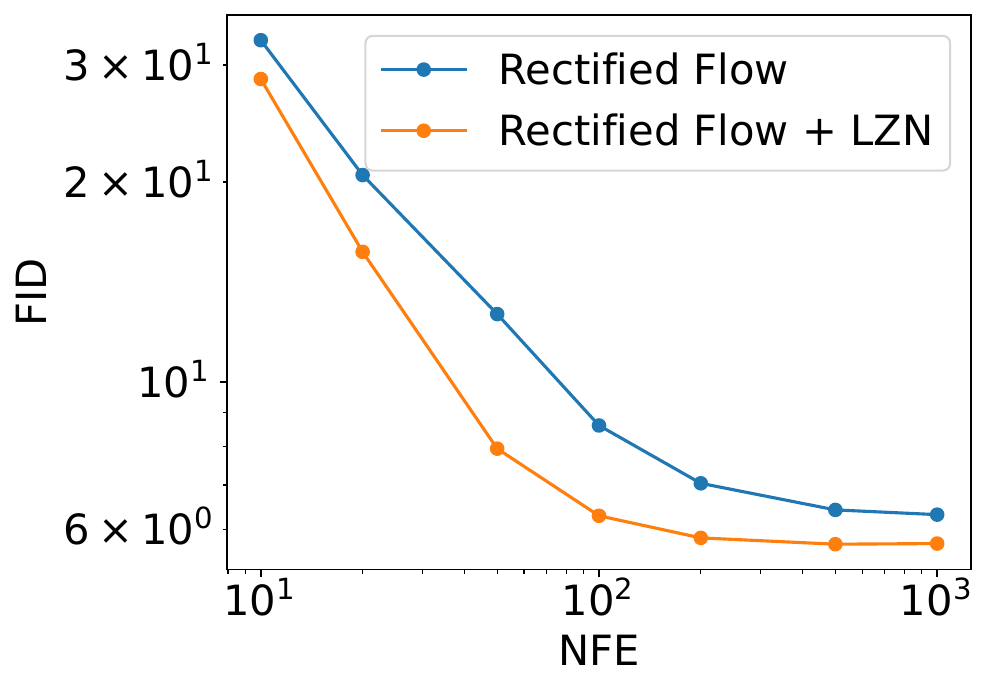}
\caption{\afhqcat{}.}
\label{fig:gen_fid_vs_nfe_afhqcat}
\end{subfigure}
\begin{subfigure}{.45\textwidth}
\includegraphics[width=1\textwidth]{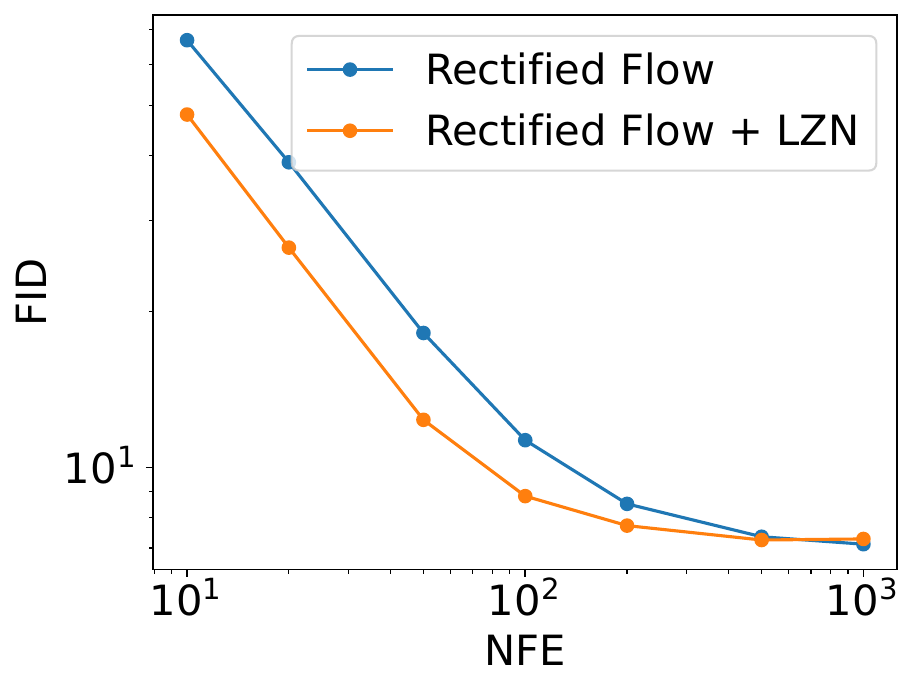}
\caption{\celebahq{}.}
\label{fig:gen_fid_vs_nfe_celebahq}
\end{subfigure}
\begin{subfigure}{.45\textwidth}
\includegraphics[width=1\textwidth]{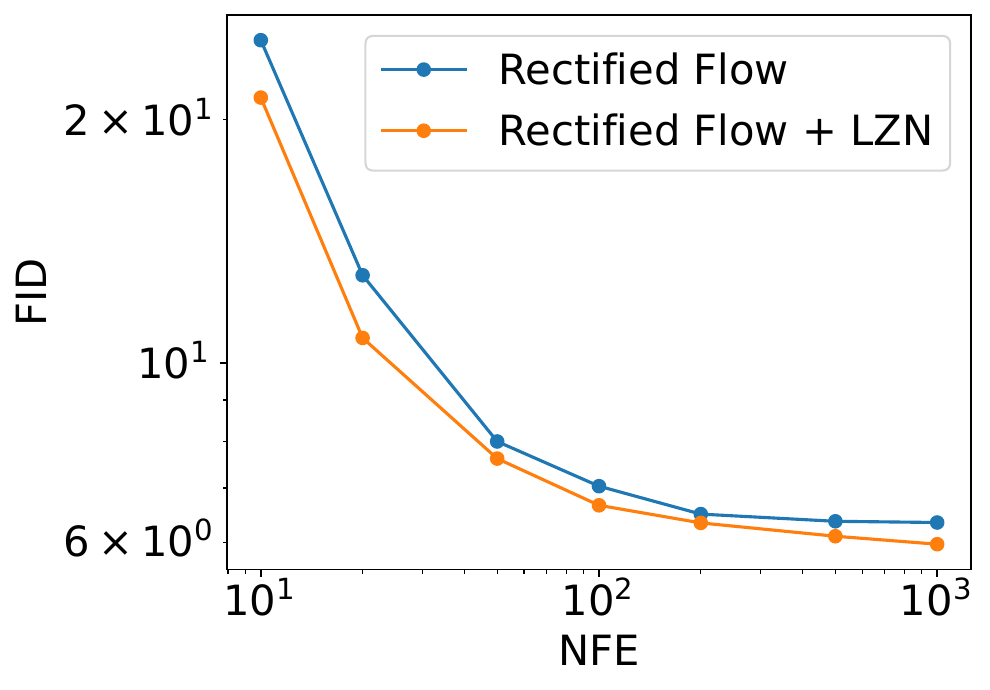}
\caption{\lsunbedroom{}.}
\label{fig:gen_fid_vs_nfe_lsunbedroom}
\end{subfigure}
\caption{FID vs. number of sampling steps in the Euler sampler. RF+\nameshort{} outperforms RF in most cases.}
\label{fig:gen_fid_vs_nfe}
\end{figure*}

\FloatBarrier
\section{More Details and Results on Case Study 2}
\label{app:repr}

\subsection{Algorithm Pseudocode}
\label{app:repr_pseudocode}
\begin{algorithm}[t]
    \DontPrintSemicolon
    \LinesNumbered
	\BlankLine
	\SetKwInOut{Input}{Input}
	\SetKwInOut{Output}{Output}
	\caption{Unsupervised representation learning with \nameshort{}}
    \label{alg:repr_training}
	\Input{Training set: $\sampleset$\\
 Encoder: $\encodernotation$\\
 Number of iterations: $T$\\
 Batch size: $B$
	}
	\BlankLine
        \For{iteration $\leftarrow 1,\ldots,T$}{
            $\sample_1,\ldots,\sample_B \leftarrow$ Draw samples from $\sampleset$\\
            $\sample'_i,\sample''_i\leftarrow$ Two random augmentations of $\sample_i$\\
            Training $\encodernotation$ using $\latentalign{\brc{\sample'_1,\ldots,\sample'_B},\brc{\sample''_1,\ldots,\sample''_B}}$
        }
\end{algorithm}

\cref{alg:repr_training} shows the pseudocode of the training process.

After training, the encoder $\encodernotation$ can be used to obtain image representations. We provide several strategies for extracting these representations. Please see \cref{app:repr_choice} for details.

\subsection{More Implementation Details}

\myparatightestn{Architecture.}
To remain consistent with prior work \cite{he2020momentum,chen2020simple,grill2020bootstrap}, we use the ResNet-50 architecture \cite{he2016deep} as the encoder for \nameshort{}. The only modification we make is replacing all batch normalization layers with group normalization. However, our early experiments indicate that this change does not lead to significant performance differences.

For the projection head following the ResNet-50 output, we use an MLP with one hidden layer, as in \cite{chen2020simple,chen2020big}.

\myparatightestn{Data augmentation.}
We follow the same data augmentation strategy as in \cite{chen2020big}.

\myparatightestn{Representation.}
Following prior work \cite{chen2020simple,chen2020big}, after training the ResNet-50, we discard the projection head and use only the ResNet-50 backbone to extract representations for training the linear classifier. As a result, obtaining representations from \nameshort{} in this way does not require going through the \nameshort{} latent computation process $\latentcomputationnotation$, and thus has the same computational efficiency as baseline methods.

\myparatightestn{Objective.} We use the version with log (\cref{eq:align_with_log}).

\subsection{More Experimental Settings} 
\myparatightestn{Datasets.} We use the \imagenet{} dataset, which contains 1281167 training images and 50000 validation images. \nameshort{} is trained on the training set, and classification accuracy is evaluated on the validation set.

\myparatightestn{Hyperparameters.} 
\begin{packeditemize}
    \item Batch size: 8192 %

    \item Optimizer: Adam
    \item Learning rate: 8e-4
    \item Gradient clipping: 1.0
    \item Latent dimension: 256
    \item Number of parameters: 24032832
    \item $\latentscale$: $0.45$
\end{packeditemize}

\myparatightestn{Computation cost.} 
Excluding the computation cost of periodic evaluation (i.e., only counting the computation cost of model training), each \nameshort{} experiment takes 1800 hours on 128 A100 (40 GB) GPUs.

\subsection{More Results}
\label{app:repr_choice}

\myparatightestn{More baselines.} \cref{tab:repr_full} shows the result with more baselines that are not using the ResNet-50 architecture.
\begin{table}[t]
    \centering
  \centering
    \caption{Classification accuracy on \imagenet{} by training a linear classifier on the unsupervised representations. Methods with \textsuperscript{\S} are based on contrastive learning.\protect\footnotemark{} The horizontal line separates baselines that perform worse or better than our \nameshort{}. ``R'' means ``ResNet''.
}
    \label{tab:repr_full}
    \resizebox{0.6\linewidth}{!}{
    \begin{tabular}{c|c|c|c}
    \toprule
       Algorithm  & Architecture & Top-1 Acc{\color{blue}$\uparrow$} &  Top-5 Acc{\color{blue}$\uparrow$} \\\midrule
       Colorization \cite{zhang2016colorful}                         & R101            & 39.6 \cite{he2020momentum}         & -\\
       Jigsaw \cite{noroozi2016unsupervised}                         & R50w2$\times$   & 44.6 \cite{he2020momentum}         & - \\
       Exemplar \cite{dosovitskiy2014discriminative}                 & R50w3$\times$   & 46.0 \cite{he2020momentum}         & - \\
       DeepCluster \cite{caron2018deep}                              & VGG             & 48.4 \cite{he2020momentum}         & - \\
       CPC v1\textsuperscript{\S} \cite{oord2018representation}      & R101            & 48.7 \cite{he2020momentum}         & - \\
       RelativePosition \cite{doersch2015unsupervised}               & R50w2$\times$   & 51.4 \cite{he2020momentum}         & - \\
       InstDisc\textsuperscript{\S} \cite{wu2018unsupervised}        & R50             & 54.0 \cite{he2020momentum}         & - \\
       Rotation \cite{gidaris2018unsupervised}                       & Rv50w4$\times$  & 55.4 \cite{he2020momentum}         & - \\
       BigBiGAN \cite{donahue2019large}                              & R50             & 56.6 \cite{he2020momentum}         & - \\
       LocalAgg\textsuperscript{\S} \cite{zhuang2019local}           & R50             & 58.8 \cite{he2020momentum}         & - \\
       MoCo\textsuperscript{\S} \cite{he2020momentum}                & R50             & 60.2 \cite{chen2020simple}         & - \\
       BigBiGAN \cite{donahue2019large}                              & Rv50w4$\times$  & 61.3 \cite{he2020momentum}         & 81.9 \cite{chen2020simple}\\
       PIRL\textsuperscript{\S} \cite{misra2020self}                 & R50             & 63.6 \cite{chen2020simple}         & - \\
       CPC v2\textsuperscript{\S} \cite{henaff2020data}              & R50             & 63.8 \cite{chen2020simple}         & 85.3 \cite{chen2020simple}\\
       CMC\textsuperscript{\S} \cite{tian2020contrastive}            & R50             & 66.2 \cite{grill2020bootstrap}     & 87.0 \cite{grill2020bootstrap} \\
        SimSiam\textsuperscript{\S}\cite{chen2021exploring} & R50 &  68.1 \cite{chen2021exploring} & - \\
       SimCLR\textsuperscript{\S} \cite{chen2020simple}              & R50             & 69.3 \cite{chen2020simple}         & 89.0 \cite{chen2020simple}\\\hline
       MoCo v2\textsuperscript{\S} \cite{chen2020improved}           & R50             & 71.7 \cite{chen2020improved}       & - \\
       SimCLR v2\textsuperscript{\S} \cite{chen2020big}              & R50             & 71.7 \cite{chen2020big}            & - \\
       BYOL\textsuperscript{\S} \cite{grill2020bootstrap}            & R50             & 74.3 \cite{grill2020bootstrap}     & 91.6 \cite{grill2020bootstrap}  \\
       DINO\textsuperscript{\S} \cite{caron2021emerging}  & R50           & 75.3 \cite{caron2021emerging}     & -  \\
       DINO\textsuperscript{\S} \cite{caron2021emerging}  & ViT-S           & 77.0 \cite{caron2021emerging}     & -  \\
       DINO\textsuperscript{\S} \cite{caron2021emerging}  & ViT-B/16           & 78.2 \cite{caron2021emerging}     & -  \\
       DINO\textsuperscript{\S} \cite{caron2021emerging}  & ViT-S/8           & 79.7 \cite{caron2021emerging}     & -  \\
       DINO\textsuperscript{\S} \cite{caron2021emerging}  & ViT-B/8           & 80.1 \cite{caron2021emerging}     & -\\
       I-JPEA \cite{assran2023self}  & ViT-B/16          & 72.9 \cite{assran2023self}     & -\\
       I-JPEA \cite{assran2023self}  & ViT-L/16          & 77.5 \cite{assran2023self}     & -\\
       I-JPEA \cite{assran2023self}  & ViT-H/14          & 79.3 \cite{assran2023self}     & -\\
       I-JPEA \cite{assran2023self}  & ViT-H/$16_{448}$          & 81.1 \cite{assran2023self}     & -\\\hline\hline
       \rowcolor{lightgray!60}
       \nameshort{}                                                  & R50             & 69.5 & 89.3\\\bottomrule
    \end{tabular}}
\end{table}
\footnotetext{Note that we use the term \emph{contrastive learning} broadly to refer not only to methods employing the traditional contrastive loss, but to all approaches that encourage relevant images to share similar representations; see \cref{sec:repr}.}

\begin{figure}[t]
    \centering
    \includegraphics[width=0.5\linewidth]{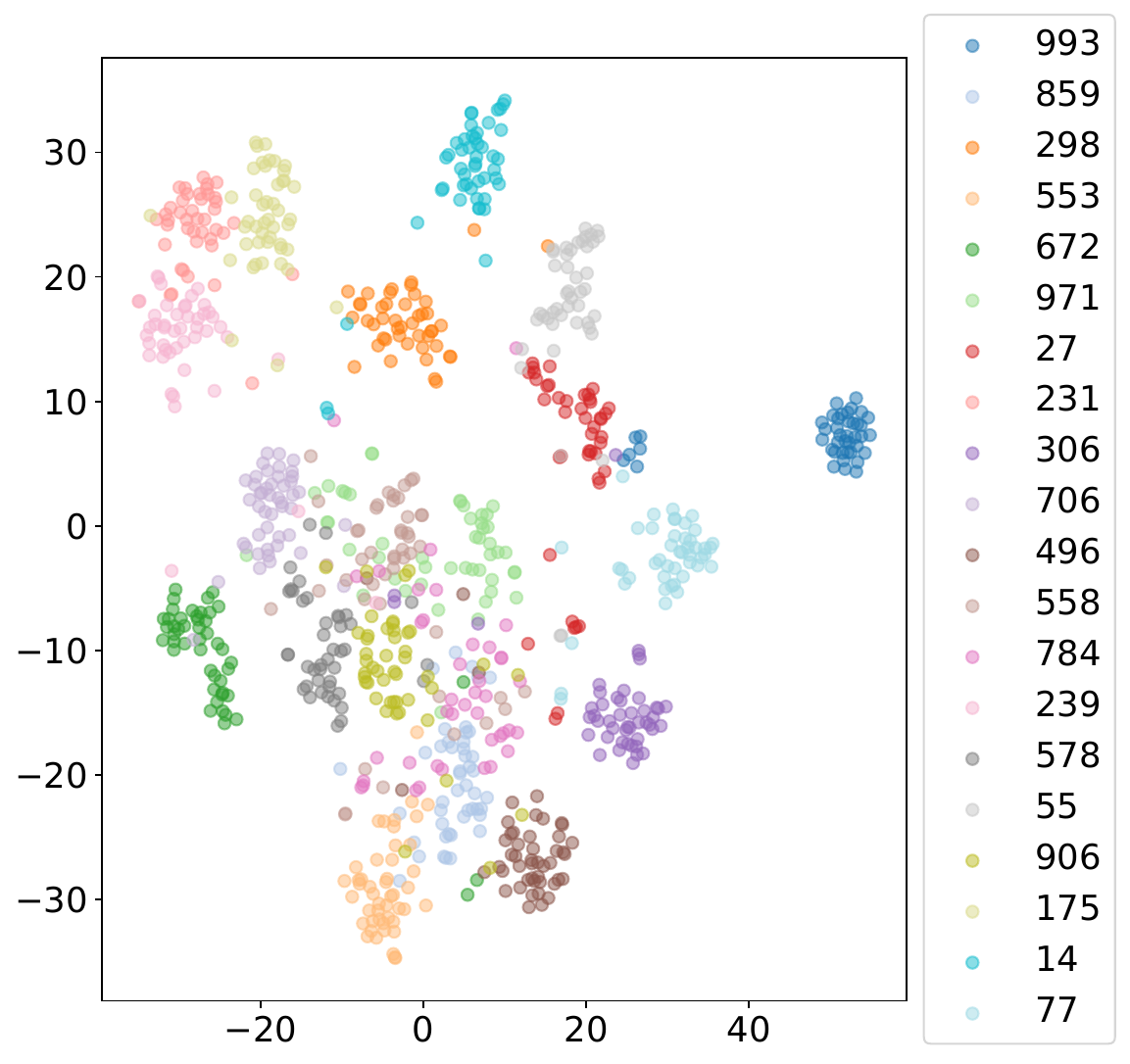}
    \caption{t-SNE visualization of \nameshort{} representations projected into 2D for 20 randomly selected ImageNet validation classes. Images from the same class form distinct clusters, indicating that \nameshort{} learns meaningful image representations.
  }
    \label{fig:tsne}
\end{figure}
\myparatightestn{Visualizing the learned representations.} 
We take images from randomly selected 20 classes from the validation set of ImageNet and computed their embeddings using the trained \nameshort{} model. We chose the validation set to ensure that the results are not influenced by training set overfitting. We then projected these embeddings into a 2D space using t-SNE--a widely used method for visualizing high-dimensional representations, following seminal works such as SimCLR \cite{chen2020simple}.
The resulting t-SNE plot is in \cref{fig:tsne}. We can see that the samples from different classes are well-clustered.  This suggests LZN learns meaningful image representations.

\myparatightestn{Ablation studies on representation choice.}
Prior work \cite{chen2020simple,chen2020big} has shown that the choice of feature extraction layer significantly affects downstream performance. In particular, removing the projection head often improves results. Motivated by this, we explore various feature extraction strategies for \nameshort{}, which offers more flexibility due to its unique latent computation process (\cref{sec:lzn_latent_computation}). Specifically, we compare the following methods:
\begin{packeditemize}
    \item \textbf{With latent.} Use the latent representation from \nameshort{} (see \cref{sec:lzn_latent_computation}) to train the classifier.
    \item \textbf{With latent ($\latentscale=0$).} Same as above but with $\latentscale=0$ in the latent computation.
    \item \textbf{With head.} Use the anchor point (i.e., encoder output before FM computation).
    \item \textbf{Without head.} Use the ResNet backbone output (before the projection head).
\end{packeditemize}
The first two methods are specific to \nameshort{}, while the last two follow the design commonly used in prior contrastive learning work \cite{chen2020simple,chen2020big}. 

The results are shown in \cref{fig:repr_ablation_head}. We observe the following:

\begin{packeditemize}
    \item \nameshort{} latent achieves the lowest prediction accuracy. As discussed in \cref{app:latent_computation_implementation}, the \nameshort{} latents are reliable only when computed over the full dataset. However, for efficiency, both training and inference rely on minibatches to approximate the latent representations. This approximation increases the size of the latent zones, leading to potential overlap between the zones of different samples across batches, which inevitably degrades downstream classification performance.

    \item In comparison, \nameshort{} with $\latentscale=0$ yields significantly higher accuracy. This improvement can be attributed to the reduced likelihood of overlap between latent zones when $\latentscale=0$, making the resulting representations more distinct and less noisy.

    \item The final two methods, ``with head'' and ``without head'', do not involve latent computation and are therefore more efficient. Consistent with findings from prior contrastive learning studies \cite{chen2020simple}, we observe that ``without head’’ performs substantially better. As explained in \cite{chen2020simple}, the projection head often discards important information—such as types of data augmentation—in order to minimize the training loss. In contrast, layers preceding the head might retain richer and more discriminative features, which are more useful for downstream classification tasks.
\end{packeditemize}

\begin{figure}[t]
    \centering
    \includegraphics[width=0.5\linewidth]{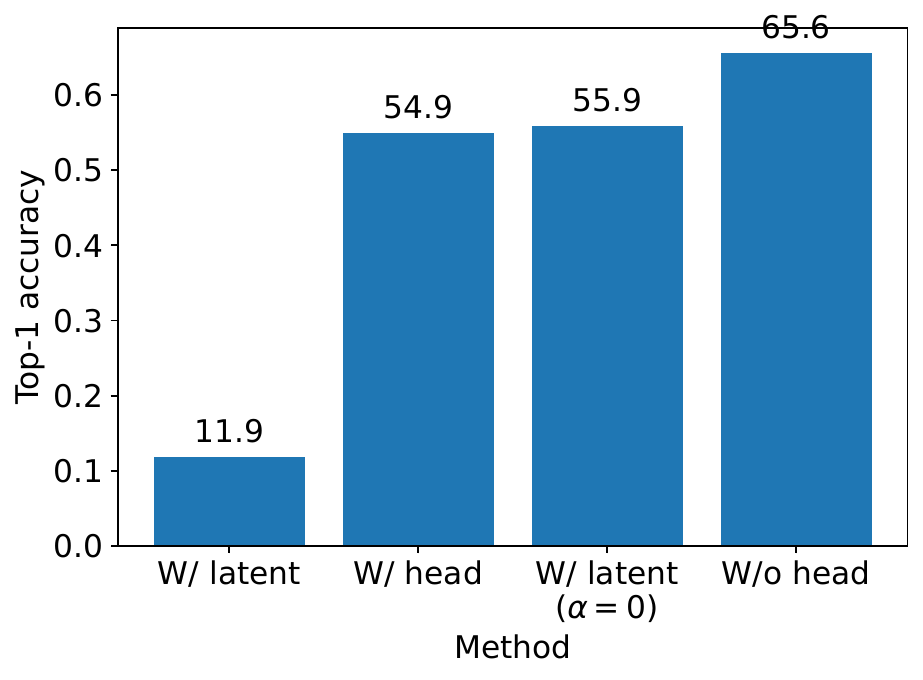}
    \caption{\nameshort{}'s linear classification accuracy with different feature extraction methods. Note that this experiment uses fewer iterations (1060000) than the main experiment (5000000) and omits data augmentation when training the linear classifier (used in the main experiment), so the accuracies are lower than the main experiment.
}
    \label{fig:repr_ablation_head}
\end{figure}

\myparatightestn{Ablation studies on the number of training steps.}
\cref{fig:repr_acc_iteration} shows classification accuracy over training iterations. Accuracy continues to improve rapidly at the end of training, suggesting that with more training, the gap between \nameshort{} and the SoTA could be further reduced.
\begin{figure}[t]
    \centering
    \includegraphics[width=0.5\linewidth]{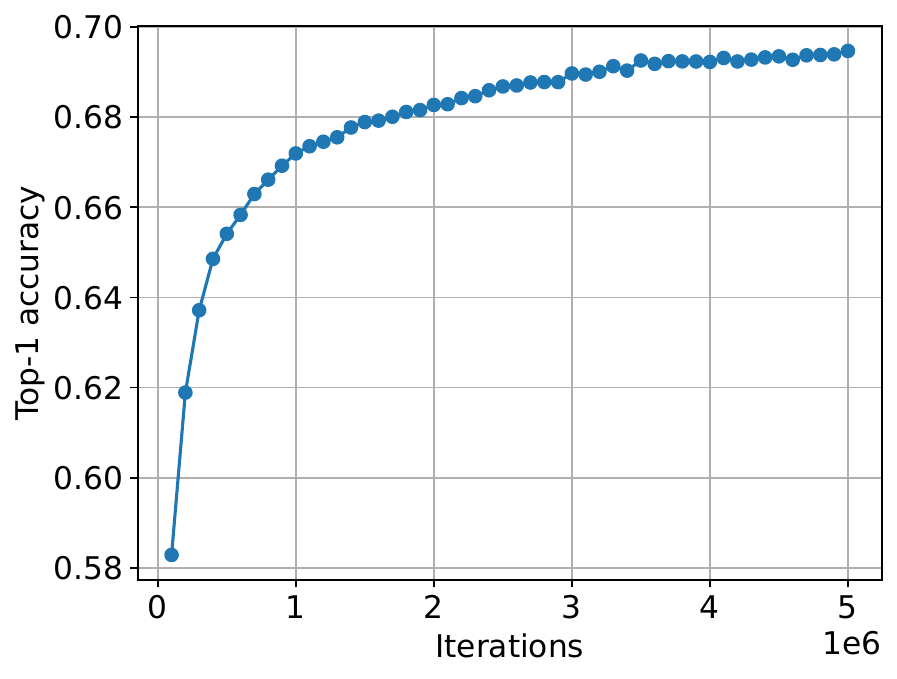}
    \caption{\nameshort{}'s linear classification accuracy vs. training iteration. The accuracy is still improving at a fast rate at the end of training. More training might further improve the result.}
    \label{fig:repr_acc_iteration}
\end{figure}

\FloatBarrier
\section{More Details and Results on Case Study 3}
\label{app:gen_and_class}

\subsection{Algorithm Pseudocode}
\label{app:gen_and_class_pseudocode}
\cref{alg:gen_and_class_training}, \cref{alg:gen_and_class_gen}, and \cref{alg:gen_and_class_class} show the algorithm pseudocode of the training, generation, and classification process.

\begin{algorithm}[t]
    \DontPrintSemicolon
    \LinesNumbered
	\BlankLine
	\SetKwInOut{Input}{Input}
	\SetKwInOut{Output}{Output}
	\caption{RF+\nameshort{} training (with class labels)}
        \label{alg:gen_and_class_training}
	\Input{Training set: labels $\sampleset$ and images $\sampleyset$\\
 Image decoder: $\decoderynotation$\\
 Image encoder: $\encoderynotation$ (used by $\latentcomputationnotation$)\\
 Label anchors: $A$ (used by $\latentalignnotation$)\\
 Number of iterations: $T$\\
 Batch size: $B$
	}
	\BlankLine
        \For{iteration $\leftarrow 1,\ldots,T$}{
            $\sampley_1,\ldots,\sampley_B \leftarrow$ Draw images from $\sampleyset$\\
            $\latent_1,\ldots,\latent_B \leftarrow \latentcomputation{\sampley_1,\ldots,\sampley_B}$\\
            $\epsilon_1,\ldots,\epsilon_B \leftarrow$ Gaussian noise\\
            $t_1,\ldots,t_B \leftarrow$ Random RF timesteps\\
            $\xi_i\leftarrow (1-t_i)\epsilon_i + t_i \sampley_i$\\
            Training using $\latentalign{\sampleset, \brc{\sampley_1,\ldots,\sampley_B}}$  and RF loss on $\decodery{\xi_i;\latent_i}$ (a weighted loss between the two)
        }
\end{algorithm}
\begin{figure}[t]
    \begin{minipage}{0.48\textwidth} %
\begin{algorithm}[H]
    \DontPrintSemicolon
    \LinesNumbered
	\BlankLine
	\SetKwInOut{Input}{Input}
	\SetKwInOut{Output}{Output}
	\caption{RF+\nameshort{} generation (unconditional)}
	\Input{Image decoder: $\decoderynotation$
	}
	\BlankLine
        $\xi\leftarrow$ Gaussian noise\\
        $\latent\leftarrow$ Gaussian noise\\
        Generated sample $\leftarrow \decodery{\xi;\latent}$
\end{algorithm}
    \end{minipage}
    \hfill %
    \begin{minipage}{0.48\textwidth}
\begin{algorithm}[H]
    \DontPrintSemicolon
    \LinesNumbered
	\BlankLine
	\SetKwInOut{Input}{Input}
	\SetKwInOut{Output}{Output}
	\caption{RF+\nameshort{} generation (unconditional)}
	\Input{Image decoder: $\decoderynotation$\\
                Label set: $c_1,\ldots,c_n$\\
                Class ID: $k$ (i.e., the class is $c_k$)\\
	}
	\BlankLine
        $\xi\leftarrow$ Gaussian noise\\
        \graybox{$\latent\leftarrow \latentcomputation{\brc{c_1,\ldots,c_n}}_k$}\\
        Generated sample $\leftarrow \decodery{\xi;\latent}$
\end{algorithm}
    \end{minipage}
    \caption{\textbf{The generation process of RF+\nameshort{} (with class labels).}  In this case, RF+\nameshort{} can simultaneously support unconditional and conditional generation.
\textbf{Left:} Unconditional generation, where the \nameshort{} latent is drawn from the prior Gaussian distribution, which is exactly the same as \cref{fig:alg_gen_gen}.  
\textbf{Right:} Conditional generation, where the \nameshort{} latent is drawn from the latent zone of the corresponding class. The changes on top of unconditional generation are highlighted in gray.
}
    \label{alg:gen_and_class_gen}
\end{figure}
\begin{algorithm}[t]
    \DontPrintSemicolon
    \LinesNumbered
	\BlankLine
	\SetKwInOut{Input}{Input}
	\SetKwInOut{Output}{Output}
	\caption{RF+\nameshort{} classification}
        \label{alg:gen_and_class_class}
	\Input{Image encoder: $\decodernotation$\\
            Label decoder: $\decodernotation$\\
            Images: $\sampley_1,\ldots,\sampley_B$
	}
	\BlankLine
        $\latent_1,\ldots,\latent_B \leftarrow \latentcomputation{\sampley_1,\ldots,\sampley_B}$\\
        $c_i\leftarrow \decoder{\latent_i}$
\end{algorithm}

\subsection{More Implementation Details}

\myparatightestn{Architecture.} 
\begin{packeditemize}
\item \textbf{Image decoder.}  
For RF, we modify the original architecture \cite{liu2022flow} to include a one-hot encoding of the class label as an additional input, concatenated with the timestep embedding.  
For RF+\nameshort{}, we apply the same modification on top of the architecture described in \cref{app:gen}.  
For unconditional generation, this one-hot encoding is deterministically derived from the \nameshort{} latent: given a \nameshort{} latent, we use the class label decoder to predict the class and then encode it as a one-hot vector.  
As a result, the decoder’s output remains fully determined by the \nameshort{} and RF latents, consistent with \cref{sec:gen}.
For conditional generation, this one-hot encoding is given as a condition, and \nameshort{} latents are sampled from the corresponding latent zone (as described in \cref{sec:lzn_latent_computation}).
    \item \textbf{Image encoder.} Same as that of \cref{app:gen}.
\end{packeditemize}

\myparatightestn{Label FM.}  
The method in \cref{sec:lzn_latent_alignment} implicitly assumes a uniform distribution over class labels.  
However, due to sampling randomness during training, each batch may have an imbalanced class distribution.  
To address this, we modify the $\pi_1$ distribution when computing $\rf{{\cdot}}$ to be a \emph{weighted} mixture of Dirac delta functions centered at $\encoder{\sample_i}$, with weights corresponding to the fraction of class $\sample_i$ samples in the batch.  
During testing, we revert to a uniform prior, as the true class distribution of the batch is not available.

\myparatightestn{Objective.} We use the version without log (\cref{eq:align_without_log}). We also tried the version with log (\cref{eq:align_with_log}) and did not observe a large difference in results.

\subsection{More Experimental Settings}
\myparatightestn{Datasets.} We use the \cifar{} dataset discussed in \cref{app:gen}.

\myparatightestn{Metrics.}
In addition to the metrics discussed in \cref{app:gen}, we evaluate on \cifar{} classification accuracy. The accuracy is evaluated on \cifar{} test set.

\myparatightestn{Sampler.} Same as \cref{app:gen}.

\myparatightestn{Hyperparameters.} 
\begin{packeditemize}
    \item \cifar{}
    \begin{packeditemize}
        \item RF:
        \begin{packeditemize}
            \item Batch size: 2000
            \item Optimizer: Adam
            \item Decoder learning rate: 0.002
            \item Gradient clipping: 1.0
            \item Number of parameters in decoder: 61809539
        \end{packeditemize}
        \item RF+\nameshort{}:
        \begin{packeditemize}
            \item Batch size: 2000
            \item Optimizer: Adam
            \item Decoder learning rate: 0.002
            \item Encoder learning rate: 0.00005
            \item Label learning rate: 0.0001
            \item Gradient clipping: 1.0
            \item Latent dimension: 200
            \item Number of parameters in decoder: 61911939
            \item Number of parameters in encoder: 49790260
        \end{packeditemize}
    \end{packeditemize}
    
\end{packeditemize}

\myparatightestn{Computation cost.} 
Excluding the computation cost of periodic evaluation (i.e., only counting the computation cost of model training), each RF+\nameshort{} experiment takes 31 hours on 16 A100 (40 GB) GPUs.

\subsection{More Results}

\myparatightestn{Generated images.} The generated images of RF and RF+\nameshort{} are in \cref{fig:gen_images_cond_lzn_cifar,fig:gen_images_cond_rf_cifar}.

\begin{figure}[t]
    \centering
    \includegraphics[width=1\linewidth]{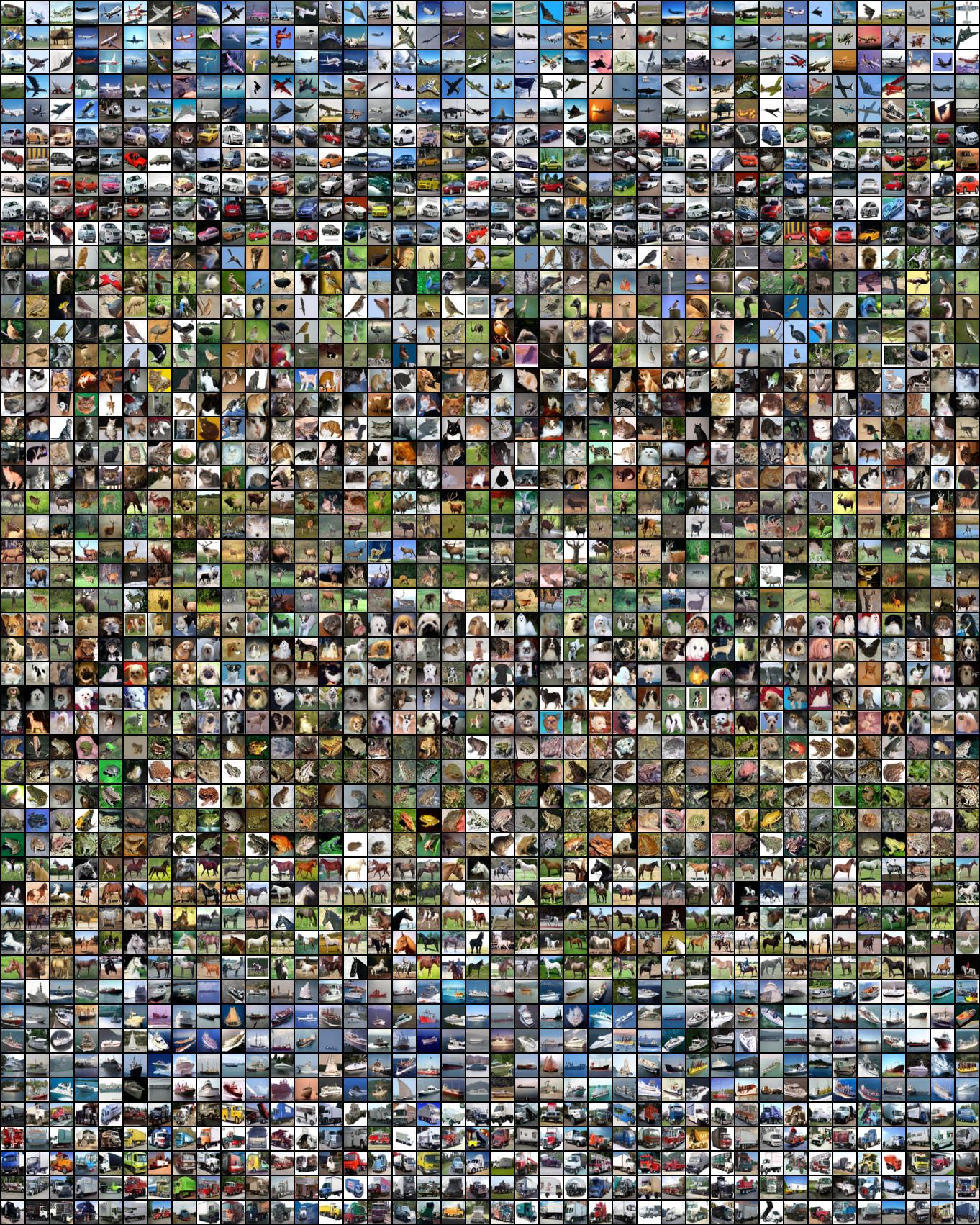}
    \caption{Generated images of RF on \cifar{} (conditional generation). Every 5 rows corresponds to one class in \cifar{}.}
    \label{fig:gen_images_cond_rf_cifar}
\end{figure}

\begin{figure}[t]
    \centering
    \includegraphics[width=1\linewidth]{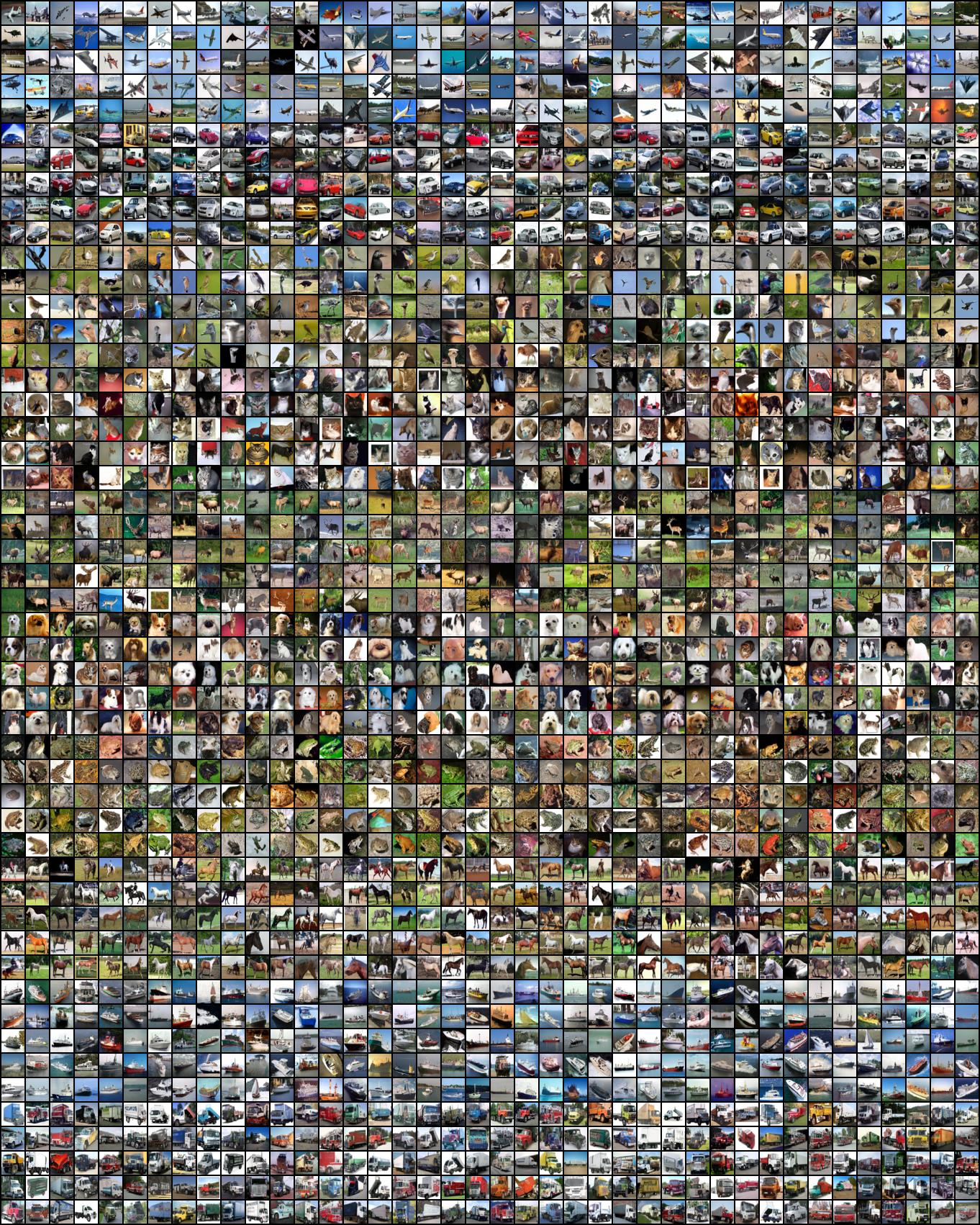}
    \caption{Generated images of RF+\nameshort{} on \cifar{}(conditional generation). Every 5 rows corresponds to one class in \cifar{}.}
    \label{fig:gen_images_cond_lzn_cifar}
\end{figure}

\myparatightestn{Ablation studies on FID implementation.}
Same as \cref{app:gen}, we present the FID scores using three different implementations in \cref{tab:gen_and_class_gen_fid}. We see that, while the numbers are different, the relative ranking across all three implementations is consistent. Especially, RF+\nameshort{} achieves the best FID in all implementations.

\begin{table*}[t]
\small
    \centering
    \caption{FID with different implementations for conditional image generation on \cifar{}. ``CM'' denotes consistency models \cite{song2023consistency}; ``RF'' denotes Rectified Flow \cite{liu2022flow}; ``clean'' denotes clean FID \cite{parmar2022aliased}. The best results are in \graybox{gray box}.}
    \vspace{-3mm}
    \label{tab:gen_and_class_gen_fid}
    \setlength{\tabcolsep}{3.8pt}
    \resizebox{0.5\textwidth}{!}{
    \begin{tabular}{l|ccc}
    \toprule
   Algo.
     & FID (clean){\color{red}$\downarrow$} & 
     FID (RF){\color{red}$\downarrow$} & FID (CM){\color{red}$\downarrow$} \\
    \hline
    RF &   2.85 & 2.50 & 2.47
  \\
    RF+\nameshort{}  &  \graybox{2.70} & \graybox{2.42} & \graybox{2.40} 
     \\
    \bottomrule
\end{tabular}
}
\vspace{-2mm}
\end{table*}

\myparatightestn{Ablation studies on sampling steps.}  
Following the experimental settings in \cref{app:gen}, we use the Euler sampler with varying numbers of sampling steps. As shown in \cref{fig:gen_fid_vs_nfe_cond_cifar}, RF+\nameshort{} generally achieves better FID than the RF baseline across most settings.

\begin{figure*}[!t]
\centering
\includegraphics[width=0.5\textwidth]{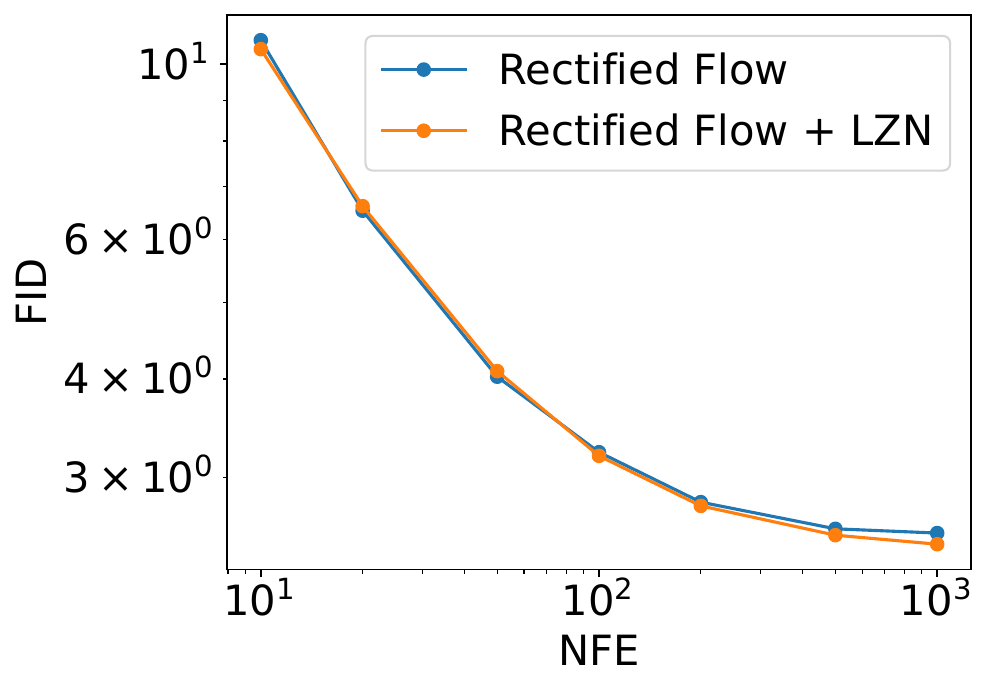}
\caption{FID vs. number of sampling steps in the Euler sampler on \cifar{} (conditional generation). RF+\nameshort{} outperforms RF in most cases.}
\label{fig:gen_fid_vs_nfe_cond_cifar}
\end{figure*}

\myparatightestn{Ablation studies on classification techniques.}
Here, we discuss several techniques for improving the classification results.
\begin{packeditemize}
    \item Recall that latent computation (\cref{eq:latent_computation}) includes randomness from $\epsilon_i$ because each sample corresponds to a latent \emph{zone}, not a single point. Empirically, for classification tasks, using the “center” of the latent zone yields better performance. Concretely, we set $\latentscale=0$ in \cref{eq:latentscale} when computing latents. This is intuitive, as the center is likely farther from zone boundaries and better represents the sample.
    \item \cref{app:latent_computation_efficiency} discusses that during training, we use a batch of samples rather than all samples to estimate latents for efficiency. However, during inference, where gradient computation is unnecessary and thus the overhead of large batch sizes is less critical, we can use a larger batch size to improve performance.
\end{packeditemize}

\cref{fig:gen_and_class_class_ablation} shows that increasing the batch size and decreasing $\latentscale$ improve the classification accuracy. The best setting improves the default setting (batch size$=2000$ and $\latentscale=1.0$) by 2.9\%.

\begin{figure}
    \centering
    \includegraphics[width=0.5\linewidth]{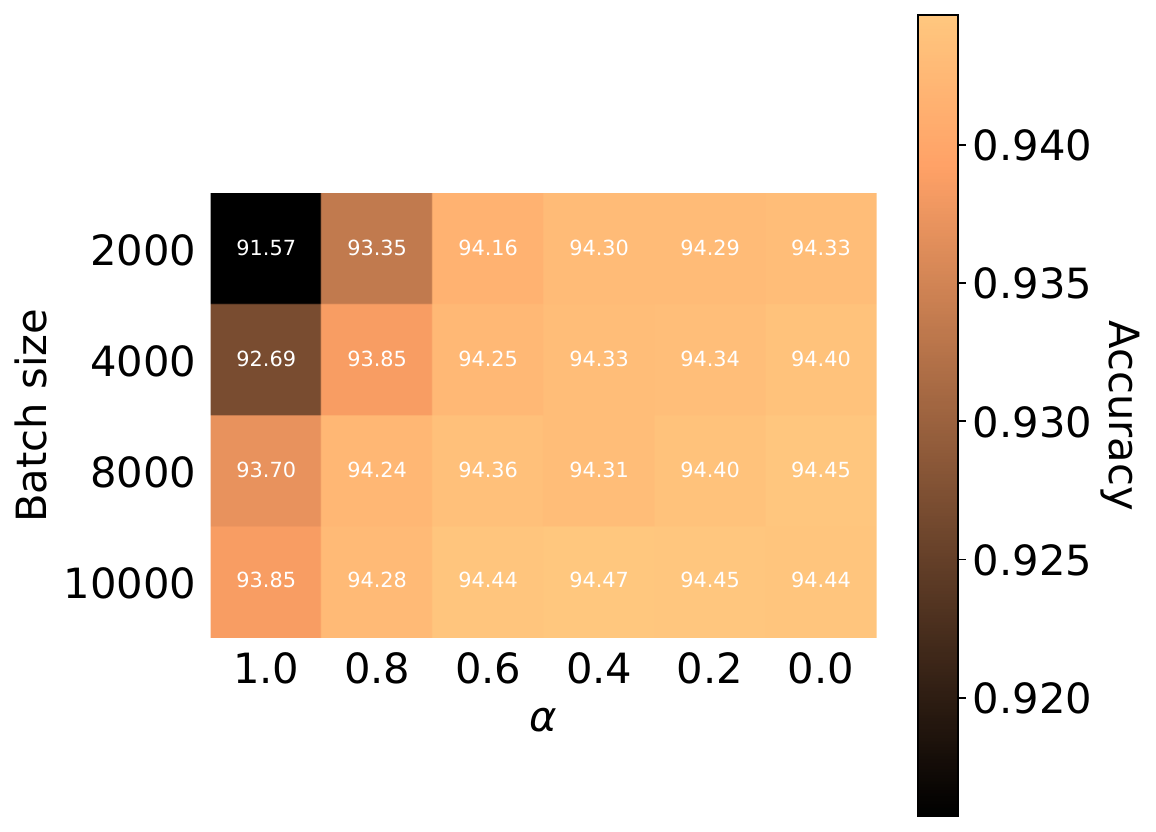}
    \caption{Classification accuracy with different hyperparameters on \cifar{}. Generally, increasing the batch size and decreasing $\latentscale$ improve the classification accuracy. }
    \label{fig:gen_and_class_class_ablation}
\end{figure}

\myparatightestn{Ablation studies on latent alignment hyperparameter.}
In latent alignment (\cref{sec:lzn_latent_alignment}), we introduced a hyperparameter $\solverstepsstart$ that controls how many time steps are excluded from the latent alignment objective. In our main experiments, we set $\solverstepsstart=20$ (out of a total of 100 steps). Here, we conduct an ablation study by reducing $\solverstepsstart$ to 5 (i.e., $4\times$ smaller). The results are shown in \cref{tab:gen_and_class_ablation}. We can see that $\solverstepsstart$ does not affect the results much, and the performance remains better than the baseline RF across most metrics. This is expected. Unlike common hyperparameters (such as loss weights) that influence the optimal solution, $\solverstepsstart$ does not alter the optimal solution, which is the perfect alignment between two latent zones. Instead, this parameter is introduced solely to help avoid getting stuck in local optima (\cref{sec:lzn_latent_alignment}). We expect that any small but non-zero value of $\solverstepsstart$ should be sufficient in practice.

\begin{table*}[t]
\small
    \centering
    \caption{Conditional image generation quality and classification accuracy on \cifar{}. The best results are in \graybox{gray box}. The hyperparameter $\solverstepsstart$ does not impact the results much. %
    }
    \label{tab:gen_and_class_ablation}
    \setlength{\tabcolsep}{3.8pt}
    \begin{tabular}{l|ccccccc}
    \toprule
   Algo.
     & FID{\color{red}$\downarrow$} & 
     sFID{\color{red}$\downarrow$} & IS{\color{blue}$\uparrow$} & Precision{\color{blue}$\uparrow$} & Recall{\color{blue}$\uparrow$} & Recon{\color{red}$\downarrow$} & Accuracy{\color{blue}$\uparrow$}\\
    \hline
    RF &   2.47 & 4.05 & 9.77 & \graybox{0.71} & \graybox{0.58} & 0.69 & -
  \\
    RF+\nameshort{} ($\solverstepsstart=20$) &  {2.40} & \graybox{3.99} & \graybox{9.88} & \graybox{0.71} & \graybox{0.58} & {0.38} & \graybox{94.47}
     \\
     RF+\nameshort{} ($\solverstepsstart=5$) &  \graybox{2.39} & \graybox{3.99} & {9.76} & \graybox{0.71} & \graybox{0.58} & \graybox{0.36} & 94.42\\
    \bottomrule
\end{tabular}
\end{table*}
\FloatBarrier
\section{Extended Discussions on Related Work}
\label{app:related_work}

\cite{dao2023flow} proposes to conduct flow matching in the latent space. However, it has quite different goals and techniques from \nameshort{}.
\begin{packeditemize}
    \item \textbf{Goals.} The goal of \cite{dao2023flow} is to improve \emph{generation tasks}. In contrast, our goal is more ambitious: to develop \emph{a unified framework that supports generation, representation learning, and classification}. This broader scope requires a different design philosophy and technical approach, as detailed next.
    \item \textbf{Techniques.} 
    \cite{dao2023flow} applies flow matching to the \emph{latent space of a pre-trained Stable Diffusion autoencoder}, which is reasonable when focusing solely on generation. However, such a latent space is \emph{high-dimensional} and retains \emph{spatial structure}, limiting its suitability for classification and compact representation learning.
    To support our broader objectives, we introduce several novel techniques:
    \begin{packeditemize}
        \item Match a \emph{discrete} distribution (i.e., the anchors) to a \emph{continuous one}, as opposed to a continuous-to-continuous distribution matching in \cite{dao2023flow}.
        \item Use an \emph{adaptive latent space}, since our encoder and decoder are trained end-to-end, as opposed to using a fixed pre-trained autoencoder and fixed latent space in \cite{dao2023flow}.
        \item \emph{Numerically solve} the flow directly, as opposed to training an additional model to learn the flow in \cite{dao2023flow}.
        \item \emph{Latent alignment} between different data types (e.g., image and label), which is new in our paper.
    \end{packeditemize}

\end{packeditemize}
\FloatBarrier
\section{Extended Discussions on Limitations and Future Work}
\label{app:discussions}

\myparatightestn{Inference efficiency. } It is important to note that while the training cost of \nameshort{} might be high, \emph{at inference time, LZN is often as efficient as existing approaches.}
\begin{packeditemize}
    \item For \emph{image generation} (\cref{sec:gen,sec:gen_and_class}), we do not need to compute the latent during inference. Instead, latents are sampled from the Gaussian prior and passed directly to the decoder, making the generation speed comparable to the base model. 
    \item For \emph{representation learning} (\cref{sec:repr}), we find that dropping the final encoder layers during inference improves performance (\cref{app:repr_choice}), similar to the observation in prior contrastive learning methods \cite{chen2020simple}. In this case, inference involves simply passing an image through the encoder \emph{without} the latent computation process (\cref{sec:lzn_latent_computation}), just like in traditional contrastive learning methods.
\end{packeditemize}

\myparatightestn{Training efficiency. } The main training bottleneck stems from the quadratic cost with respect to the batch size. Notably, this is also the case for many contrastive learning methods, including the seminal works MoCo \cite{he2020momentum} and SimCLR \cite{chen2020simple}, which compute pairwise similarities between all examples in a batch.

\myparatightestn{The parallel between LLM training and \nameshort{} training.} We observe an interesting parallel between the training of LLMs and \nameshort{}. Specifically, in LLM training, computing attention weights requires $\bigO\bra{c^2 d v}$, where $c$ is the context length, $d$ is the attention dimension, and $v$ is the number of layers. In \nameshort{}, computing the latents (\cref{sec:lzn_latent_computation}) requires $\bigO(n^2 \latentdim \solversteps)$, where $n$ is the number of samples in a batch, $\latentdim$ is the latent dimension, and $\solversteps$ is the number of solver steps. Several parallels emerge:
\begin{packeditemize}
    \item Context length in LLMs ($c$) $\leftrightarrow$ Number of samples in \nameshort{} ($n$)
    \item Attention dimension in LLMs ($d$) $\leftrightarrow$ Latent dimension in \nameshort{} ($\latentdim$)
    \item Number of layers in LLMs ($v$) $\leftrightarrow$ Number of solver steps in \nameshort{} ($\solversteps$)
\end{packeditemize}
Not only do these parameter pairs affect the time complexity in similar ways, but their computation flows are also analogous: in LLMs, the pairwise inner product of token features is computed to derive attention weights, and these weights are computed sequentially across layers. Similarly, in \nameshort{}, the pairwise distances between intermediate anchor points of samples are computed to derive velocity, and this velocity is updated sequentially across solver steps.

While LLM training is known to be computationally expensive, recent advances have significantly improved its efficiency. Given the structural similarities, we expect that such advances in LLM training could be adapted to enhance the training efficiency of \nameshort{} as well.

\myparatightestn{Using \nameshort{} solely to implement generative modeling.} In theory, \nameshort{} can be used solely for generative modeling. By construction (\cref{sec:gen}), if the decoder is trained to map latents to the corresponding data perfectly, then the generative distribution of \nameshort{} is exactly $\frac{1}{n} \sum_{i=1}^{n} \delta(s - \sample_i)$, i.e., the empirical distribution of the training set. We explored this approach in our early experiments. It performs well on simple datasets such as \mnist{} \cite{deng2012mnist}, but generates blurry images on more complex datasets such as \cifar{}. 
We hypothesize that this may be due to the minibatch approximation (\cref{app:latent_computation_efficiency}), which can break the disjoint latent property, and/or the strict requirement that latent zones have no gaps between them. %
We leave a deeper exploration of this direction to future work.

\myparatightestn{Societal impacts.} Since \nameshort{} can be used to improve ML models, it has the potential for both beneficial and harmful applications. Positive use cases include creative content generation and improved information retrieval, while negative applications may involve the creation of fake or misleading content.

\FloatBarrier

\end{document}